\documentclass[11pt]{article}

\usepackage{xcolor}
\usepackage[sc]{mathpazo}
\usepackage{multirow}
\usepackage{makecell}
\usepackage{booktabs}
\usepackage{colortbl}
\usepackage{rotating}

\usepackage[margin=1in]{geometry}
\usepackage[english]{babel}
\usepackage[compact]{titlesec}

\usepackage[utf8]{inputenc} 
\usepackage[T1]{fontenc}    

\usepackage{url}            
\usepackage{booktabs}       
\usepackage{comment}        
\usepackage{amsfonts}       
\usepackage{nicefrac}       
\usepackage{microtype}      
\usepackage{latexsym,amsthm,amssymb,amsmath,amscd,euscript}
\usepackage{algorithm}
\usepackage{algorithmic}
\usepackage{xcolor}
\usepackage{subcaption}
\usepackage{wrapfig}
\usepackage{tablefootnote}

\usepackage{enumitem}
\setlist[itemize]{leftmargin=8pt}

\newtheorem{definition}{Definition}
\newtheorem{assumption}{Assumption}
\newtheorem{lemma}{Lemma}
\newtheorem{theorem}{Theorem}
\newtheorem{claim}{Claim}
\newtheorem{remark}{Remark}

\newtheorem{corollary}{Corollary}

\definecolor{lred}{RGB}{236,33,39}
\definecolor{lyellow}{RGB}{251,168,26}
\definecolor{lblue}{RGB}{94,202,231}
\definecolor{lgreen}{HTML}{03AC13}

\definecolor{mblue}{rgb}{0.37,0.51,0.71}
\definecolor{myellow}{rgb}{0.88,0.61,0.14}
\definecolor{mgreen}{HTML}{B8DCAB}

\definecolor{niceRed}{RGB}{190,38,38}
\definecolor{blueGrotto}{HTML}{059DC0}
\definecolor{royalBlue}{HTML}{057DCD}
\definecolor{navyBlueP}{HTML}{0B579C}
\definecolor{limeGreen}{HTML}{81B622}

\definecolor{Sepia}{HTML}{7F462C}

\usepackage{accents}
\usepackage{setspace,mathtools}
\usepackage{tikz-cd}
\usepackage{hyperref}       
\hypersetup{
  colorlinks = true,
  urlcolor = {limeGreen},
  linkcolor = {blueGrotto},
  citecolor = {royalBlue}
}

\makeatletter
\renewcommand{\ALG@name}{Dynamics}
\makeatother

\makeatletter
\newtheorem*{rep@theorem}{\rep@title}
\newcommand{\newreptheorem}[2]{%
\newenvironment{rep#1}[1]{%
 \def\rep@title{#2 \ref{##1}}%
 \begin{rep@theorem}}%
 {\end{rep@theorem}}}
\makeatother
\DeclarePairedDelimiter{\norm}{\lVert}{\rVert}
\newcommand{\R}{\mathbb{R}}
\newcommand{\eps}{\varepsilon}

\begin{document}

\begin{center}

{\bf{\LARGE{STay-ON-the-Ridge: Guaranteed Convergence to Local \\[5pt] Minimax Equilibrium in Nonconvex-Nonconcave Games}}}

\vspace*{.2in}

{\large{
\begin{tabular}{c}
Constantinos Daskalakis$^\diamond$ \and Noah Golowich$^\diamond$ \and Stratis Skoulakis$^\dagger$ \and Manolis Zampetakis$^\star$ 
\end{tabular}
}}

\vspace*{.2in}

\begin{tabular}{c}
Massachusetts Institute of Technology$^\diamond$ \\
École Polytechnique Fédérale de Lausanne$^\dagger$ \\
University of California, Berkeley$^\star$
\end{tabular}

\vspace*{.2in}

\today

\vspace*{.2in}

\end{center}

\begin{abstract}
  Min-max optimization problems involving nonconvex-nonconcave objectives have found important applications in adversarial training and other multi-agent learning settings. Yet, no known gradient descent-based method is guaranteed to converge to (even local notions of) min-max equilibrium in the nonconvex-nonconcave setting. For all known methods, there exist relatively simple objectives for which they cycle or exhibit other undesirable behavior different from converging to a point,  let alone to some game-theoretically meaningful one~\cite{flokas2019poincare,hsieh2021limits}. The only known convergence guarantees hold under the strong assumption that the initialization is very close to a local min-max equilibrium~\cite{wang2019solving}. Moreover, the afore-described challenges are not just theoretical curiosities. All known methods are  unstable  in  practice, even in simple settings.
  
  We propose the first method that is guaranteed to converge to a local min-max equilibrium for smooth nonconvex-nonconcave objectives. Our method is second-order and provably escapes limit cycles as long as it is initialized at an easy-to-find initial point. Both the definition of our method and its convergence analysis are motivated by the topological nature of the problem. In particular, our method is not designed to decrease some potential function, such as the distance of its iterate from the set of local min-max equilibria or the projected gradient of the objective, but is designed to satisfy a topological property that guarantees the avoidance of cycles and implies its convergence. 
\end{abstract}

\thispagestyle{empty}

\clearpage
\section{Introduction} \label{sec:intro}
  
   Min-max optimization lies at the foundations of
Game Theory~\cite{VN28}, Convex Optimization~\cite{Dantzig1951,Adler13} and
Online Learning 
\cite{B56, Hannan57,C06}, and  has found many applications in theoretical and applied fields including, more recently, in adversarial training and other multi-agent  learning problems~\cite{GAN14,madry2017towards,zhang2019multi}. In its general  form, it can be written as 
\begin{align}
  \min_{\theta \in {\Theta}} \max_{\omega \in {\Omega}} f(\theta, \omega), \label{eq:costis1}
\end{align}
where $\Theta$ and $\Omega$ are convex  subsets of the Euclidean space, and $f$ is continuous.

Eq.~\eqref{eq:costis1} can be viewed as a model of a sequential-move game wherein a player who is interested in minimizing $f$ chooses $\theta$ first, and then a player who is interested in maximizing $f$ chooses $\omega$ after seeing $\theta$. A solution to~\eqref{eq:costis1} corresponds to a Nash equilibrium of this sequential-move game.  

We may also study the simultaneous-move game with the same objective  $f$ wherein the minimizing player and the maximizing player choose $\theta$ and $\omega$ simultaneously. The Nash equilibrium of the simultaneous-move game, also called a {\em min-max equilibrium}, is a pair
$(\theta^{\star}, \omega^{\star}) \in \Theta \times {\Omega}$ such that
\begin{align}
    & f(\theta^{\star}, \omega^{\star}) \le f(\theta, \omega^{\star}), ~~\text{for all~$\theta \in \Theta$}  ~~ \text{  and  } ~~  f(\theta^{\star}, \omega^{\star}) \ge f(\theta^{\star}, \omega), ~~\text{for all~$\omega \in \Omega$}. \label{eq:local max costis}
\end{align}

  It is easy to see that a Nash equilibrium of the simultaneous-move game also constitutes a Nash equilibrium of the sequential-move game, but the converse need not be true. Here, we focus on solving the (harder) simultaneous-move game. In particular, we study the existence of \emph{dynamics} which converge to solutions of the simultaneous-move game, namely the existence of methods that make incremental updates to a pair $(\theta_t, \omega_t)$ so as the sequence $(\theta_t, \omega_t)$ converges, as $t \to \infty$, to some $(\theta^*,\omega^*)$ satisfying~\eqref{eq:local max costis} or some relaxation of it. 
  
  This problem has been extensively studied in the special case where $\Theta$ and $\Omega$ are convex and compact and $f$ is convex-concave --- i.e.~convex in $\theta$ for all $\omega$ and concave in $\omega$ for all $\theta$. In this case, the set of Nash equilibria of the simultaneous-move game is equal to the set of Nash equilibria of the sequential-move game, and these sets are non-empty and convex~\cite{VN28}. Even in this simple setting, however, many natural dynamics surprisingly fail to converge: \emph{gradient descent-ascent}, as well as various continuous-time versions of \emph{follow-the-regularized-leader}, not only fail to converge to a min-max equilibrium, even for very simple objectives, but may even exhibit chaotic behavior~\cite{MertikopoulosPP18,flokas2019poincare,hsieh2021limits}. In order to circumvent these negative results, an extensive line of work has introduced other algorithms, such as \emph{extragradient}~\cite{korpelevich1976extragradient} and \emph{optimistic gradient descent}~\cite{Popov1980}, which exhibit last-iterate convergence to the set of min-max equilibria in this setting; see e.g.~\cite{daskalakis2017training,daskalakis2018limit, mazumdar2018convergence, rafique2018non, hamedani2018primal, adolphs2018local, daskalakis2019last, liang2019interaction, gidel2019negative, mokhtari2019unified, abernethy2019last,golowich2020last,golowich2020tight}.
Alternatively, one may take advantage of the convexity of the problem, which implies that several no-regret learning procedures, such as online gradient descent, exhibit \emph{average}-iterate convergence to the set of min-max equilibria~\cite{C06,shalev2012online, bubeck2012regret,shalev2014understanding, hazan2016introduction}. 
Moreover, \cite{lin2020gradient, kong2021accelerated, ostrovskii2021efficient} show that convexity with respect to one of the two players is enough to design algorithms that exhibit average-iterate convergence to min-max equilibria.

Our focus in this paper is on the more general case where $f$ is not assumed to be convex-concave, i.e.~it may fail to be convex in $\theta$ for all $\omega$, or may fail to be concave in $\omega$ for all $\theta$, or both. We call this general setting where neither convexity with respect to $\theta$ nor concavity with respect to $\omega$ is assumed, the \textit{nonconvex-nonconcave} setting.  
This setting presents some substantial challenges. First, min-max equilibria are {\em not} guaranteed to exist, i.e.~for general objectives there may be no $(\theta^{\star}, \omega^{\star})$ satisfying~\eqref{eq:local max costis};  this happens even in very simple cases, e.g.~when $\Theta=\Omega=[0,1]$ and  $f(\theta,\omega)=(\theta-\omega)^2$. Second, it is $\mathrm{NP}$-hard to determine whether a  min-max equilibrium exists~\cite{DaskalakisSZ21} and, as is easy to see, it is also $\mathrm{NP}$-hard to compute Nash equilibria of the sequential-move game (which do exist under compactness of the constraint sets). For these reasons, the optimization literature has targeted the computation  of local and/or approximate  solutions in this setting~\cite{daskalakis2018limit,mazumdar2018convergence,jin2019minmax,wang2019solving,DaskalakisSZ21,MangoubiV21}. 
This is the approach we also take in this paper,  targeting the computation of {\em $(\eps, \delta)$-local min-max equilibria}, which were proposed in~\cite{DaskalakisSZ21}. These are approximate and local Nash equilibria of the simultaneous-move game, defined as feasible points $(\theta^\star,\omega^\star)$ which satisfy a relaxed and local version of \eqref{eq:local max costis}, namely:
\begin{align}
  &f(\theta^{\star}, \omega^{\star}) < f(\theta, \omega^{\star}) + \eps, ~\text{for all~$\theta \in \Theta$ such that $\norm{\theta - \theta^{\star}} \le \delta$}; \label{eq:local min costis2}\\
  &f(\theta^{\star}, \omega^{\star}) > f(\theta^{\star}, \omega) - \eps, ~\text{for all~$\omega \in \Omega$ such that $\norm{\omega - \omega^{\star}} \le \delta$}. \label{eq:local max costis2}
\end{align}
Besides being a natural concept of local, approximate min-max equilibrium, an attractive feature of $(\eps,\delta)$-local min-max equilibria is that they are guaranteed to exist when $f$ is $\Lambda$-smooth and the locality parameter, $\delta$, is chosen  small enough in terms of the smoothness, $\Lambda$, and the approximation parameter, $\eps$, namely whenever $\delta \le \sqrt{\frac{2 \eps}{\Lambda}}$. Indeed, in this regime of parameters the $(\eps, \delta)$-local min-max equilibria are in correspondence with the approximate fixed points of the 
\textit{Projected Gradient Descent/Ascent} dynamics. Thus,  the existence of the former can be established by invoking Brouwer's fixed point theorem to establish the existence of the latter. (Theorem 5.1 of~\cite{minmaxComplexityArxiv}).

There are a number of existing approaches which would be natural to use to find a solution $(\theta^\star, \omega^\star)$ satisfying \eqref{eq:local min costis2} and \eqref{eq:local max costis2}, but all run into significant obstacles. First, the idea of averaging, which can be leveraged in the convex-concave setting to obtain provable guarantees for otherwise chaotic algorithms, such as online gradient descent, no longer works, as it critically uses Jensen's inequality which needs convexity/concavity. On the other hand, negative results abound for last-iterate convergence: \cite{hsieh2021limits} show that a variety of zeroth, first, and second order methods may converge to a limit cycle, even in simple settings. \cite{flokas2019poincare} study a particular class of nonconvex-nonconcave games and show that continuous-time gradient descent-ascent (GDA) exhibits \emph{recurrent} behavior. Furthermore, common variants of gradient descent-ascent, such as optmistic GDA (OGDA) or extra-gradient (EG), may be unstable even in the proximity of local min-max equilibria, or  converge to fixed points that are not local min-max equilibria~\cite{daskalakis2018limit,jin2019minmax}. While there do exist algorithms, such as {\sc Follow-The-Ridge} proposed by \cite{wang2019solving}, which provably exhibit \emph{local convergence} to a (relaxation of) local min-max equilibrium, these algorithms do not enjoy global convergence guarantees, and no algorithm is known with guaranteed convergence to a local min-max equilibrium. 

These negative theoretical results are consistent with the practical experience with min-maximization of nonconvex-nonconcave objectives, which is rife with frustration as well. A common experience
is that the training dynamics of first-order methods are unstable, oscillatory 
or divergent, and the quality of the points encountered in the course of
training can be poor; see e.g.~\cite{goodfellow2016nips,metz2016unrolled,daskalakis2017training,mescheder2018training,daskalakis2018limit,mazumdar2018convergence,MertikopoulosPP18,adolphs2018local}.  In light of the failure of essentially all known algorithms to guarantee convergence, even asymptotically, to local min-max equilibria, we ask the following question: \emph{Is there an algorithm which  is guaranteed to converge to a local min-max equilibrium in the nonconvex-nonconcave setting~\cite{wang2019solving}?}

\subsection{Our Contribution}

In this work we answer the above question in the affirmative: \textbf{we propose a second-order method that is guaranteed to converge to a local min-max equilibrium (Theorem~\ref{t:main})}. Our algorithm, called {\sc STay-ON-the-Ridge} or STON'R, has some similarity to {\sc Follow-The-Ridge} or FTR, which  only converges locally and to a relaxed notion of min-max equilibrium. Both the structure of our algorithm and its global convergence analysis are motivated by the topological nature of the problem, as established by~\cite{DaskalakisSZ21} who showed that the problem is computationally (and mathematically) equivalent to Brouwer fixed point computation. In particular, the structure and analysis of STON'R are not based on a potential function argument but on a {\em parity argument} (see Section \ref{sec:ppadArgument}), akin to the combinatorial argument used to prove the existence of Brouwer fixed points. 

Table~\ref{tbl:comparison} shows our contributions in the context of what was known prior to our work about equilibrium existence, equilibrium complexity, and existence of dynamics with guaranteed convergence to equilibrium in zero-sum games with objectives of differing complexity.  

\newcolumntype{Q}{>{\centering\let\newline\\\arraybackslash\hspace{0pt}}m{2.2cm}}
\newcolumntype{S}{>{\centering\let\newline\\\arraybackslash\hspace{0pt}}m{3.5cm}}
\newcolumntype{C}{>{\centering\let\newline\\\arraybackslash\hspace{0pt}}m{3.5cm}}
\newcolumntype{R}{>{\centering\let\newline\\\arraybackslash\hspace{0pt}}m{4cm}}
\begin{table}
\renewcommand{\arraystretch}{2}
\centering
\begin{tabular}{ l Q | S  C  R }
     & & convex-concave & nonconvex-concave & \large \textbf{nonconvex-nonconcave} \normalsize \\ \specialrule{1.5pt}{0pt}{0pt}
     \rowcolor[gray]{.93} \cellcolor{white} & existence  & \textcolor{lgreen}{\textbf{yes}} \footnotesize \cite{v1928theorie} \normalsize & 
       $\textcolor{lred}{\textbf{no}}^{\dagger}$
      &  $\textcolor{lred}{\textbf{no}}^{\dagger}$  \\ 
    \rowcolor[gray]{.97} \cellcolor{white} & complexity & \textcolor{lgreen}{\textbf{poly-time}} \footnotesize  e.g.~\cite{dantzig1951proof,freund1997decision,shalev2012online} \normalsize & $\textcolor{lred}{\textbf{NP-hard}}^{\star}$
    & \textcolor{lred}{\textbf{NP-hard}}  \footnotesize\cite{DaskalakisSZ21} \normalsize \\ 
    \cellcolor{white} \multirow{-3}{*}{\begin{sideways} $~~~~~$Nash Eq. \end{sideways}}  & convergent dynamics & \textcolor{lgreen}{\textbf{many}} \footnotesize e.g.~\cite{freund1997decision,C06,shalev2012online} \normalsize & not applicable & not applicable  \\ \specialrule{1pt}{0pt}{0pt}
     \rowcolor[gray]{.93} \cellcolor[gray]{1} & existence  & \cellcolor[gray]{0.75} \textcolor{white}{\textit{same as above}} & \textcolor{lgreen}{\textbf{yes}} & \textcolor{lgreen}{\textbf{yes}} \footnotesize\cite{DaskalakisSZ21}\normalsize \\ 
     \rowcolor[gray]{.97}  \cellcolor[gray]{1}                         & complexity & \cellcolor[gray]{0.75} \textcolor{white}{\textit{same as above}} & \textcolor{lgreen}{\textbf{poly-time}} \footnotesize \cite{lin2020gradient, kong2021accelerated, ostrovskii2021efficient} \normalsize &  \textcolor{lred}{\textbf{PPAD-hard}} \footnotesize\cite{DaskalakisSZ21} \normalsize\\ 
     \cellcolor[gray]{1} \multirow{-3}{*}{\begin{sideways}$~~~~~~~~~~$Local Nash Eq.\end{sideways}}    & convergent dynamics & \cellcolor[gray]{0.75} \textcolor{white}{\textit{same as above}} & \textcolor{lgreen}{\textbf{many}} \footnotesize \cite{lin2020gradient, kong2021accelerated, ostrovskii2021efficient} \normalsize &  \cellcolor{mgreen} \large\textcolor{lgreen}{\textbf{This paper}} \normalsize \\ \specialrule{1.5pt}{0pt}{0pt}
\end{tabular}
\caption{Summary of known results for equilibrium existence, equilibrium complexity, and existence of dynamics converging to equilibrium for simultaneous zero-sum games with differing complexity in their objective function.\\[5pt]
\footnotesize$(\dagger)$ For example, the zero-sum game with objective function $f(\theta,\omega)=-(\theta-\omega)^2$, where the minimizing player chooses $\theta \in [-1, 1]$ and the maximizing player chooses $\omega \in [-1, 1]$, does not have any Nash Equilibrium.\\[5pt] 
$(\star)$ Although it is not explicitly stated in \cite{DaskalakisSZ21}, this is a consequence of the proof of Theorem 10.1 in \cite{minmaxComplexityArxiv}.\normalsize}
\label{tbl:comparison}
\end{table}

\subsection{Simulated Experiments}

As a warm-up we present some simulated experiments to compare the performance of our algorithm with the widely used algorithms for min-max optimization. More precisely, we compare: Gradient Descent Ascent (GDA; Figure \ref{fig:simulations:GDA}),
Extra-Gradient (EG; Figure \ref{fig:simulations:EG}),
Follow-the-Ridge (FtR; Figure \ref{fig:simulations:FtR}), and
STay-ON-the-Ridge (STON'R; Figure \ref{fig:simulations:STONR}) in the following 2-D examples:
\[\min_{\theta \in [-1, 1]} \max_{\omega \in [-1, 1]} f_1(\theta,\omega) := (4 \theta^2 -(\omega - 3 \theta +\frac{ \theta^3}{20})^2 - \frac{\omega^4}{10}) \exp(-\frac{\theta^2 + \omega^2}{100})\text{, and} \]
\[\min_{\theta \in [-1, 1]} \max_{\omega \in [-1, 1]} f_2(\theta,\omega) := - \theta \omega - \frac{1}{20} \cdot \omega^2 + \frac{2}{20} \cdot S\left(\frac{\theta^2 + \omega^2}{2}\right) \cdot \omega^2 \]
where $S$ is the smooth-step function $S(\theta) = \begin{cases} 0, \theta\le 0\\3 \theta^2 - 2 \theta^3,\theta \in [0, 1]\\ 1, \theta \ge 1\end{cases}$.

We do not provide separate plots for Optimistic Gradient Descent Ascent (OGDA) because its behavior is almost identical with the behavior of EG in these examples and hence all our comments about EG transfer to OGDA as well. In all the following figures the different colors represent trajectories with different initialization. The initialization of every trajectory is represented by a dot and the line represent the path that the algorithm follows starting from the dot.

Observe that all the known methods either get trapped on a limit cycle, or they only converge when initialized very close to the solution. Our algorithm (Figure \ref{fig:simulations:STONR}) is the only one that converges in both of these examples when initialized in $(-1, -1)$ which is far away from the solution.


\begin{figure}[!ht]
  \centering
  \begin{subfigure}[b]{0.38\textwidth}
    \centering
    \includegraphics[width=\textwidth]{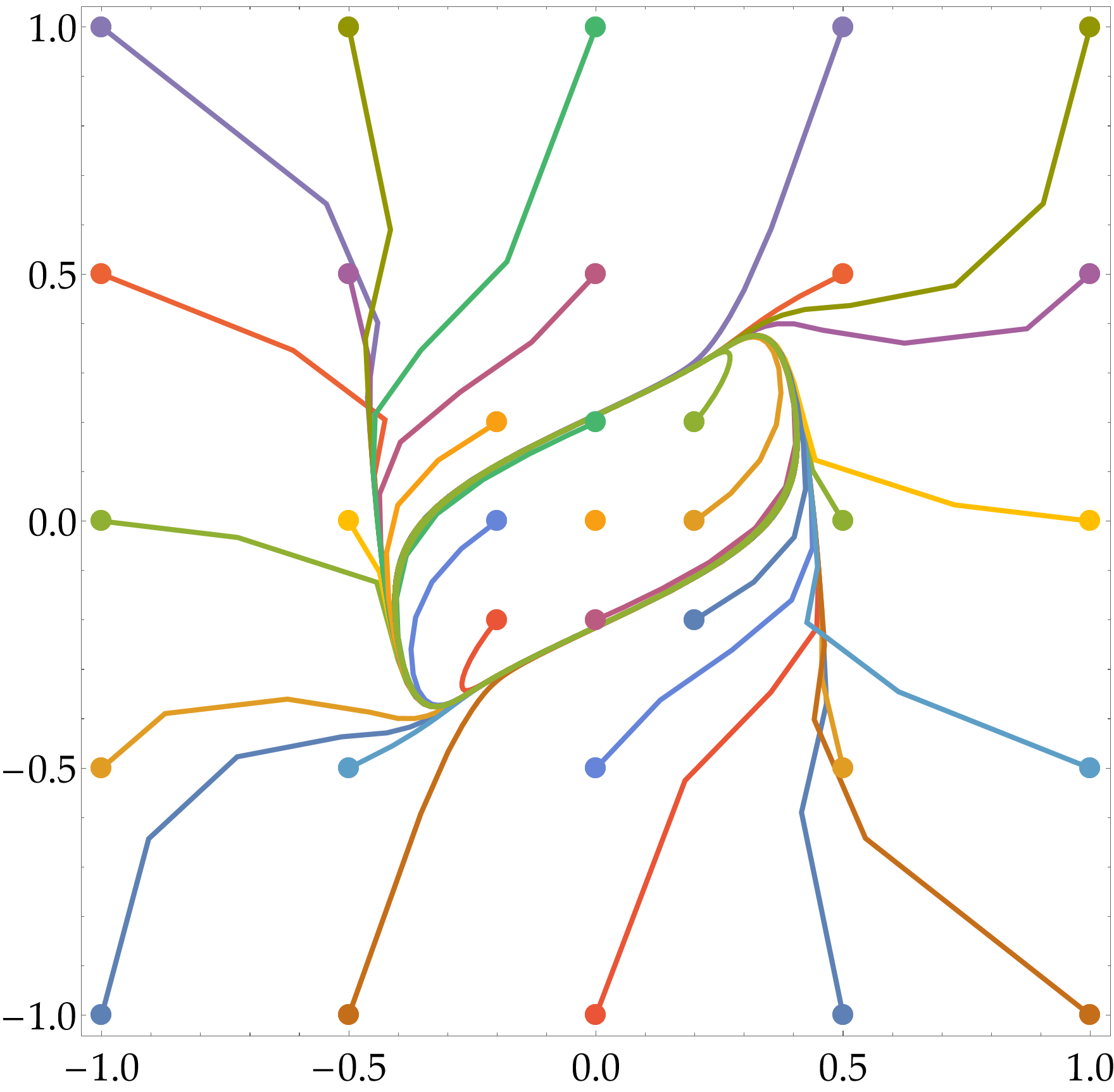}
    \caption{$f_1(\theta, \omega)$.}
  \end{subfigure}
  ~~~~~~~~~
  \begin{subfigure}[b]{0.38\textwidth}
    \centering
    \includegraphics[width=\textwidth]{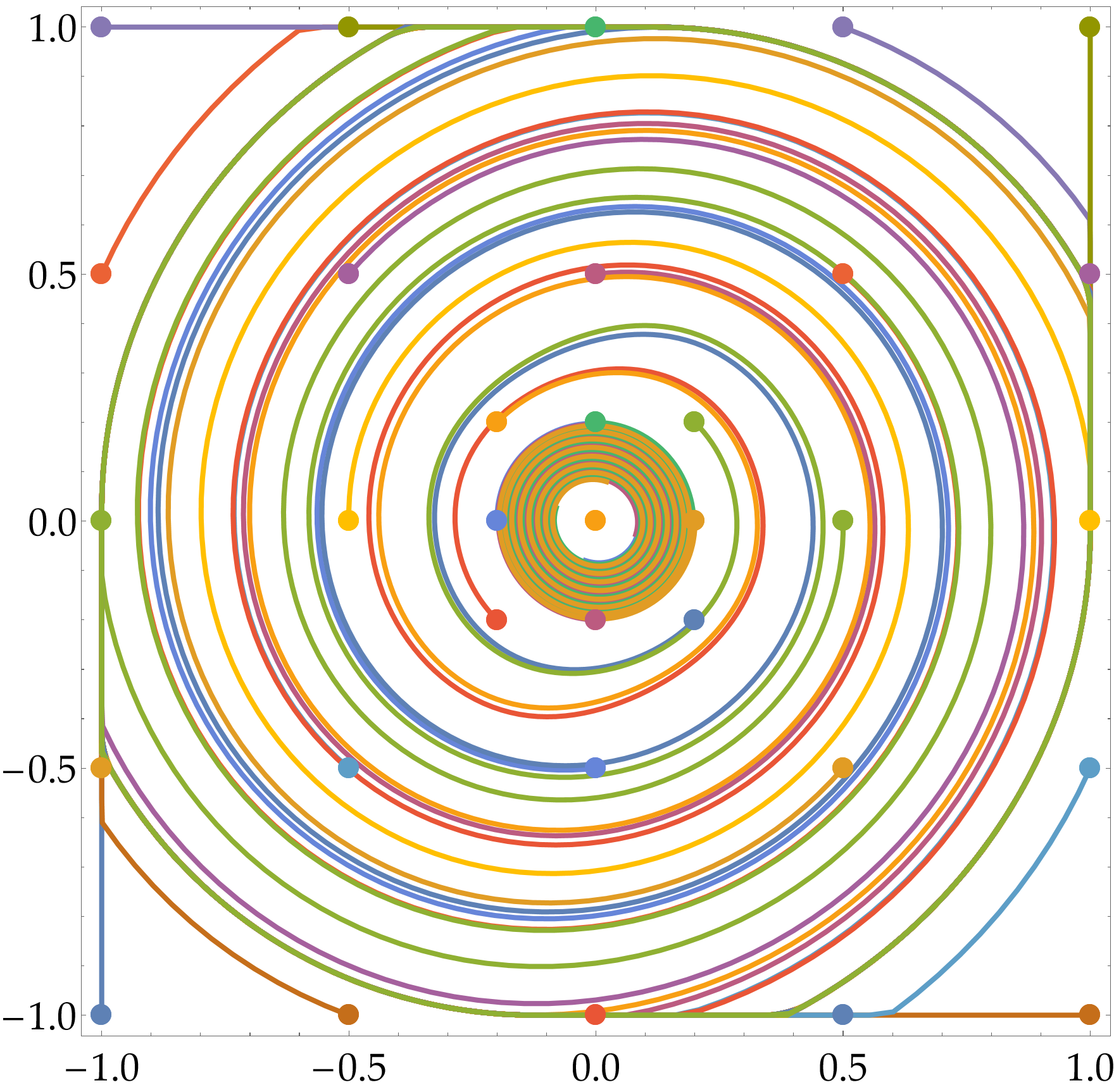}
    \caption{$f_2(\theta, \omega)$.}
  \end{subfigure}

  \caption{(Algorithm: GDA) \textbf{(a)} We observe that for any initial condition the algorithm converges to the same limit cycle. The only exception is when the algorithm is initialized exactly on $(0, 0)$ where the gradients are $0$ and hence it does not move. So in this example, unless initialized on the equilibrium, the algorithm converges to a specific limit cycle. \textbf{(b)} In this example, if the algorithm is initialized far away from the equilibrium, which is $(0, 0)$, then it \textit{diverges}, i.e., it moves towards the boundary. On the other hand, if the algorithm is initialized close enough to the equilibrium then it slowly converges to the equilibrium point with a very slow rate.}
  \label{fig:simulations:GDA}
\end{figure}


\begin{figure}[!ht]
  \centering
  \begin{subfigure}[b]{0.38\textwidth}
    \centering
    \includegraphics[width=\textwidth]{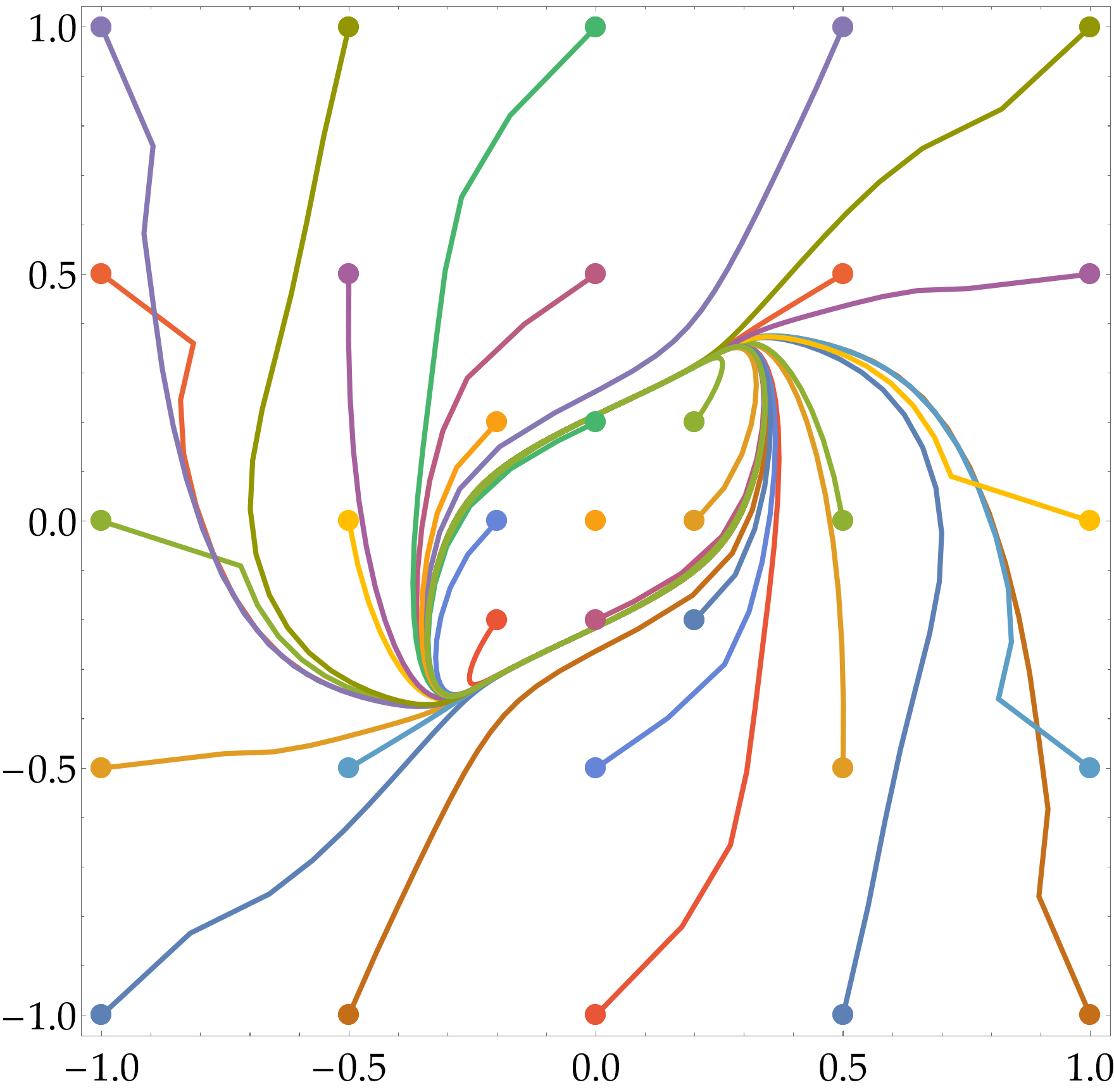}
    \caption{$f_1(\theta, \omega)$.}
  \end{subfigure}
  ~~~~~~~~~
  \begin{subfigure}[b]{0.38\textwidth}
    \centering
    \includegraphics[width=\textwidth]{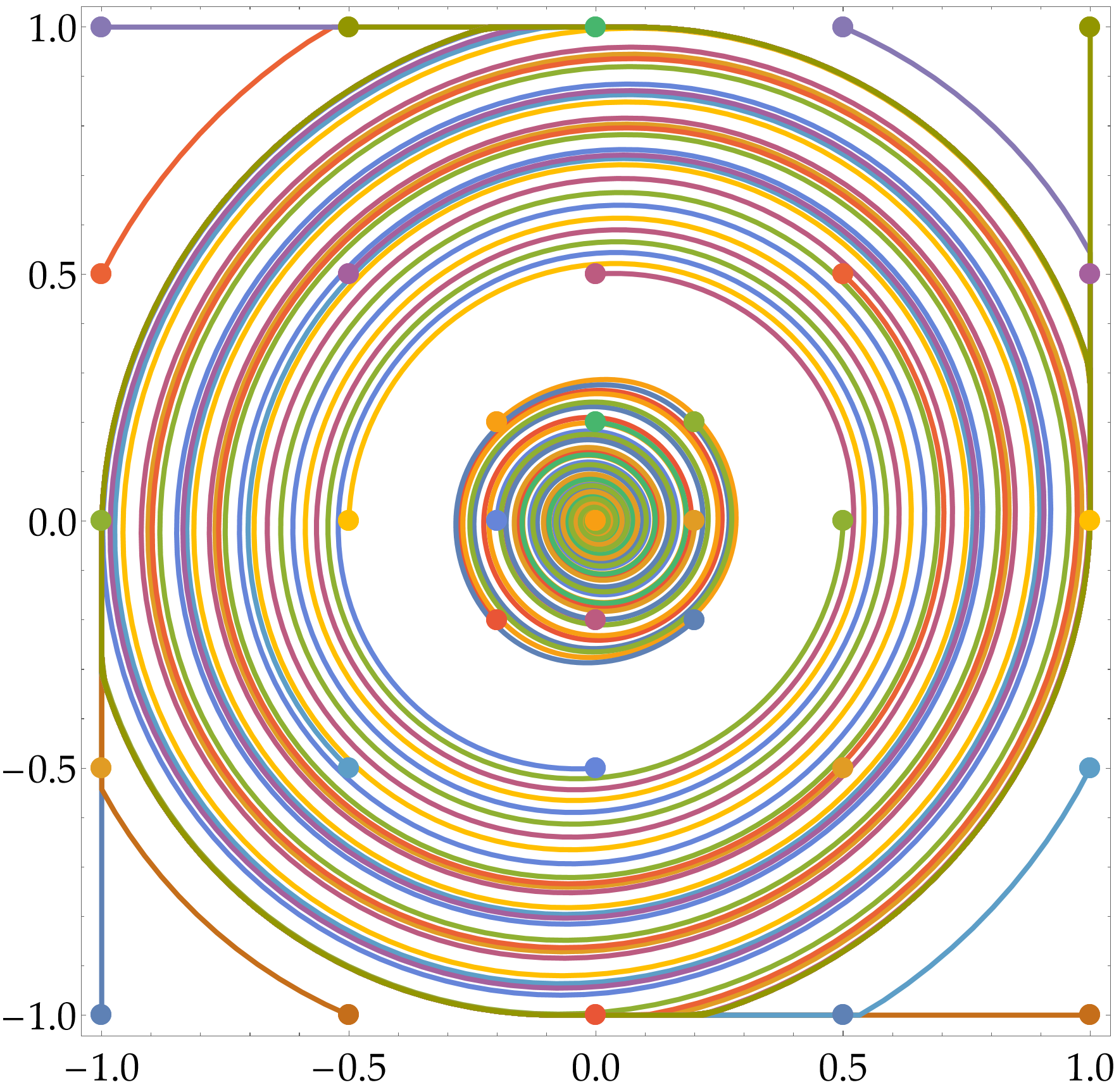}
    \caption{$f_2(\theta, \omega)$.}
  \end{subfigure}

  \caption{(Algorithm: EG/OGDA) 
  \textbf{(a)} we observe that for every initial conditions the algorithm converges to the same limit cycle with the only exception of $(0, 0)$ as for GDA in Figure \ref{fig:simulations:GDA}. \textbf{(b)} The behavior of the algorithm for $f_2(\theta, \omega)$ is again similar to the behavior of GDA as we can see in Figure \ref{fig:simulations:GDA} (b). There only two differences with GDA: (1) when initialized close to equilibrium, EG converges very fast, and (2) the region of attraction to the equilibrium is larger compared to GDA.}
  \label{fig:simulations:EG}
\end{figure}

\begin{figure}[!ht]
  \centering
  \begin{subfigure}[b]{0.38\textwidth}
    \centering
    \includegraphics[width=\textwidth]{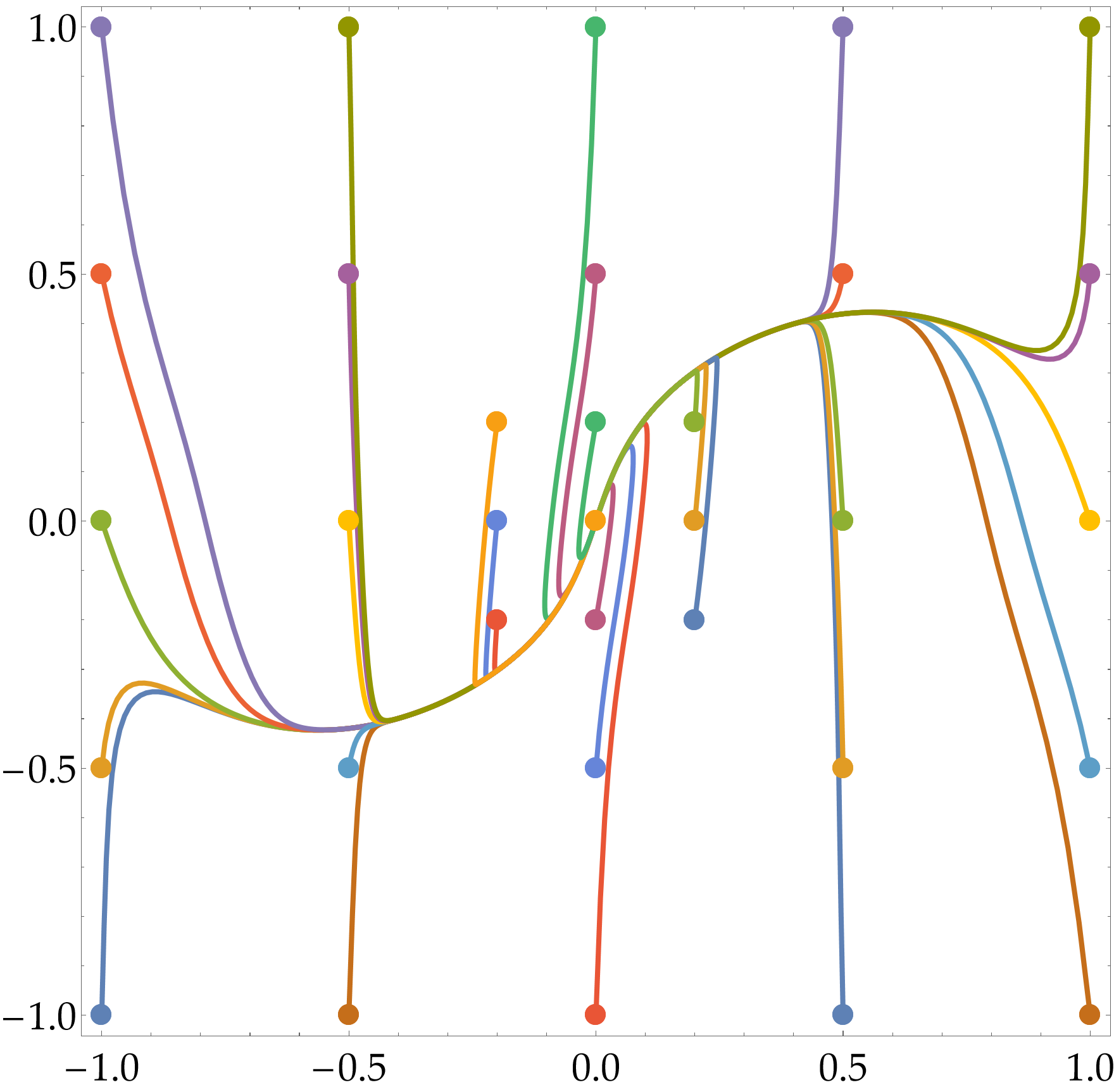}
    \caption{$f_1(\theta, \omega)$.}
  \end{subfigure}
  ~~~~~~~~~
  \begin{subfigure}[b]{0.38\textwidth}
    \centering
    \includegraphics[width=\textwidth]{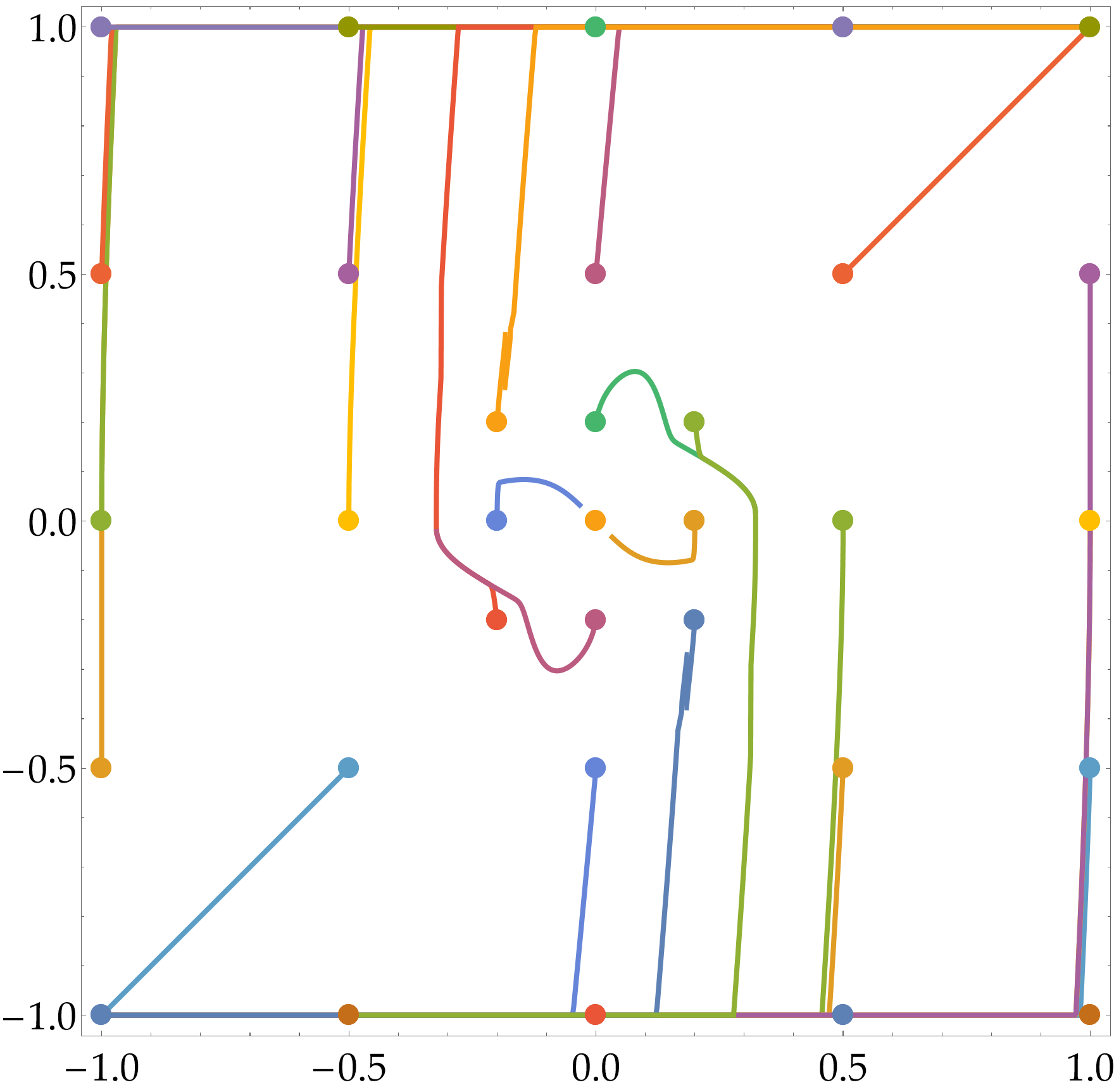}
    \caption{$f_2(\theta, \omega)$.}
  \end{subfigure}

  \caption{(Algorithm: FtR) 
  \textbf{(a)} We observe that for any initial condition the algorithm, in this example, converges to the equilibrium, in contrast with GDA or EG or OGDA. \textbf{(b)} In this example the behavior of the algorithm is very similar with GDA or EG or OGDA. If the algorithm is initialized far away from the equilibrium then it converges to either $(1, 1)$ or $(-1, -1)$ and none of them are equilibrium points. It is only when the algorithm is initialized next to the equilibrium that it converges to the equilibrium. Moreover, the algorithm needs to be initialized even closer than GDA to guarantee convergence. On the other hand, if the algorithm is initialized next to the equilibrium then it converges extremely fast, even faster than EG.}
  \label{fig:simulations:FtR}
\end{figure}

\begin{figure}[!ht]
  \centering
  \begin{subfigure}[b]{0.38\textwidth}
    \centering
    \includegraphics[width=\textwidth]{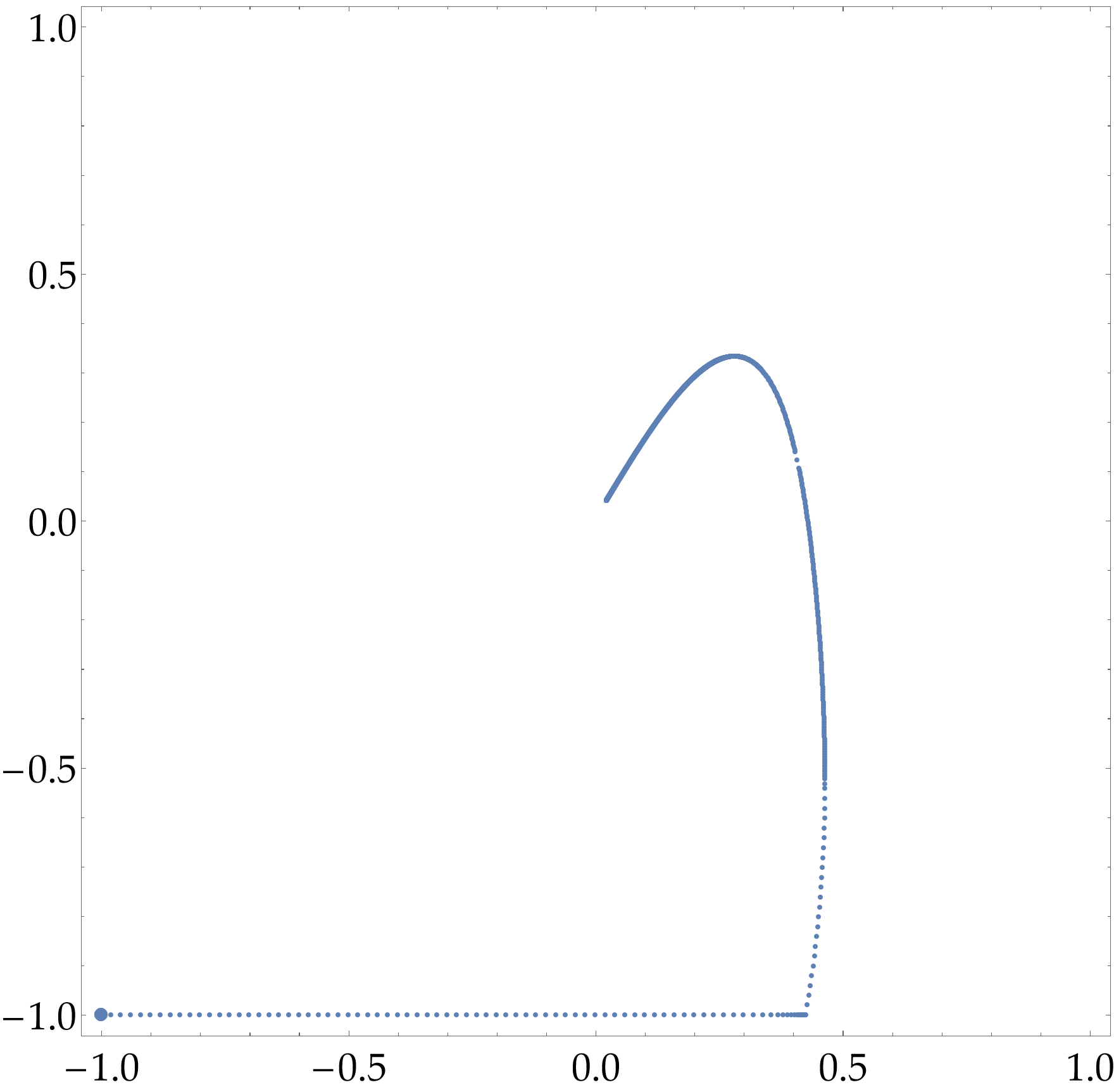}
    \caption{$f_1(\theta, \omega)$.}
  \end{subfigure}
  ~~~~~~~~~
  \begin{subfigure}[b]{0.38\textwidth}
    \centering
    \includegraphics[width=\textwidth]{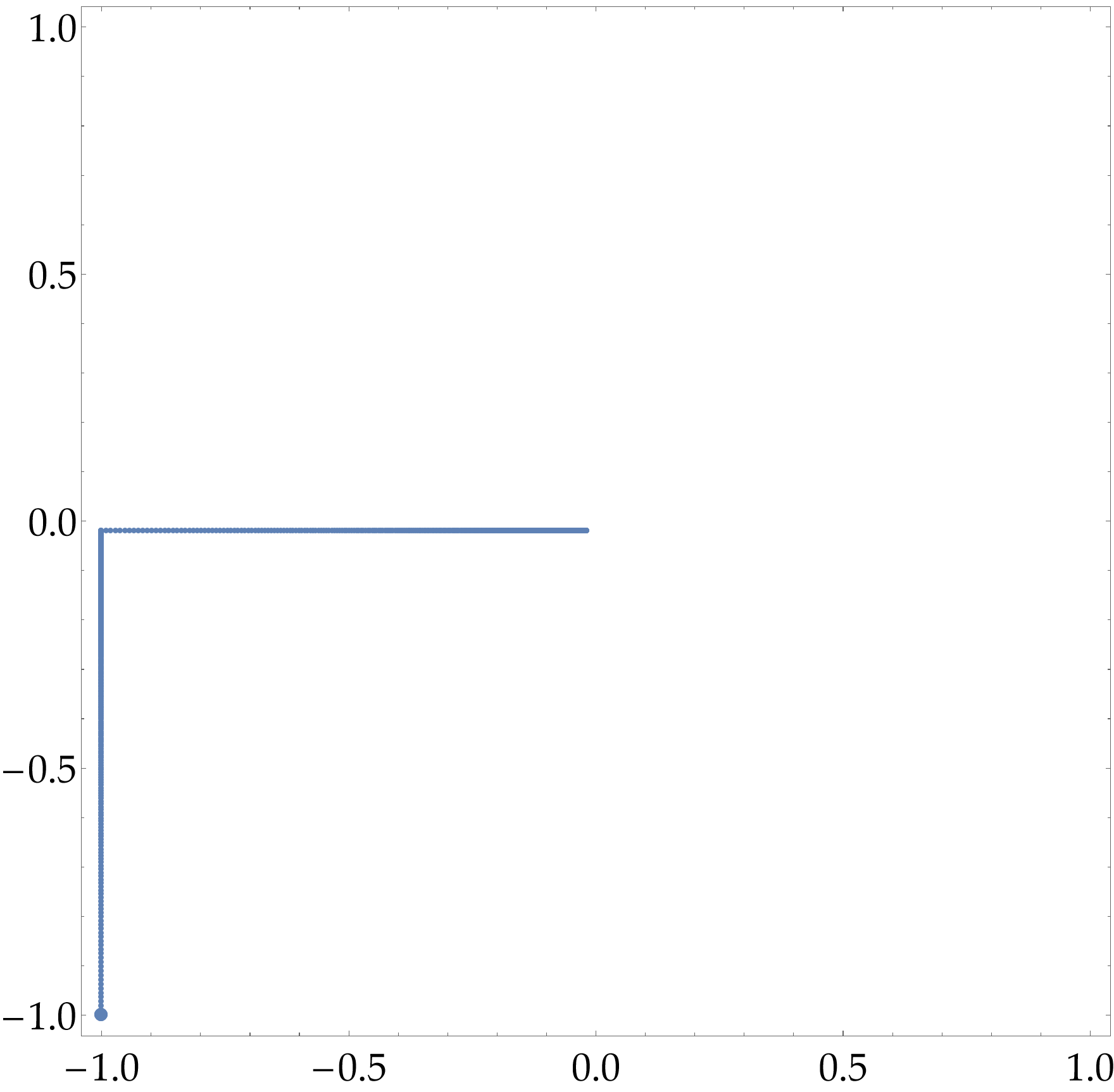}
    \caption{$f_2(\theta, \omega)$.}
  \end{subfigure}

  \caption{(Algorithm: STON'R) The STON'R algorithm is always initialized at $(-1, -1)$ independently of the objective function $f$. Hence, there is a good initialization for STON'R that is trivial to compute. This is contrast with the FtR algorithm that requires to be initialized close to the equilibrium. Such initialization might be as difficult to compute as finding the equilibrium itself. \textbf{(a)} we observe that the algorithm converges to the equilibrium almost directly and in particular it does not even need to spiral around the equilibrium. \textbf{(b)} The same for this example as well. The algorithm converges very fast and directly to the equilibrium although it is initialized far away from it. To the best of our knowledge, none of the known algorithms can achieve such a converge guarantee in this example.}
  \label{fig:simulations:STONR}
\end{figure}


\section{Solution Concept} \label{sec:prelimss}

First, a standard notation that we use  is this: if $m$ is a natural number then  $[m]=\{1,\ldots,m\}$. Although our main goal is to design optimization methods that have guaranteed convergence to local min-max equilibria of smooth objectives in the nonconvex-nonconcave setting, we choose formulate this problem in the language of non-monotone variational inequalities. This not only simplifies our definitions and notations but also makes our framework applicable to more general settings such as multi-player concave games that are easily captured from the framework of variational inequalities \cite{rosen1965existence}.
\medskip

\noindent \textbf{Variational Inequalities (VI).} For $K \subseteq \mathbb{R}^n$, consider a continuous map $V : K \to \R^n$. We say that $x \in K$ is a solution of the variational inequality $\mathrm{VI}(V, K)$ iff: $V(x)^{\top} \cdot (x - y) \le 0$ for all $y \in K$.
\medskip

It is well known that finding local min-max equilibria of smooth objectives can be expressed as a non-monotone VI problem. Specifically, consider the min-max optimization problem~\eqref{eq:costis1}, 
take $K=\Theta \times \Omega$ and simplify notation by using $x\in K$ to denote points $(\theta,\omega) \in K$. Call the subset of coordinates of $x$ identified with $\theta$ the ``\emph{minimizing} coordinates'' and the subset of coordinates of $x$ identified with $\omega$ the ``\emph{maximizing} coordinates.'' Then define $V: K \rightarrow \mathbb{R}^n$ as follows:
$$\text{
For $j\in [n]$: set $V_j(x) :=-\frac{\partial f(x)}{\partial x_j}$, if $j$ is minimizing,  and $V_j(x) :=\frac{\partial f(x)}{\partial x_j}$, otherwise.}$$
With these definitions, it is easy to see that computing $(\varepsilon,\delta)$-local min-max equilibria of smooth objectives, i.e.~points satisfying~\eqref{eq:local min costis2} and~\eqref{eq:local max costis2}, can be reduced to finding solutions to $\mathrm{VI}(V, K)$. (In fact, finding even an approximate VI solution $x$ satisfying $V(x)^{\top} (x-y)\ge -\alpha, \forall y \in K$, would suffice as long as $\alpha>0$ is small enough. For more details see Theorem 5.1 of~\cite{minmaxComplexityArxiv}.)


In view of the above, for the remainder of the paper we focus on solving non-monotone variational inequality problems. For simplicity of exposition throughout we will take our constraint set to be $K = [0, 1]^n$. In this case there is a simple characterization of the solutions to $\mathrm{VI}(V, K)$.

\begin{definition} \label{def:satisfied}
We call a coordinate $i$ at point $x \in [0, 1]^n$, 
\vspace{-5pt}
\begin{enumerate}[leftmargin=15pt]
    \item \textbf{zero-satisfied} if $V_i(x) = 0$,
    
    \item \textbf{boundary-satisfied} if $\left( V_i(x) \le 0\text{ and }x_i = 0\right)$ or $\left( V_i(x) \ge 0\text{ and }x_i = 1\right)$,
    
    \item \textbf{satisfied} if $i$ is  zero- or boundary- satisfied and \textbf{unsatisfied} if it is not satisfied.
\end{enumerate}
\end{definition}

\begin{lemma}[Proof in Appendix~\ref{sec:simple VI characterization lemma}] \label{lem:viSolutions}
  $x$ is a solution of $\mathrm{VI}(V, [0, 1]^n)$ iff  $j$ is satisfied at $x$, $\forall j \in [n]$.
\end{lemma}


\noindent Finally, in the rest of the paper we make the following assumptions for $V$:
\begin{itemize}[leftmargin=80pt]
  \item[\textbf{($\boldsymbol{\Lambda}$-Lipschitz)}] $~~~\norm{V(x) - V(y)}_2 \leq \Lambda \cdot  \norm{x - y}_2$, for all $x, y \in [0, 1]^n$.
  \item[\textbf{($\boldsymbol{L}$-smooth)}] $~~~\norm{J(x) - J(y)}_F \leq L \cdot  \norm{x - y}_2$, for all $x, y \in [0, 1]^n$.
\end{itemize}
where $J$ is the Jacobian of V, and $\norm{A}_F$ denotes the Frobenious norm of the matrix $A$.

\medskip

\section{STay-ON-the-Ridge: High-Level Description} \label{sec:stonr}
 \label{sec:stonr high level}

In this section we describe our algorithm and discuss the main design ideas  leading to its convergence properties presented in Section~\ref{sec:stonr:details}. 
As explained in the previous section, our goal is to find a point $x$ such that every coordinate $i \in [n]$ is satisfied at $x$ according to the Definition \ref{def:satisfied}. 

Our algorithm
is initialized at point $x{(0)}=(0,\ldots,0)$ where a number of coordinates may be unsatisfied. The goal of the algorithm is  to satisfy all unsatisfied coordinates one-by-one in lexicographic order (although, as we will see, coordinates may go from being satisfied to being unsatisfied in the course of the algorithm). We say that our algorithm ``starts epoch $i$ at point $x$'' iff all coordinates $\le i - 1$ are satisfied at $x$ and the algorithm's immediate goal is to find a point $x' \neq x$ that satisfies all coordinates $\le i$, namely:
\begin{quote}
  Goal of epoch $i$, starting at point $x$: find  $x' \neq x$  satisfying all coordinates $\le i$.
\end{quote}
We  now describe how the algorithm  tries to meet the afore-described goal. 
 So let us assume that, at time $t$, our algorithm starts epoch $i$ at point $x{(t)}$. Postponing full details to Section~\ref{sec:continuousDynamics}, where we describe our algorithm in detail, for simplicity of exposition  let us assume in this section that, at $x{(t)}$, all coordinates $\le i-1$ are zero-satisfied, i.e.~$V_j(x{(t)}) = 0$ for all $j  \le i-1$. To achieve the goal of epoch $i$ starting at $x{(t)}$, our algorithm will {\em try} to find a point $x' \neq x{(t)}$ where all coordinates $\le i-1$ remain zero-satisfied and  coordinate $i$ is also satisfied as follows:
\begin{itemize}
    \item First, it will try to find such a point in the connected subset $S^i(x{(t)}) \subseteq [0,1]^n$ that contains $x{(t)}$ and all points $z$ satisfying the following: (a) all coordinates $\le i-1$ are zero-satisfied at $z$, and (b) for all $j\ge i+1$, $z_j = x_j{(t)}$. 

    \item Next, let us describe how our algorithm   navigates $S^i(x{(t)})$ in the hopes of identifying a point $x' \neq x{(t)}$ where all coordinates $\le i$ are satisfied. A natural approach is  to run a continuous-time dynamics $\{z(\tau)\}_{\tau \ge 0}$ that is initialized at $z(0)=x{(t)}$ and moves inside $S^i(x{(t)})$. What are possible directions of movement for such dynamics so that it stays within $S^i(x{(t)})$? If the dynamics is at some point $z\in S^i(x{(t)})$, it will remain in this set if it moves, infinitessimally, in a unit direction $d$ satisfying the following constraints:
 %
    \begin{enumerate}
        \item $d_j=0$, for all $j\ge i+1$;~~~~{\em /* this is so that  (b) in the definition of $S^i(x{(t)})$ is maintained */} \label{costis constraint 1}
        
        \item $(\nabla V_j(z))^{\top} \cdot d=0$, for all $j\in \{1,\ldots,j-1\}$.~~~~{\em /* this is so that  (a)  is maintained */} \label{costis constraint 2}

    \end{enumerate}
    \smallskip
    
    Notice that~\ref{costis constraint 1} and~\ref{costis constraint 2} specify $n-1$ constraints on $n$ variables. We will place mild assumptions on $f$ so that there is a unique, up to a sign flip, unit direction satisfying these constraints. (Specifically see Assumption~\ref{a:1} in Section~\ref{sec:formalStatements}, where our main result is formally stated.) Moreover, we will specify a way to break ties so that we  choose one of the two unit directions satisfying our constraints. (Specifically this is done in part 3 of Definition~\ref{d:directionality1} in Section~\ref{sec:continuousDynamics}.) Let us denote by $D^i(z)$ the unit direction that our tie-breaking rule selects at $z$. 
    
    \item With the above choices, the continuous-time dynamics  $\dot{z}(\tau)=D^i(z(\tau))$, initialized at $z(0)=x{(t)}$, is well-defined. We follow this dynamics until the earliest time that one of the following happens (if both events happen at the same time we will say that the good event happened):
    \begin{itemize}
        \item (Good Event): the dynamics stops at a point $x' \neq x{(t)}$ where coordinate $i$ is satisfied;
        \item (Bad Event): the dynamics stops at a point $x'$ lying on the boundary of $[0,1]^n$ (and if it were to continue it would violate the constraints).
    \end{itemize}
\end{itemize}
So  we have described what our algorithm does if, at time $t$, it starts epoch $i$ at some point $x{(t)}$. Suppose $x'$ is the point where the continuous-time dynamics executed during epoch $i$ terminates. If the good event happened, coordinate $i$ is satisfied at $x'$, and our algorithm starts epoch $i+1$ at $x'$. If the bad event happened, our algorithm will in fact {\em start epoch $i-1$} at point $x'$. What does this mean? That it will run the continuous-time dynamics corresponding to epoch $i-1$ on the set $S^{i-1}(x')$ starting at $x'$ in order to find some point $x'' \neq x'$ where all coordinates $\le i-1$ are satisfied. It may fail to do this, in which case it will start epoch $i-2$ next. Or it may succeed, in which case, it will start epoch $i$, and so on so forth until (as we will show!) all coordinates will be satisfied. The high-level pseudocode of our algorithm is given in Dynamics \ref{dyn:Sperner high level}.



\begin{algorithm}
  \caption{STay-ON-the-Ridge (STON'R) --- High-Level Description}\label{dyn:Sperner high level}
 \begin{algorithmic}[1]
 \STATE Initially $x{(0)} \leftarrow (0,\ldots,0)$, $i \leftarrow 1$, $ t \leftarrow 0$.

 \WHILE {$x{(t)}$ is not a VI solution}
        \STATE Initialize epoch $i$'s continuous-time dynamics, $\dot{z}(\tau)=D^i(z(\tau))$, at $z(0)=x{(t)}$.
        \WHILE{exit condition of this dynamics has not been reached}
        \smallskip
          \STATE Execute $\dot{z}(\tau)=D^i(z(\tau))$ forward in time.
        \smallskip
        \ENDWHILE
        \STATE Set $x{(t+\tau)}=z(\tau)$ for all $\tau \in [0, \tau_{\rm exit}]$ (where $\tau_{\rm exit}$ is time exit condition was met).
        \smallskip
        \IF{ $x{(t+\tau_{\rm exit})} \neq x{(t)}$ and coordinate $i$ is satisfied at $x{(t+\tau_{\rm exit})}$}
        \STATE Update the epoch $i \leftarrow i + 1$.
        \ELSE 
        \STATE (Bad event happened so) move to the previous epoch $i \leftarrow i - 1$.
        \ENDIF
        \STATE Set $t \leftarrow t + \tau_{\rm exit}$.
\ENDWHILE
  \STATE \textbf{return }$x(t)$
  \end{algorithmic}
\end{algorithm}


At this point we have described an algorithm that explores the space in a natural way in its effort to satisfy coordinates, but it is unclear why it would succeed in eventually satisfying  all of them, how it would  escape cycles, and how it would not get stuck at non-equilibrium points. Importantly, there is no quantity that seems to be consistently improving during the execution of the algorithm. For example, the number of satisfied coordinates might decrease during the algorithm's execution.
\begin{quote}
  How we can show this algorithm converges since no quantity seems to be consistently improving during its execution?
\end{quote}
To show the convergence of our algorithm we need to use a different kind of argument than the classical arguments used in optimization which are based on some quantity improving. In particular, we use a topological argument that we describe in Section \ref{sec:ppadArgument}. 




\section{A Topological Argument of Convergence} \label{sec:ppadArgument}

  As discussed in Section~\ref{sec:stonr high level}, there seems to be no clear potential function that decreases in the course of our algorithm's execution, which we could track to show that it converges. Indeed, even the number of satisfied coordinates might decrease in the course the algorithm's execution as we explained in Section~\ref{sec:stonr high level}. So how we can show that our algorithm converges?
  
  Our main idea is to use topological arguments that have been successfully employed to show the convergence of other equilibrium computation algorithms. In the celebrated \cite{lemke1964equilibrium} algorithm, e.g., the following argument is used to prove the algorithm's convergence. 
  
  \begin{lemma} \label{lem:ppad}
    Let $G = (N, E)$ be a directed graph such that every node has in-degree at most $1$ and out-degree at most $1$. If there exists some node $v \in N$ with in-degree $0$ and out-degree $1$, then there is unique directed path starting at $v$ and ending at some $v' \in N$ that has in-degree $1$ and out-degree $0$.
  \end{lemma}
  
  \begin{figure}[!ht]
      \centering
      \includegraphics[width=0.7\textwidth]{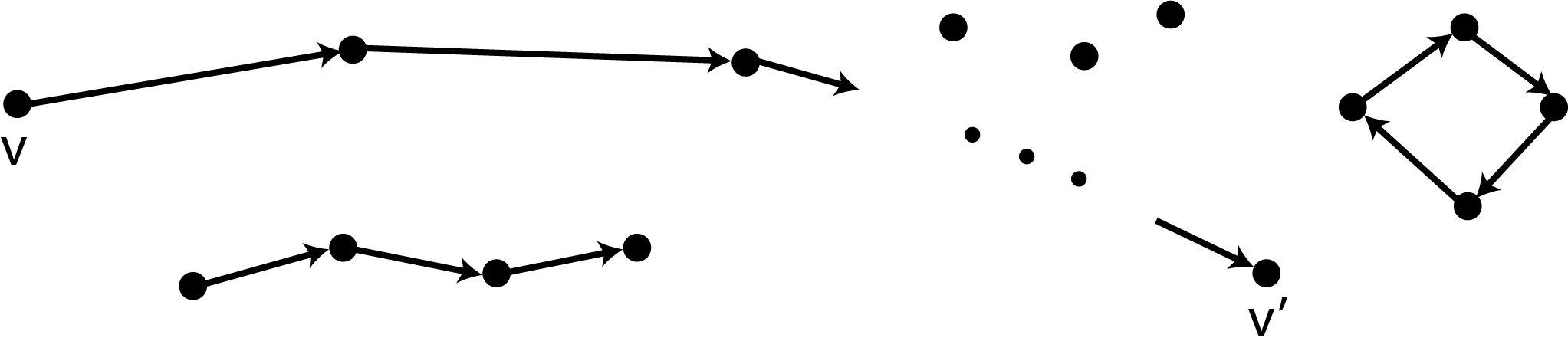}
      \caption{A directed graph whose nodes have in-degree and out-degree at most $1$ is a collection of directed paths, directed cycles, and isolated nodes. Hence, if a node $v$ has in-degree $0$ and out-degree $1$ then it has to be the start of a directed path that must end at a node $v'$ after a finite number of steps. 
      }
      \label{fig:topologicalArgument}
  \end{figure}
  
  \noindent The proof of Lemma~\ref{lem:ppad} is straightforward, as Figure~\ref{fig:topologicalArgument} illustrates. The lemma suggests the following recipe for proving the convergence of some deterministic, iterative algorithm, with update rule $v_{t+1} \leftarrow F(v_t)$, whose iterates lie in a finite set $N$:
  \begin{enumerate}[leftmargin=10pt]
    \item Define a directed graph $G$ whose vertex set is $V$ and edge set is $E=\{(u,v)~|~u\neq v~\text{and}~v=F(u)\}$, i.e.~there is a directed edge from $u$ to $v$ iff $v$ is different from $u$ and $v$ is reached after one iteration of the algorithm starting at $u$.
    \item Argue that every vertex of $G$ has in-degree $\le 1$. It is clear that every vertex has out-degree $\le 1$.
    \item Show that the algorithm can be initialized at some $v_0$ that has in-degree $0$ and out-degree $1$.
    \item Employ Lemma~\ref{lem:ppad} to argue that if the algorithm is initialized at $v_0$ it must, eventually, arrive at some node $v_{\rm end}$ whose out-degree is $0$. Out-degree $0$ means that $v_{\rm end}=F(v_{\rm end})$.
    \item The above prove that if the algorithm starts at $v_0$ it is guaranteed to converge. 
    \end{enumerate}
  
  
  Having this topological argument in place we are ready to formally describe our algorithm and argue its convergence. In the course of our description we will be sure to specify a finite set of points $V$ that will act as the nodes of the finite graph that we will construct to employ the above convergence argument. 
  Intuitively, these are all the points at which our algorithm can possibly start an epoch. The map $F(\cdot)$ that we use to construct our graph is the outcome of the continuous-time process that our algorithm execute when it starts an epoch at such a point.
  
\section{Detailed Description of STON'R and Main Result} \label{sec:stonr:details}

We provide a formal description of our algorithm, state our main convergence theorem, and provide the main components of its proof building on the ideas from Section~\ref{sec:ppadArgument}.


\subsection{STON'R: Detailed Description} \label{sec:continuousDynamics}

We provide a detailed description of our algorithm, building on the framework from Section~\ref{sec:stonr high level}. To simplify our exposition in that section, we only described the behavior of the algorithm when it starts an epoch at some  $x$ where coordinates $\le i-1$ are zero-satisfied and its goal is to identify some $x' \neq x$ at which coordinates $\le i$ are satisfied. To achieve this goal the algorithm executed a continuous-time dynamics constrained by keeping all coordinates $\le i-1$ zero-satisfied. However, in the course of its execution the algorithm might be hitting the boundary in its effort to satisfy coordinates. So, in the general case, when it starts a new epoch, some coordinates will be zero-satisfied and some will be boundary-satisfied. As such, what the algorithm will do in the general case during some epoch is execute a continuous-time dynamics constrained by keeping the zero-satisfied coordinates zero-satisfied and the boundary-satisfied coordinates at the right boundary. 

More precisely, the epochs of our algorithm are, in fact, indexed not only by some coordinate $i$ but also by a subset  of coordinates $S \subseteq [i-1]$ that are zero-satisfied at the point $x$ where the epoch starts. The goal of the epoch is the following.
\vspace{-5pt}
\begin{quote}
  Goal of epoch $(i, S)$, starting at point $x$ (where $S \subseteq [i-1]$, coordinates in $S$ are zero-satisfied and coordinates in $[i-1]\setminus S$ are boundary-satisfied): find  $x' \neq x$ where  all coordinates $\le i$ are satisfied, all coordinates in $S$ are zero-satisfied and all coordinates in $[i-1]\setminus S$ are boundary-satisfied.
\end{quote}
\vspace{-5pt}


As in our high-level description in Section \ref{sec:stonr high level}, epoch $(i, S)$ starting at $x$ might achieve its goal or end before it achieves its goal. In both cases, a new epoch will start. Now, how does the algorithm try to achieve its goal in some epoch? Similar to the special case discussed in Section \ref{sec:stonr high level}, in the general case considered here the algorithm will execute a continuous-time dynamics that maintains all the coordinates $j \in S$ zero-satisfied, all the coordinates $j \in [i-1] \setminus S$ boundary-satisfied, and leaves all coordinates $[n] \setminus (\{i\} \cup S)$ unchanged. 
The following definition captures the tangent unit vector of the curve that this continuous-time dynamics travels.

\begin{definition} \label{d:directionality1}
Given $i \in [n]$, a set of coordinates $S=\{s_1,\ldots,s_m\} \subseteq [i-1]$, and some point $x$,
we say that a unit vector $d\in\mathbb{R}^n$ is {\em admissible} iff it satisfies the following:
\vspace{-10pt}
\begin{enumerate}
    \item $d_j=0$, for all $j \notin S\cup \{i\}$.
    \item $\nabla V_j(x)^\top \cdot d = 0$, for all $j \in S$.
    \item The sign of $\begin{vmatrix}
\frac{\partial V_{s_1}(x)}{\partial x_{s_1}} & \frac{\partial V_{s_2}(x)}{\partial x_{s_1}}&\ldots &\frac{\partial V_{s_m}(x)}{\partial x_{s_1}} & d_{s_1}\\ 
\vdots& \vdots & \vdots & \vdots & \vdots\\
\frac{\partial V_{s_1}(x)}{\partial x_{s_m}} & \frac{\partial V_{s_2}(x)}{\partial x_{s_m}}&\ldots &\frac{\partial V_{s_m}(x)}{\partial x_{s_m}} & d_{s_m}\\ 
\frac{\partial V_{s_1}(x)}{\partial x_{i}} & \frac{\partial V_{s_2}(x)}{\partial x_{i}}&\ldots &\frac{\partial V_{s_m}(x)}{\partial x_{i}} & d_{i}\\ 
\end{vmatrix}$ equals ${\rm sign}\left((-1)^{|S|} \right)$.
\end{enumerate}
If there is a unique unit direction satisfying the above constraints, we denote that direction $D^i_S(x)$.
\end{definition}
We will place mild assumptions on $V$ so that $D^i_S(x)$ is (uniquely) defined for all $x \in [0, 1]^n$ where coordinates $S$ are zero-satisfied. (Specifically see Assumption~\ref{a:1} in Section~\ref{sec:formalStatements}, where our main result is formally stated.) With this definition in place, when our algorithm starts epoch $(i,S)$ at point $x$, it will execute the continuous-time dynamics $\dot{z}(\tau)=D^i_S(z(\tau))$, initialized at $z(0)=x$, forward in time. The algorithm executes this dynamics until the earliest time $\tau_{\rm exit}$ such that $z(\tau_{\rm exit})$ is an exit point, as per the definition below. 




\begin{definition} \label{def:exitConditions}
    Suppose $i\in [n]$, $S \subseteq [i-1]$ and $x'$ is a point where coordinates in $S$ are zero-satisfied and coordinates in $[i-1]\setminus S$ are boundary-satisfied. Then $x'$ is an {\em exit point} for epoch $(i, S)$ iff it satisfies one of the following:
    \vspace{-6pt}
    \begin{itemize}
        \item \textbf{(Good Exit Point)}: 
        Coordinate $i$ is satisfied at $x'$, i.e., $V_i(x') = 0$, or $x'_i = 0$ and $V_i(x') < 0$, or $x'_i = 1$ and $V_i(x') > 0$.
        \vspace{-4pt}
        \item \textbf{(Bad Exit Point)}:  For some $j \in S \cup \{i\}$, it holds that $(D^i_S(x'))_j > 0$ and $x'_j = 1$, or $(D^i_S(x'))_j < 0$ and $x'_j = 0$; in other words, if the dynamics for epoch $(i,S)$ were to continue from $x'$ onward, they would violate the constraints.
        \vspace{-4pt}
        \item \textbf{(Middling Exit Point)}: For some $j \in [i-1]\setminus S$, it holds that $V_j(x') = 0$ and one of the following holds: $\nabla V_j(x')^{\top} D^i_S(x') > 0$ and $x'_j = 0$, or $\nabla V_j(x')^{\top} D^i_S(x') < 0$ and $x'_j = 1$; in other words, if the dynamics for epoch $(i,S)$ were to continue from $x'$ onward, some boundary-satisfied coordinate would become unsatisfied.
    \end{itemize}
\end{definition}
\vspace{-4pt}
We will place mild assumptions on $V$ so that there can be a unique $j$ triggering the condition of Bad Exit Point in Definition~\ref{def:exitConditions} and there can be a unique $j$ triggering the Middling Exit Point condition, when $x'$ is a point where coordinates in $S$ are zero-satisfied and coordinates in $[i-1]\setminus S$ are boundary-satisfied. (Specifically, see Assumptions~\ref{a:2} and~\ref{a:3} in Section~\ref{sec:formalStatements}). Here are the actions that we need to take if one of the above exit conditions has been reached happen.
\medskip

\noindent \textbf{Action at Good Events.} In case of a good event, we start  epoch $(i + 1, S')$ at $x'$, where $S' = S \cup \{i\}$, if $i$ is zero-satisfied at $x'$, and $S' = S$, if $i$ is boundary-satisfied at $x'$.
\medskip

\noindent \textbf{Action at Bad Events.} In case of a bad event, note that the coordinate $j$ responsible for the condition in the bad event to trigger must lie in $S \cup \{i\}$ because in all other coordinates $(D^i_S(x'))_j = 0$ by definition. Depending on which $j$ triggers the condition of the bad event we do one of the following:

\noindent (1) if the triggering $j=i$, then we start  epoch $(i - 1, S \setminus \{i - 1\})$ at $x'$.

\noindent (2) if the triggering $j\neq i$, then we start  epoch $(i, S \setminus \{j\})$ at $x'$.
\medskip



\noindent \textbf{Action at Middling Events.} In the case of a middling event, we start epoch $(i, S \cup \{j\})$ at $x'$ (because the  coordinate $j$ that trigger this event is both zero- and boundary-satisfied at $x'$ so we add it to $S$ to constrain the dynamics to zero-satisfy it next.).

\medskip

Combining all the aforementioned ideas we describe our algorithm in Dynamics \ref{dyn:Sperner}. In Section~\ref{ref:bilinear} we do a step-by-step analysis of what the algorithm would do for a simple min-max optimization problem.

\begin{algorithm}
  \caption{STay-ON-the-Ridge (STON'R)}\label{dyn:Sperner}
 \begin{algorithmic}[1]
 \STATE Initially $x{(0)} \leftarrow (0,\ldots,0)$, $i \leftarrow 1$, $S \leftarrow \emptyset$, $t \leftarrow 0$.

 \WHILE {$x{(t)}$ is not a VI solution}
        \STATE Initialize epoch $(i, S)$'s continuous-time dynamics, $\dot{z}(\tau) = D^i_S(z(\tau))$, at $z(0)=x{(t)}$.
        \WHILE{${z}(\tau)$ is not an exit point as per Definition \ref{def:exitConditions}}
        \smallskip
          \STATE Execute $\dot{z}(\tau)=D^i_S(z(\tau))$ forward in time.
        \smallskip
        \ENDWHILE
        \STATE Set $x{(t+\tau)}=z(\tau)$ for all $\tau \in [0, \tau_{\rm exit}]$ \textit{(where $\tau_{\rm exit}$ is earliest time ${z}(\tau)$ became an exit point)}.
        \smallskip
        \IF{$x{(t+\tau_{\rm exit})}$ is (Good Exit Point) as in Definition \ref{def:exitConditions}}
        \IF{$i$ is zero-satisfied at $x{(t+\tau_{\rm exit})}$}
        \STATE Update $S \leftarrow S \cup \{i\}$.
        \ENDIF
        \STATE Update $i \leftarrow i + 1$.
        \ELSIF{$x{(t + \tau_{\rm exit})}$ is a (Bad Exit Point) as in Definition \ref{def:exitConditions} for $j = i$}
        \STATE Update $i \leftarrow i - 1$ and $S \leftarrow S \setminus \{i - 1\}$.
        \ELSIF{$x{(t + \tau_{\rm exit})}$ is a (Bad Exit Point) as in Definition \ref{def:exitConditions} for $j \neq i$}
        \STATE Update $S \leftarrow S \setminus \{j\}$.
        \ELSIF{$x{(t + \tau_{\rm exit})}$ is a (Middling Exit Point) as in Definition \ref{def:exitConditions} for $j < i$}
        \STATE Update $S \leftarrow S \cup \{j\}$.
        \ENDIF
        \STATE Set $t \leftarrow t + \tau_{\rm exit}$.
\ENDWHILE
  \STATE \textbf{return }$x(t)$
  \end{algorithmic}
\end{algorithm}

\subsection{Our Assumptions and Our Main Theorem} \label{sec:formalStatements}
We next present the assumptions on $V$ that are needed for our convergence proof. We discuss these assumptions further in Appendix~\ref{app:assumptions} explaining why they are mild.

\begin{assumption}\label{a:1}
There exist positive real numbers $0<\sigma_{\min} < \sigma_{\max}$ so that the following holds: 
For all $x \in [0,1]^n$ and set of coordinates $S=\{s_1,\ldots,s_m\}$, if $V_\ell(x) =0$ for all $\ell \in S$, then the singular values of the $m \times m$ matrix
\[
J_S^K(x):=\begin{pmatrix}
\frac{\partial V_{s_1}(x)}{\partial x_{s_1}} & \ldots & \frac{\partial V_{s_1}(x)}{\partial x_{s_m}} \\
\vdots & & \vdots \\
\frac{\partial V_{s_m}(x)}{\partial x_{s_1}} & \ldots & \frac{\partial V_{s_m}(x)}{\partial x_{s_m}} \\
\end{pmatrix}
\]
are greater than $\sigma_{\text{min}}$ and less than $\sigma_{\text{max}}$.
\end{assumption}

\begin{assumption}\label{a:2}
For any $x\in [0,1]^n$, set of coordinates $S=\{s_1,\ldots,s_m\}$, and $i \notin {S}$:  If $\left(V_\ell(x) = 0\text{ for all }\ell \in S\right)$ and $\left(x_\ell \in \{0,1\}~~\text{ for all }\ell \notin {S} \cup \{i\}\right)$ then there is at most one coordinate $j \in S \cup \{i\}$ such that $x_j = 0 $ or $x_j = 1$.
\end{assumption}
One may ensure Assumption~\ref{a:2} holds by restricting the domain of each variable $i$ in the subset $[\alpha_i,1 - \beta_i]$ of $[0,1]$, where $\alpha_i,\beta_i$ are uniformly random in $[0,\epsilon]$. For details we refer to Section~\ref{app:assumptions}.
\begin{assumption}\label{a:3}
For all: (i) collection of coordinates $S=(s_1,\ldots,s_m)$, (ii) coordinate $i \notin {S}$, (iii) point $x\in [0,1]^n$ such that ($V_\ell(x) = 0$ for all $\ell \in S$)
and ($x_\ell \in \{0,1\}$ for all $\ell \notin {S} \cup \{i\}$), and (iv) vector $(d_{s_1},\ldots,d_{s_m},d_i)$ satisfying the equations,
\[\nabla_{S\cup\{i\}}V_j(x)^\top \cdot
    \left(d_{s_1},\ldots,d_{s_m},d_i\right) = 0
    \text{ for all }j \in S,\]
we have that $d_j \neq 0$ if $x_j =0$ or $x_j =1$.
\end{assumption}

\noindent We are now ready to state our main theorem whose is presented in Appendix \ref{sec:convergence_proof}.

\begin{theorem}\label{t:main}
  Under Assumptions~\ref{a:1},~\ref{a:2}, and ~\ref{a:3}, there exists some $\bar{T}=\bar{T}(\sigma_{\rm min}, \sigma_{\rm \max}, n, L, \Lambda)>0$ such that {\sc STay-On-the-Ridge} (Dynamics~\ref{dyn:Sperner}) will stop, at some time $T \le \bar{T}$, at some point $x(T) \in [0, 1]^n$ that is a solution of $\mathrm{VI}(V, [0, 1]^n)$.
\end{theorem}

\begin{remark}[Discrete-time Algorithm]
  It is possible to combine the proof of Theorem \ref{t:main} with standard numerical analysis techniques to show the convergence of a simple discrete version of the dynamics assuming that the step size is small enough. For more details about this we refer to Appendix \ref{sec:discrete}.
\end{remark}





\subsection{Sketch of Proof of Theorem~\ref{t:main}}\label{s:sketch_proof}

A sketch of our proof of Theorem \ref{t:main} comes from the recipe that we described in Section \ref{sec:ppadArgument}.
\begin{enumerate}[leftmargin=10pt]
  \item We start with the definition of the set of nodes $N$. The set $N$ contains triples of the form $(i, S, x)$ where $i \in [n]$, $S$ is a subset of $[i - 1]$ and $x \in [0, 1]^n$ that satisfies the following:
  \begin{itemize}
    \item (a) all coordinates in $S$ are zero-satisfied, (b) all coordinates in $[i - 1] \setminus S$ are boundary-satisfied, (c) $x_j = 0$ for all $j \geq i + 1$, and either (d1) $x_i = 0$ or (d2) $x$ is an exit point for epoch $(i, S)$ according to Definition \ref{def:exitConditions} \footnote{The actual set of nodes that we used in the proof does not contain the information of $i$ and $S$ but we refer to the Appendix for the exact proof.}.
  \end{itemize}
  
  We show in the Appendix the size of $N$ is finite (see Lemma~\ref{l:pivot_finite}).

  
  Next we describe a mapping $F : N \to N$ as in Section \ref{sec:ppadArgument}. Let $(i, S, x) \in N$ we use the dynamics $\dot{z} = D^{i}_S(z)$ with initial condition $z(0) = x$ and we find the minimum time $\tau_{\rm exit}$ such that $z(\tau_{\rm exit})$ is an exit point. We then update $i$, $S$ to $i'$, $S'$ according to the rules for actions on exit points of Section \ref{sec:continuousDynamics} and we define $F((i, S, x)) = (i', S', z(\tau_{\rm exit}))$. As we show in the appendix the dynamics $\dot{z} = D^{i}_S(z)$ have a unique solution under our assumptions and hence $F$ is well defined (Lemma~\ref{l:outdegree}).
  
  The set $N$ and the mapping $F$ define the directed graph $G$ as we described in Section \ref{sec:ppadArgument} that is guaranteed to have vertices with out-degree at most $1$. We also show that any $v \in V$ with out-degree 0 is an equilibrium point (Lemma~\ref{l:outdegree}).
  
  \item To show that the in-degree is at most $1$ too we show that we can actually solve the dynamics backwards in time. In particular, if we specify $z(0)$ and there is the smallest time $\tau_{\rm exit}$ such that $z(- \tau_{\rm exit})$ is an exit point then $z(- \tau_{\rm exit})$ is uniquely determined. This means that there exists $F^{-1} : N \to N$ such that if $v' = F(v)$ then $F^{-1}(v') = v$ which means that no vertex in $N$ can have in-degree more than $1$ (Lemma~\ref{l:indegree}).
  
  \item We show that $v_0 = (1, \emptyset, (0, \dots, 0)) \in N$. Also, if run the dynamics $\dot{z} = D^{1}_{\emptyset}(z)$ backwards in time starting at $z(0) = 0$ then we get outside $[0, 1]^n$ and so $v_0$ has in-degree $0$. We also show that the dynamics $\dot{z} = D^{1}_{\emptyset}(z)$ can move forward in time and stay inside $[0, 1]^n$ so $v_0$ has out-degree $1$ (Lemma~\ref{l:0}).
  \item The above show that our algorithm converges.
\end{enumerate}

\bibliographystyle{alpha}
\bibliography{ref}




\newpage
\appendix

\section*{Appendix}

\section{Discussion Of Assumptions~\ref{a:1},~\ref{a:2} and~\ref{a:3}}\label{app:assumptions}

In this section we discuss about the generality of our Assumptions \ref{a:1}, \ref{a:2}, and \ref{a:3}. We follow a general recipe in our arguments. In particular, we consider any VI problem $\mathrm{VI}(V, K)$ that does not satisfy some of our Assumptions, then we argue that there exists a small random perturbation $\mathrm{VI}(\tilde{V}, \tilde{K})$ of $\mathrm{VI}(V, K)$ such that:
(1) any approximate solution of $\mathrm{VI}(\tilde{V}, \tilde{K})$ is also an approximate solution of $\mathrm{VI}(V, K)$ with slightly higher approximation loss, and (2) $\mathrm{VI}(\tilde{V}, \tilde{K})$ satisfies all our Assumptions.

The arguments that we present in the next sections are heuristic but we conjecture that our statements are true in general which we leave as an interesting open problem. The main component that we miss towards this direction is the following: we can so that for a particular problem $\mathrm{VI}(V, K)$ if a particular point $x \in K$ violates some of the assumptions then a random perturbation suffices to make $x$ satisfy all the assumptions. The argument is missing is to show that these random perturbations produce instances that satisfy the assumptions for every point in the space. As we said before we conjecture that this is actually true and we have verified our conjecture in some simple experiments.

\subsection{Assumption~\ref{a:1}}

To understand this assumption, consider the instance $V(x)=3x^2$ which is the simplest single-dimensional VI problem violating our assumption. This corresponds to a local maximization problem with objective function $f(x)=x^3$ as per our discussion in Section~\ref{sec:prelimss}.  In this case, at $x = 0$ we have that $V(0) = 0$ and $V'(0) = 0$ at the same time and hence Assumption ~\ref{a:1} is violated.
However, it is easy to perturb $f$ and $V$ in this problem to a problem that does not have this issue. We can simply add to $f$ a periodic function, e.g., $\alpha \cdot \sin(x + \psi)$, with parameter $\alpha$ very small and in particular $\alpha \le \eps$ and we suppose that we chose $\psi$ uniformly. 
Let $\tilde{f}$ be the modified maximization objective, i.e., $\tilde{f}(x) = f(x) + \alpha \cdot \sin(x + \psi)$. It is not hard to see that any stationary point of $\tilde{f}$ is also approximate stationary points of $f$ and that the probability that $\tilde{f}$ has both first and second derivatives small at the same time is close to zero. This example suggests that in single-dimensional problems adding a periodic function with small magnitude and random period can produce an instance that satisfies Assumption~\ref{a:1}, while preserving the set of solutions.

In higher dimensions the situation is more complicated and a formal argument to show that a \textit{regularization} procedure, as the one we described above, exists becomes more technically challenging. Our conjecture though is that such a procedure exists for high-dimensions as well. To support this conjecture we ran some simple experiments with objective functions that do not satisfy \ref{a:1} and we observe that indeed small random perturbations always produce objective functions that satisfy Assumption \ref{a:1}. A theoretical proof for the possibility of this \textit{regularization} approach is a very interesting open problem.

\subsection{Assumption~\ref{a:2}}

We can ensure that Assumption~\ref{a:2} holds by restricting the domain of each variable $i$ in the subset $[\alpha_i,1 - \beta_i]$ of $[0,1]$, where $\alpha_i,\beta_i$ are uniformly random in $[0,\epsilon]$. Is we choose $\alpha_i$ and $\beta_i$ to be very small, a solution to the VI problem in the restricted domain, corresponds to a $\Theta(\epsilon)$-approximate one in the original domain. Moreover, it is not hard to see that due to the randomness in the $\alpha_i$'s and the $\beta_i$'s, Assumption~\ref{a:2} holds with probability $1$ in the new domain (with the natural adjustment of the assumption statement, taking $\alpha_i$ and $1-\beta_i$ be the boundary values for each coordinate $i$).

As an example, consider the curve $C = \{x\in [0,1]^n \text{ such that } V_1(x)=0,\ldots ,V_{n-1}(x)=0\}$ and assume (because of Assumption~\ref{a:1}) that for all $x \in C$ the matrix 
\[J(x):= \begin{pmatrix}
\frac{\partial V_1(x)}{\partial x_1}& \ldots & \frac{\partial V_1(x)}{\partial x_n}\\
\vdots & & \vdots\\
\frac{\partial V_{n-1}(x)}{\partial x_1}& \ldots & \frac{\partial V_{n-1}(x)}{\partial x_n}
\end{pmatrix}
\]
admits singular values greater than $\sigma_{\text{min}}$ and smaller than $\sigma_{\text{max}}$. If the boundaries $[\alpha_i,1-\beta_i]$ for each coordinate $i$ are selected uniformly at random from the interval $[0,\epsilon]$, then with high probability the curve $C$ hits the random rectangle $[\alpha_1,1-\beta_1] \times \cdots \times [\alpha_n,1-\beta_n]$ only in \textit{pure facets} (only one coordinate $i$ equals $\alpha_i$ or $1-\beta_i$).


\subsection{Assumption~\ref{a:3}}

We argue about the generality of Assumption \ref{a:3} using the same idea as before. We argue that there exists a small random perturbation of every problem so that the resulting VI satisfies Assumption \ref{a:3} with high probability. In particular, consider any VI problem with map $V(x)$ and define $\tilde{V}(x)=V(x)+Ax$, where each entry $A_{ij}$ is selected uniformly at random from $[-\epsilon,\epsilon]$. A VI solution $x^\ast$ for $\tilde{V}$ is a $\Theta(\epsilon n)$-approximate VI solution for $V$.


Now Item~$2$ of Definition~\ref{d:directionality1} defining the notion of direction $d=D^i_S(x)$ 
takes the following form,
\[ \left( \nabla_{S\cup \{i\}}V_j(x) + A^j_{S\cup \{i\}} \right)^\top \cdot (d_{s_1},\ldots,d_{s_m},d_i) = 0
\] 
where $A^j_{S\cup \{i\}}$ denotes the $j$-th row of $A$ restricted to the columns $\ell \in S\cup\{i\}$. Due to the fact that all vectors $\nabla_{S\cup \{i\}}V_j(x)$ are linearly independent and the fact that the entries $A_{ij}$ have been selected uniformly at random in $[-\epsilon,\epsilon]$ we can easily conclude that 
\[\Pr[\text{there exists }j \in S\cup\{i\}\text{ with }d_j =0] = 0\]
which suggests that Assumption \ref{a:3} holds with high probability at $x$.

\section{2-d Example of STOR'N Execution} \label{ref:bilinear}
 
 In Figure \ref{fig:exampleMain}, we show the trajectory that our algorithm follows when it is applied to solve a min-max optimization problem with objective $f(\theta,\omega):= (\theta-1/2)\cdot(\omega-1/2)$ where $\theta$ is the minimizing and $\omega$ is the maximizing variable. We explain below how this trajectory is derived by following Dynamics~\ref{dyn:Sperner}.

\begin{figure}[!ht]
\centering
    \includegraphics[width=0.4\textwidth]{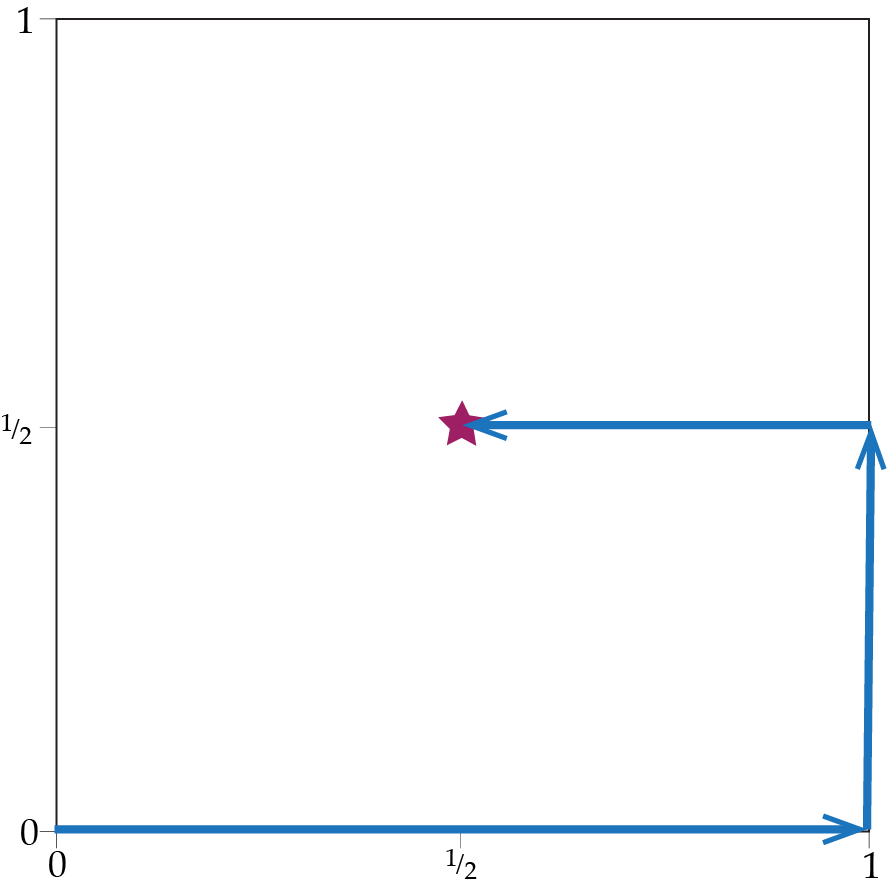}
    \caption{The path of STON'R for $f(\theta, \omega) = (\theta - 1/2) \cdot (\omega - 1/2)$.}
    \label{fig:exampleMain}
\end{figure}

First, using our notation in Section~\ref{sec:prelimss}, let $x_1$ correspond to $\theta$ and $x_2$ correspond to $\omega$. As explained in the same section, finding a local min-max equilibrium can be reduced to a non-monotone VI problem where $V_1(x_1,x_2) := 1/2 -x_2$ and $V_2(x_1,x_2) := x_1 -1/2$. Next we describe the steps that our algorithm follows.

\begin{itemize}
    \item[\textcolor{Sepia}{$\triangleright$}] \textcolor{Sepia}{\underline{$x(0) = (0,0), i = 1, S=\emptyset, t=0 \text{, STON'R goes to Step 3.}$}}
    $V_1(0,0) = 1/2>0$ and $x_1 =0$, hence coordinate $1$ is not satisfied. Thus, the loop of Step~$2$ is activated and STON'R goes to Step~$3$.
    
    \item[\textcolor{Sepia}{$\triangleright$}] \textcolor{Sepia}{\underline{STON'R goes to Step~$5$ and executes $\dot{z}(\tau) = (1 , 0)$ with initialization $z(0) = (0,0)$.}} Note that at $x=(0,0)$ the only unit direction satisfying the constraints of Definition~\ref{d:directionality1} is $(1,0)$ and that the same is true for any point $(\cdot,0)$. Thus, for all these points $D^1_\emptyset((\cdot,0))=(1,0)$, and  the continuous-time dynamics executed at Step~$5$ is $\dot{z}(\tau) = (1 , 0)$. 
    
    \end{itemize}
    
    \begin{itemize}    
    \item[\textcolor{Sepia}{$\triangleright$}] \textcolor{Sepia}{\underline{STON'R goes to Step~$7$ and sets $x(1)=(1,0)$.}}
    For any point $z =(z_1,0)$, $V_1(z) = 1/2$. Thus the continuous-time dynamics of Step~$5$ only terminates when it hits the boundary of the square at point $(1,0)$, which happens at time $\tau_{\rm exit}=1$. At Step~$7$, the algorithm sets $x(1)=z(1)=(1,0)$. 

    \item[\textcolor{Sepia}{$\triangleright$}] \textcolor{Sepia}{\underline{STON'R goes to Step~$12$ and sets $i=2$.}}
     $V_1(x(1)) = 1/2 >0$ thus coordinate~$1$ is boundary-satisfied at this point. Because this is  the good event of Definition~\ref{def:exitConditions}, the condition of the if statement of Step~$8$ triggers. Because coordinate $1$ is boundary-satisfied the condition of the if statement of Step~$9$ is not triggered. Thus the algorithm arrives at Step~$12$ and sets $i=2$.
    
    \item[\textcolor{Sepia}{$\triangleright$}] \textcolor{Sepia}{\underline{STON'R goes to Step~$3$ with $i = 2$, $S=\emptyset$.}}
    At $x(1)=(1,0)$ coordinate~$1$ is boundary-satisfied since $V_1(1,0)=1/2 > 0$ but coordinate $2$ is not satisfied since $V_2(1,0)=1/2 > 0$. Thus, the while condition of Step~$2$ is triggered and STON'R goes to Step~$3$. 

     \item[\textcolor{Sepia}{$\triangleright$}] \textcolor{Sepia}{\underline{STON'R goes to Step~$5$ and executes $\dot{z}(\tau) = (0 , 1)$ with initialization $z(0) = (1,0)$.}} Note that at $x=(1,0)$ the only unit direction satisfying the constraints of Definition~\ref{d:directionality1} is $(0,1)$ and that the same is true for any point $(1,\cdot)$. Thus, for all these points $D^2_\emptyset((1,\cdot))=(0,1)$, and the continuous-time dynamics executed at Step~$5$ is $\dot{z}(\tau) = (0 , 1)$. 
     
     
    \item[\textcolor{Sepia}{$\triangleright$}] \textcolor{Sepia}{\underline{STON'R goes to Step~$7$ and sets $x(1.5)=(1,0.5)$.}}
     For any point $z =(1,z_2)$, $V_1(z) = 1/2-z_2$ and $V_2=1/2$. Thus the continuous-time dynamics of Step~$5$ only terminates when it hits point $(1,0.5)$, which happens at time $\tau_{\rm exit}=1/2$. The reason the continuous-time dynamics terminates at this point is because the middling condition of Definition~\ref{def:exitConditions} is triggered for~$j=1$. Indeed, coordinate~$1$ is boundary satisfied from the beginning of the continuous-time dynamics until it reaches point $(1,0.5)$ but if the continuous-time dynamics were to continue onward, then coordinate~$1$ would become unsatisfied as $V_1$ would turn negative. Thus the continuous-time dynamics stops at time $\tau_{\rm exit}=1/2$, the algorithm moves to Step~$7$ and it sets $x(1.5)=z(0.5)=(1,0.5)$. 

    \item[\textcolor{Sepia}{$\triangleright$}] \textcolor{Sepia}{\underline{STON'R goes to Step~$18$ and sets $S=\{1\}$.}}
    Since the most recently executed continuous-time dynamics at Step~$5$ ended at a middling exit point, the condition of Step~$17$ is activated, so the algorihtm moves to Step $18$ where $S$ is set to $\{1\}$. 
     
      \item[\textcolor{Sepia}{$\triangleright$}] \textcolor{Sepia}{\underline{STON'R goes to Step~$3$ with $i = 2$, $S=\{1\}$.}}
    At $x(1.5)=(1,0.5)$ coordinate~$1$ is both zero- and boundary-satisfied since $V_1(1,0.5)=0$ but coordinate $2$ is still not satisfied since $V_2(1,0.5)=1/2$. Thus, the while condition of Step~$2$ is triggered and STON'R goes to Step~$3$. 

     \item[\textcolor{Sepia}{$\triangleright$}] \textcolor{Sepia}{\underline{STON'R goes to Step~$5$ and executes $\dot{z}(\tau) = (-1 , 0)$ with initialization $z(0) = (1,0.5)$.}} Note that at $x=(1,0.5)$ the only unit direction satisfying the constraints of Definition~\ref{d:directionality1} is $(-1,0)$ and that the same is true for any point $(\cdot,0.5)$. Thus, for all these points $D^2_{\{1\}}((\cdot,0.5))=(-1,0)$, and the continuous-time dynamics executed at Step~$5$ is $\dot{z}(\tau) = (-1 , 0)$.
    
    
  
  \item[\textcolor{Sepia}{$\triangleright$}] \textcolor{Sepia}{\underline{STON'R goes to Step~$7$ and sets $x(2)=(0.5,0.5)$.}}
     For any point $z =(z_1,0.5)$, $V_1(z) = 0$ and $V_2=z_1-1/2$. Thus the continuous-time dynamics of Step~$5$ only terminates when it hits point $(0.5,0.5)$, which happens at time $\tau_{\rm exit}=1/2$. The reason the continuous-time dynamics terminates at this point is because the good condition of Definition~\ref{def:exitConditions} is triggered for~$i=2$ at this point. Thus the continuous-time dynamics stops at time $\tau_{\rm exit}=1/2$, the algorithm moves to Step~$7$ and it sets $x(2)=z(0.5)=(0.5,0.5)$.

    \item[\textcolor{Sepia}{$\triangleright$}] \textcolor{Sepia}{\underline{STON'R goes to Step~$22$ and outputs $(0.5,0.5)$.}} The condition of the if statement of both Steps~$8$ and~$9$ are triggered, so $S=\{1,2\}$ and $i=3$. At $x(2)=(0.5,0.5)$ both coordinate $1$ and coordinate $2$ are satisfied, so the while loop of Step $2$ is not activated. So the algorithm goes to Step~$22$ and returns $(0.5,0.5)$.

\end{itemize}

It is easy to verify that the point $(\theta,\omega) = (1/2,1/2)$ is a (local) min-max equilibrium of $(\theta - 1/2)\cdot (\omega - 1/2)$.

\section{Proof of Lemma~\ref{lem:viSolutions}} \label{sec:simple VI characterization lemma}
\begin{proof}
$(\longleftarrow)$ Let $Z$ denote the zero-satisfied coordinates ($V_i(x) =0$), $\mathrm{BS}^+$ the boundary satisfied coordinates with $x_i =1$ (and thus $V_i(x) >0$) and  
$\mathrm{BS}^-$ the boundary satisfied coordinates with $x_i =0$ (and thus $V_i(x) <0$). For any $y \in [0,1]^n$, we have $\sum_{i=1}^n V_i(x)(x_i - y_i) \geq 0$, which can easily be seen by breaking up the sum into three sums corresponding to indices in $Z$, $\mathrm{BS}^+$ and $\mathrm{BS}^-$.\\
\\
$(\longrightarrow)$ Let $x \in [0,1]^n$ be a solution of the V, i.e.~$V(x)^\top (x - y) \leq 0$ for all $y \in [0,1]^n$. Consider an arbitrary $i \in [n]$ and a vector $y$ such that $y_j = x_j$ for all $j \neq i$. If $x_i=1$, take $y_i=0$, and plug this into $V(x)(x - y) \leq 0$ to get $V_i(x) \geq 0$. If $x_i=0$, take $y_i=1$, and plug this into $V(x)^\top (x - y) \leq 0$ to get $V_i(x) \leq 0$. If $x_i \in (0,1)$ consider first $y_i = x_i + \delta$ for some small $\delta>0$ and plug this into $V(x)^\top (x - y) \leq 0$ we get $V(x_i) \geq 0$. By repeating the same argument for $y_i=x_i-\delta$ we that get $V_i(x)\leq 0$. As a result, $V_i(x) = 0$.   
\end{proof}

\section{Proof of Theorem~\ref{t:main}}\label{sec:convergence_proof}
In this section we present the proof of Theorem~\ref{t:main}. The proof follows closely the sketch exhibited in Section~\ref{s:sketch_proof} with some slight modifications on the definition of the nodes $N$ of the directed graph $G$. 

\subsection{Helpful Definitions and Lemmas}
We start with the definition of pivots that will play the role of nodes $N$. 

\begin{definition}\label{d:pivot_new}
 A point $x\in [0,1]^n$ is called a pivot if and only if the following hold,
\begin{itemize}
    \item If coordinate $i$ is not satisfied then $V_i(x) > 0$.
    \item If $\ell$ is the minimum unsatisfied coordinate then $x_j =0$ for all coordinates $j \geq \ell + 1$.
    \item If $\ell$ is the minimum unsatisfied coordinate then there exists at least one coordinate $j \in M\cup \{\ell\}$ with $x_j =0$ or $x_j =1$ where $M:= \{j \leq \ell -1:~ V_j(x) = 0\}$.
\end{itemize}
\end{definition}

As in the proof sketch of Section~\ref{s:sketch_proof}, given a pivot $x$ (that admits at least one unsatisfied variable) we argue that STOR'N visits another pivot at some finite time. As depicted in Dynamics~\ref{dyn:Sperner}, the latter happens by following the continuous curve $\dot{z}(t) = D_S^i(z(t))$. Recall that in Dynamics~\ref{dyn:Sperner} the pair $(i,S)$ is updated in the previous steps of the algorithm (Steps~$8,13,15$~and~$17$). In the next Definitions~\ref{d:directionality2},~\ref{d:frozen}~and~\ref{d:admissible_pairs}, we provide an alternative way of "computing locally" the pair $(i,S)$ by using only the knowledge of $x(t)$ at Step~$3$ of Dynamics~\ref{dyn:Sperner}.

\begin{definition}\label{d:directionality2}
Consider the direction $D^i_{S}(x) := (d_1,\ldots,d_n)$ of Definition~\ref{d:directionality1} for the set of zero-satisfied coordinates $S=\{j<i \text{ with } V_j(x) = 0\}$ (recall that $d_j =0$ for all $j \notin S \cup \{i\}$). If,
for all $k \in S$, one of the following holds:
(a) $x_k \in (0,1)$, or (b) $x_k =0$ and $d_k \geq 0$, or (c) $x_k =1$ and $d_k \leq 0$, then we define $D^i(x) := D_S^i(x)$. Otherwise, let $j \in S$ be the unique coordinate (uniqueness follows from Assumption~\ref{a:2}) 
such that either $\{x_j = 0$ and $d_j <0\}$ or 
$\{x_j = 1$ and $d_j >0\}$, and we define $D^i(x):= D_{S \setminus  \{j\}}^i(x)$. $D^i(x)$ is called the {\em ideal direction of movement} at point $x \in [0,1]^n$ with respect to coordinate $i$.
\end{definition}

\begin{definition}\label{d:frozen}
Given a point $x \in [0,1]^n$ coordinate $i$ is called frozen if and only if ($x_i =0$ and $[D^i(x)]_i < 0)$ or ($x_i =1$ and $[D^i(x)]_i > 0)$ where $D^i(x)$ is the \textit{ideal direction} at $x$ with respect to coordinate $i$ (Definition~\ref{d:directionality2}).
\end{definition}

\begin{definition}\label{d:admissible_pairs}
Given a pivot $x \in [0,1]^n$ consider
\begin{itemize}
    \item $\ell := \min_{1\leq j \leq n}\{\text{coordinate }j \text{ is not satisfied at } x \}.$
    \item $i := \max_{j \leq \ell}\{\text{coordinate } j \text{ is not frozen at }x\}$.
    
    \item $S \leftarrow$ the set of coordinates such that $D^i(\cdot)  = D_S^i(\cdot)$ (see Definition~\ref{d:directionality2}).
\end{itemize}
The coordinate $i$ is called the \textit{under examination} coordinate, the pair $(i,S)$ is called the admissible pair for pivot $x$. 
\end{definition}

\begin{remark}
Computing the $(i,S)$ admissible pair of the pivot $x(t)$ at Step~$3$ in Dynamics~\ref{dyn:Sperner} is equivalent with Dynamics~\ref{dyn:Sperner} at which $(i,S)$ is updated at Steps~$8,13,15$~and~$17$.
\end{remark}
\subsection{Main Steps of the Proof}
To simplify notation we describe STOR'N using the notion of pivots and admissible pairs $(i,
S)$ of Definition~\ref{d:pivot_new}~and~\ref{d:admissible_pairs}.
\begin{algorithm}
  \caption{STay-ON-the-Ridge (STON'R)}\label{dyn:Sperner_EQ}
 \begin{algorithmic}[1]
 \STATE Initially $x{(0)} \leftarrow (0,\ldots,0)$, $i \leftarrow 1$, $S \leftarrow \emptyset$, $t \leftarrow 0$.

 \WHILE {$x{(t)}$ is not a VI solution}
        \STATE At point $x(t)$ compute the admissible pair $(i,S)$ for pivot $x(t)$.
        \smallskip
        \STATE Follows the continuous-time dynamics, $\dot{z}(\tau) = D^i_S(z(\tau))$, at $z(0)=x{(t)}$.
        \smallskip
        \WHILE{${z}(\tau)$ is not an exit point as per Definition \ref{def:exitConditions}}
        \smallskip
          \STATE Execute $\dot{z}(\tau)=D^i_S(z(\tau))$ forward in time.
        \smallskip
        \ENDWHILE
        \STATE Set $x{(t+\tau)}=z(\tau)$ for all $\tau \in [0, \tau_{\rm exit}]$ \textit{(where $\tau_{\rm exit}$ is earliest time ${z}(\tau)$ became an exit point)}.
        
        \STATE Set $t \leftarrow t + \tau_{\rm exit}$.
\ENDWHILE
  \STATE \textbf{return }$x(t)$
  \end{algorithmic}
\end{algorithm}

We are now ready to present the topological argument described in Section~\ref{s:sketch_proof}. As already mentioned the nodes $N$ of the directed graph $G$ will be the set of pivots while we say that there exists an edge $(x,x')$ from pivot $x$ to pivot $x'$ in case setting $z(0):= x$ and following the direction $\dot{z}(t) = D_S^i(z(t))$ (where $(i,S)$ is the admissible pair of $x$) leads to pivot $x'$ once one of the "if" loops in Steps~$9,11,13$~and~$15$ is activated.

In Lemma~\ref{l:pivot_finite} we establish the fact that the pivots which correspond to the number of nodes of directed graph $G$ are finite.

\begin{lemma}\label{l:pivot_finite}
There exists a finite number of pivots.
\end{lemma}

In Definition~\ref{d:edges} we formalize the notion of directed edge $(x,x')$ in graph $G$ which we additionally denote as $x' = \mathrm{Next}(x)$.

\begin{definition}\label{d:edges}
Given a pivot $x \in [0,1]^n$ consider the trajectory $\dot{z}(t) =D_S^i(z(t))$ with $z(0) = x$ where $(i,S)$ is the admissible pair of $x$. We say that pivot $x'$ is the next pivot of $x$, i.e. $x' = \mathrm{Next}(x)$ if and only if there exists $t^\ast >0$ such that
\begin{itemize}
    \item $z(t^\ast) = x'$
    \item $z(t)$ is not a pivot for all $t \in (0,t^\ast)$.
\end{itemize}
\end{definition}

In Lemma~\ref{l:outdegree} we establish the fact that any pivot with at least one unsatisfied variable must necessarily admit outdegree equal to $1$. The latter directly implies that any pivot with outdegree $0$ must correspond to a solution since all coordinates are satisfied. 
\begin{lemma}\label{l:outdegree}
For any pivot $x\in [0,1]^n$ with at least one unsatisfied coordinate there exists a pivot $x'$ such that $x' = \mathrm{Next}(x)$. Moreover let $(i,S)$ be the admissible pair for pivot $x$, $z(t)$ be the trajectory $\dot{z}(t) = D_S^i(z(t))$ with $z(0) = x$ and $t'$ be the time at which $x' = z(t')$. Then for all $t \in [0,t']$,
\begin{itemize}
  \item all coordinates $j \in S$ admit $V_j(z(t)) =0$.
  \item all coordinates $j \leq i-1$ are satisfied at $z(t)$.
    \item all coordinates $j \geq i + 1$ admit $z_j(t) =0$.
      \item all coordinates $j$ admit $z_j(t) \in [0,1]$.
\end{itemize}
\end{lemma}
Using Lemma~\ref{l:outdegree} we additionally obtain Corollary~\ref{c:1} ensuring that the point $x(t)$ at Step~$3$ of Dynamics~\ref{dyn:Sperner_EQ} is always a pivot and thus Dynamics~\ref{dyn:Sperner_EQ} is well-defined.
\begin{corollary}\label{c:1}
Let $x(t)$ at Step~$3$ of Dynamics~\ref{dyn:Sperner_EQ} be a pivot. Then the point $x(t + \tau_{exit})$ at Step~$8$ of Dynamics~\ref{dyn:Sperner_EQ} is also a pivot. Moreover the point $(0,\ldots,0)$ is a pivot. 
\end{corollary}
In Lemma~\ref{l:indegree} we establish the fact that no pivot/node can admit in-degree more than $2$. The latter implies if we start with a pivot with $0$ in-degree we must essentially visit a pivot with out-degree $0$ that consists a solution. 

\begin{lemma}\label{l:indegree}
Any pivot $x\in [0,1]^n$ admits in-degree at most $1$. In other words in case $x^\ast = \mathrm{Next}(x_1)$ and $x^\ast = \mathrm{Next}(x_2)$ for some pivots $x_1,x_2$ then $x_1 = x_2$.
\end{lemma}
We conclude the proof by showing that $(0,\ldots,0)$ that is the initial pivot that Dynamics~\ref{dyn:Sperner_EQ} visits  admits $0$ in-degree.
\begin{lemma}\label{l:0}
There is no pivot $x\in [0,1]^n$ such that $ \mathrm{Next}(x) = (0,\ldots,0)$.
\end{lemma}

\section{Proof of Lemma~\ref{l:pivot_finite}}
\begin{lemma}\label{l:finite}
Let the functions $F_1(x),\ldots,F_i(x)$ where $F_\ell: [0,1]^i \mapsto \R$ and the set $B:=\{x\in [0,1]^i:~~F_\ell(x)=0~~\text{for all } \ell = 1,\ldots,i\}$. In case $F_1,\ldots, F_\ell$ satisfy the following assumptions
\begin{itemize}
    \item $\norm{\nabla F_\ell(x) - \nabla F_\ell(y)}_2
\leq L \cdot \norm{x-y}_2$

\item For all $x \in B$ the matrix
\[
J(x):= 
\begin{pmatrix}
\frac{\partial F_1(x)}{\partial x_1}& \ldots & \frac{\partial F_1(x)}{\partial x_i}\\
\vdots& & \vdots\\
\frac{\partial F_i(x)}{\partial x_1}& \ldots & \frac{\partial F_i(x)}{\partial x_i}
\end{pmatrix}
\]
admits singular values that are at least $\sigma_{\text{min}}$ and at most $\sigma_{\text{max}}$.
\end{itemize}
Then the set $B$ is finite. More precisely, $|B|  \leq  
2^i/\text{Vol}^i\left( \frac{2\sigma^2_{min}}{\sqrt{i}L\sigma^2_{max}} \right)$
where $\text{Vol}^i(\rho)$ is the volume of the $i$-dimensional ball with radius $\rho$.
\end{lemma}
Lemma~\ref{l:pivot_finite} directly follows by Lemma~\ref{l:finite}. More precisely, we get that the number of pivots in $[0,1]^n$ is at most $n \cdot 4^n/ \text{Vol}^n\left( \frac{2\sigma^2_{min}}{\sqrt{n}L\sigma^2_{max}} \right)$.
\subsection{Proof of Lemma~\ref{l:finite}}
Let us assume the existence of $x,y \in B$ such that $\norm{x -y}_2 \leq \rho$ and $x\neq y$. Notice that the $\nabla F_1(x),\ldots, \nabla F_i(x)$ are linearly independent and thus
\[x - y = \sum_{j=1}^i \mu_j \cdot \nabla F_j(x)\]
which implies that
\begin{equation}\label{eq:mu}
\norm{\mu}_2 \leq \frac{\rho}{\sigma_{\text{min}}}    
\end{equation}
By Taylor expansion of $x$ and the fact that $\norm{\nabla F_\ell(x) - \nabla F_\ell(y)} \leq L \cdot \norm{x-y}_2$ we get,
\[
\left| F_\ell(y) - F_\ell(x) - (\nabla F_\ell(x))^\top \cdot
\sum_{j=1}^i \mu_j \cdot \nabla F_j(x) 
\right| \leq \frac{1}{2}L \cdot \norm{\sum_{\ell=1}^i\mu_j \cdot \nabla F_j(x)}^2
\]
which due to the fact that $F_\ell(y)= F_\ell(x) =0$ implies,
\[\left[J^\top(x) \cdot J(x) \cdot \mu\right]_\ell \leq \frac{1}{2}L \cdot \sigma^2_{\text{max}}\cdot \norm{\mu}^2 \]
and thus
\begin{equation}\label{eq:mu2}
\norm{\mu}_2 \geq \frac{2 \sigma_{\text{min}}}{\sqrt{i}L \sigma^2_{\text{max}}}    
\end{equation}
Combining Equation~\ref{eq:mu} and~\ref{eq:mu2} we get $\rho \geq  \frac{2 \sigma^2_{\text{min}}}{\sqrt{i}L \sigma^2_{\text{max}}}$.
To this end we know that in case $x,y \in B$ with $x\neq y$ then $\norm{x - y}_2 \geq \frac{2 \sigma^2_{\text{min}}}{\sqrt{i}L \sigma^2_{\text{max}}}$. Thus,
\[|B|  \leq  
2^i/\text{Vol}^i\left( \frac{2\sigma^2_{min}}{\sqrt{i}L\sigma^2_{max}} \right)\]

\section{Proof of Lemma~\ref{l:outdegree}}
\begin{lemma}\label{l:main_new}
Let a pivot $x \in [0,1]^n$ and $(i,S)$ the admissible pair for $x$. Then the following hold,
\begin{itemize}
    \item There exists a unique trajectory $z(t)$ with $\dot{z}(t) = D^i_S\left(z(t)\right)$ and $z(0) = x$.
    
    \item $V_j(z(t)) =0$ for all coordinates $j \in S$.
    \item There exists $t^\ast >0$ such that for all $ t \in [0,t^\ast]$ all coordinates $j \leq i-1$ are satisfied at $z(t)$ and $z_j(t) \in [0,1]$ for all coordinates $j$. 
\end{itemize}
\end{lemma}

\begin{lemma}\label{l:hit}
Let a set of coordinates $S$, a coordinate $i$ and a point $x \in [0,1]$ such that $x_j \in (0,1)$ for all $j \in S\cup \{i\}$. Consider the trajectory $\dot{\gamma}(t) = D_S^i(\gamma(t))$ with $\gamma(0) = x$. Then there exists $t^\ast \in (0,C]$ such that
\[\gamma_j(t^\ast) = 0\text{ or }1~~\text{ for some }j \in S \cup \{i\}\]
where $C$ is constant depending on the parameters $\sigma_{\text{min}},\sigma_{\text{max}}$ and $L$.
\end{lemma}

\begin{lemma}\label{l:new}
Let a pivot $x \in [0,1]^n$ with at least one unsatisfied coordinate. Let $(i,S)$ the admissible pair of pivot $x$ (Definition~\ref{d:admissible_pairs}) and $\ell$ the minimum unsatisfied coordinate at $x$. Then the following hold,
\begin{itemize}
    \item The under examination variable admits $i \geq 1$.
    
    \item $x_j=0\text{ for all coordinates }j \geq i + 1.$
\end{itemize}
 Additionally one of the following holds,
\begin{itemize}
    \item $i= \ell$ and $ V_\ell(x) >0$
    \item $x_i =1$ and $V_i(x) > 0$
    \item$V_i(x) =0$ and $D_S^i(x)^\top  \nabla V_i(x) > 0$
\end{itemize}
\end{lemma}

\subsection{Proof of Lemma~\ref{l:outdegree}}
Given the pivot $x \in [0,1]^n$ with at least one unsatisfied variable and let $(i,S)$ denote its admissible pair. By Lemma~\ref{l:new} we know that the under examination variable $i$ admits $i \geq 1$. Now consider the the trajectory $\dot{z}(t) = D^i_S(z(t))$ with $z(0) = x$. Due to the fact that $i \geq 1$ and by Assumption~\ref{a:3} we known that for all $t \in(0,\delta)$ where $\delta > 0$ is sufficiently small, the following hold
\begin{itemize}
    \item $z_j (t) \in (0,1)$ for all $j \in S \cup \{i\}$
    
    \item $z_j(t) =0$ or $z_j(t)=1$ for all coordinates $j \notin S \cup \{i\}$.
\end{itemize}
By Lemma~\ref{l:hit} there exists $t^\ast>0$ such that $z_j(t^\ast) =0$ or $1$ for some coordinate $j \in S \cup \{i\}$ and $z_j(t) \in [0,1])$ for all coordinates $j$ and $t \in [0,t^\ast]$.

We first show that if there exists a coordinate $j \leq i-1$ such that $j$ is not satisfied at $z(t^\ast)$ then there exists $\hat{t} < t^\ast$ such that $z(\hat{t})$ is a pivot.

Let $\ell$ denote the unsatisfied coordinate at $z(t^\ast)$. Notice that by Lemma~\ref{l:main_new} all coordinates $j \in S$ admit $V_j\left(z(t^\ast)\right) = 0$ and thus $\ell \notin S$. The latter implies that ($x_\ell =0$ and $V_\ell(x)\leq 0$) or ($x_\ell =1$ and $V_\ell(x)\geq 0$) and since coordinate $\ell$ stands still in the trajectory $\dot{z}(t) = D_S^i(z(t))$, $z_\ell(\hat{t}) = x_\ell$ there are two mutually exclusives cases:
\begin{itemize}
    \item $x_\ell=z_\ell(\hat{t})=0$, $V_\ell(x) \leq 0$ and $V_\ell(z(\hat{t})) >0$
    
    \item $x_\ell=z_\ell(\hat{t})=1$, $V_\ell(x) \geq 0$ and $V_\ell(z(\hat{t})) <0$
\end{itemize}
Then by Lemma~\ref{c:main} we additionally get that for sufficiently small $\delta>0$,
\begin{itemize}
    \item If $x_\ell =0$ then $V_\ell(z(t)) <0$ for $t \in (0,\delta)$
    
    \item If $x_\ell =1$ then $V_\ell(z(t)) >0$ for $t \in (0,\delta)$
\end{itemize}
As a result, in any case there exists $t_\ell \in (0,t^\ast)$ such that $V_\ell(z(t_\ell)) =0$, coordinate $\ell$ lies on the boundary at $z(t_\ell) =0$ and coordinate $\ell$ is satisfied at $z(t)$ for all $t \in [0,t_\ell]$.

Now consider the set of coordinates $A:= \{j\leq i-1:~\text{coordinate }j\text{ is not satisfied at } z(t^\ast) \}$ and let $\hat{t} := \min_{\ell \in A} \hat{t}_\ell$. Then all coordinates $j \leq i-1$ are satisfied at $z(\hat{t})$ while there exists a coordinate $\hat{\ell} \in  A$ such that $V_{\hat{\ell}}(z(\hat{t})) =0$ with coordinate $\hat{\ell}$ being on the boundary at $z(\hat{t})$. Up next we argue that $z(\hat{t})$ is a pivot.

Consider the set of coordinates $M:= \{j \leq i-1:~V_j(z(\hat{t})) = 0\}$. Since $\hat{\ell} \in M$ and $\hat{\ell}$ lies on the boundary at $z(\hat{t})$ the third item of Definition~\ref{d:pivot_new} is satisfied. Since $x$ is a pivot, Lemma~\ref{l:new} implies all coordinates $j \geq i+1$ admit $x_j =0$. Since $[D^i_S(\cdot)]_j =0$ for all $j \geq i+1$ the latter implies that $z_j(\hat{t}) = x_j = 0$. Due to the fact that all coordinates $j \leq i-1$ are satisfied at $z(\hat{t})$ we get that the second item of Definition~\ref{d:pivot_new} is satisfied. Finally notice that by Lemma~\ref{l:new}, $V_i(z(t)) >0 $ for all $t \in (0,\delta)$ for sufficiently small $\delta$. In case $V_i(z(\hat{t})) < 0$ then there exist $t'<\hat{t}$ such that $V(z(t')) =0$ implying that $z(t')$ is a pivot. As a result, without loss of generality we assume that $V_i(z(\hat{t})) > 0$ which implies that the first item of Definition~\ref{d:pivot_new} is satisfied.

As a result, without loss of generality we assume that all coordinates $j \leq i - 1$ are satisfied for all $z(t)$ in $[0,t^\ast]$. Up next we show that in this case either the point $x^\ast := z(t^\ast)$ is a pivot or $z(\hat{t})$ is pivot for some $\hat{t} \in (0,t^\ast)$.

Notice that by Lemma~\ref{l:new} all coordinates $j \geq i+1$ admit $x_j =0$. Since $[D^i_S(\cdot)]_j =0$ for all $j \geq i+1$ the latter implies that $x_j^\ast = x_j = 0$. As a result, $x_j^\ast=0$ for all $j \geq i+1$. Due to our assumption that all coordinates $j \leq i -1$ are satisfied at $x^\ast$, the minimum unsatisfied coordinate at $x^\ast$ is greater than $i$ and thus the second item of Definition~\ref{d:pivot_new} is satisfied.
Moreover due to the fact that $x_j^\ast =0$ or $1$ for some $j \in S \cup \{i\}$ and $V_j(x^\ast) =0$ for all $j \in S$, the third item of Definition~\ref{d:pivot_new} is satisfied.

Up next we argue that in case coordinate $i$ is not satisfied at $x^\ast$ then $V_i(x^\ast) > 0$. Let us assume that coordinate $i$ is not satisfied at $x^\ast$ and $V_i(x^\ast) < 0$. Let $\ell$ denote the minimum unsatisfied variable at $x$ then Lemma~\ref{l:new} provides us with the following mutually exclusive cases:
\begin{itemize}
    \item \underline{$i = \ell$ then $V_\ell(x) > 0$}: Since $V_\ell(z(0)) = V_\ell(z) > 0$ and $V_\ell(z(t^\ast)) = V_\ell(x^\ast) < 0$ there exists $\hat{t} \in (0,t^\ast)$ such that $V_\ell(z(\hat{t})) =0$. Notice that $z(\hat{t})$ satisfies all the three items of Definition~\ref{d:pivot_new}.
    
    \item \underline{$x_i =1$ with $V^i(x) > 0$}: Same as above.
    
    \item \underline{$V_i(x) =0$ and $D_S^i(x)^\top \cdot V^i(x) > 0$}: Notice that $V_i(z(t)) >0$ for all $t \in (0,\delta)$ once $\delta$ is selected sufficiently small. By repeating the same argument as above we conclude that there is $\hat{t}$ such that $z(\hat{t})$ is a pivot.
\end{itemize}

\subsection{Proof of Lemma~\ref{l:main_new}}\label{appendix:main_new}
Let the set of coordinates $S$, we first establish in Lemma~\ref{l:Lip2} that $D^i_S(x)$ is $M$-Lipschitz. The proof of Lemma~\ref{l:Lip2} is presented at Section~\ref{l:9}
\begin{lemma}\label{l:Lip2}
Let $x \in [0,1]^n$ and a set of coordinates $S$ such that $V_j(x) = 0$ for all $j \in S$. Then for any coordinate $i \notin S$ and for any $y \in \R^n$ such that $\norm{x - y}_2 \leq \sigma_{\text{min}}/2L\sqrt{n}$,
\[\norm{D^i_S(x) - D^i_S(y)}_2 \leq M \cdot \norm{x - y}_2 \]
for $M:= \Theta \left(\frac{\sigma_{\text{max}}}{\sigma^2_{\text{min}}}\cdot \sqrt{n} \cdot L\right)$.
\end{lemma}
To simplify notation let $Z^i(x) := \{\ell <i \text{ such that } V_\ell(x) = 0\}$ and $F^i(x) := \{\ell<i \text{ such that coordinate }\ell \text{ is fixed at } x \}$. Since $x$ is a pivot, all coordinates $\ell \leq i-1$ are satisfied and thus each coordinate $\ell \leq i-1$ either belongs to $Z^i(x)$ or to $F^i(x)$.\\

Let us first consider the case where one of the following holds for all coordinates $\ell \in Z^i(x)$.
\begin{itemize}
    \item $x_\ell \in (0,1)$
    \item $x_\ell =0$ and $[D^i_{Z_i(x)}]_\ell \geq 0$
    \item $x_\ell =1$ and $[D^i_{Z_i(x)}]_\ell \leq 0$
\end{itemize}
Notice that in this case Definition~\ref{d:directionality2}~and~\ref{d:admissible_pairs} imply $S = Z^i(x)$. Now consider the set $B:=\{y \in \R^n \text{ such that } \norm{x -y}_2 \leq \sigma_{\min}/2L\}$. Then combining Lemma~\ref{l:Lip2} with the Picard–Lindelöf theorem we get that there exists a unique $z(t)$ such that
\begin{enumerate}
    \item $\dot{z}(t) = D^i_S\left(z(t)\right)$
    \item $z(0) = x$
\end{enumerate}
By taking $\delta > 0$ sufficiently small we get that for all $t \in [0,\delta]$ the following hold,
\begin{itemize}
    \item $V_\ell \left(z(t)\right) =0$ for all $\ell \in S$
    
    \item $z_\ell(t) \in (0,1)$ for all $\ell \in S$.
    
    \item coordinate $\ell$ is boundary satisfied at $z(t)$ for all $\ell \in \{1,\ldots,i-1\} / S$.
    
    \item $z_j(t) \in [0,1]^n$ for all coordinates $j$.
    
\end{itemize}
Now consider the case in which there exists a coordinate $j \in Z^i(x)$ such that ($x_j = 0$ and $[D^i_{Z_i(x)}]_j <0$) or ($x_j = 1$ and $[D^i_{Z_i(x)}]_j >0$). By Assumption~\ref{a:2} we know that such a coordinate must be unique. In this case by Definition~\ref{d:directionality2} we get $D^i(x) = D^i_{Z^i(x)/\{j\}}(x)$ and thus by Definition~\ref{d:admissible_pairs}, $S = Z^i(x)/\{j\}$.

Lemma~\ref{c:main} establish the fact that in this case following the direction $D^i_S(x)$ consists the variable $j$ boundary satisfied. The proof of Lemma~\ref{c:main} is presented in Section~\ref{appendix:c:main}.
\begin{lemma}\label{c:main}
For any $x\in [0,1]^n$ if there exists coordinate $j$ with
\begin{itemize}
    \item $x_j =0$ and $[D^i_{Z^i(x)}(x)]_j < 0$ then 
    $\left(D^i_{Z^i(x)/\{j\}}(x)\right)^\top \cdot \nabla V_j(x) < 0$.
    
    \item $x_j =1$ and $[D^i_{Z^i(x)}(x)]_j > 0$ then 
    $\left(D^i_{Z^i(x)/\{j\}}(x)\right)^\top \cdot \nabla V_j(x) > 0$.
\end{itemize}
\end{lemma}
By the exact same arguments as above, we get that there exists a unique trajectory $z(t)$ such that $\dot{z}(t) = D^i_S\left(z(t)\right)$ and $x(0) = x$ and by taking $\delta >0$ sufficiently small we get, 
 \begin{itemize}
    
    \item $V_\ell\left(z(t)\right) = 0 $ for all $\ell\in S$
    
    \item $z_\ell(t) \in (0,1) $ for all $\ell\in S$ 
    \item coordinate $\ell$ is boundary satisfied at $z(t)$ for all $\ell \in F^i(x)$
    \item $z_j(t) \in [0,1]^n$ for all coordinates $j$.

\end{itemize}
In order to complete the proof of Lemma~\ref{l:main_new} we need to argue that the coordinate $j$ is satisfied for all $z(t)$ with $t \in [0,\delta']$. Without loss of generality consider $x_j =0$ (the case $x_j = 1$ follows symmetrically). Recall that $V_j(x) = 0$ and by Lemma~\ref{c:main} we get that $\left(D_S^i(x)\right)^\top \cdot \nabla V_j(x) <0$. Thus by selecting $\delta' < \delta$ sufficiently small we get
 \[V_j\left(z(t)\right) < 0~~\text{ and }~~z_j(t) = 0\]
for all $t \in (0,\delta']$.

\subsection{Proof of Lemma~\ref{c:main}~\label{appendix:c:main}}
To simplify notation let
$Z^i(x) = \{1,\ldots,i-1\}$,  $D^i_{Z^i(x)}(x) = (d_1,\ldots,d_j,\ldots,d_i)$ and $D^i_{Z^i(x)/\{j\}}(x) = (\hat{d}_1,\ldots,\hat{d}_{j-1},\hat{d}_{j-1},\ldots,\hat{d}_i)$. Moreover let assume that $x_j = 0$ and $i$ is even. The cases $x_j = 0$ and $i$ is odd, 
$x_j = 1$ and $i$ is even, $x_j = 1$ and $i$ is odd follow symmetrically.\\
\\
We will prove that
\[\left(\hat{d_1},\ldots,\hat{d}_{j-1},\hat{d}_{j+1},\ldots,\hat{d}_{i}\right)^\top \cdot \left(\frac{\partial V_j(x)}{\partial x_1},\ldots ,  \frac{\partial V_{j}(x)}{\partial x_{j-1}} ,\frac{\partial V_{j}(x)}{\partial x_{j+1}},
\ldots ,\frac{\partial V_{j}(x)}{\partial x_{i-1}} \right) < 0\]
Since $i$ is even we get by Definition~\ref{d:directionality1},
\begin{equation}\label{eq:posa}
    \begin{vmatrix}
\frac{\partial V_1(x)}{\partial x_1} & \ldots &  \frac{\partial V_{j-1}(x)}{\partial x_{1}} & \frac{\partial V_{j+1}(x)}{\partial x_{1}}
&\ldots &\frac{\partial V_{i-1}(x)}{\partial x_{1}} & \hat{d}_1\\
\vdots &  &  \vdots
& & \vdots
& &\vdots \\
\frac{\partial V_1(x)}{\partial x_{j-1}} & \ldots &  \frac{\partial V_{j-1}(x)}{\partial x_{j-1}} & \frac{\partial V_{j+1}(x)}{\partial x_{j-1}}
&\ldots &\frac{\partial V_{i-1}(x)}{\partial x_{j-1}} & \hat{d}_{j-1}\\
\frac{\partial V_1(x)}{\partial x_{j+1}} & \ldots &  \frac{\partial V_{j-1}(x)}{\partial x_{j+1}} & \frac{\partial V_{j+1}(x)}{\partial x_{j+1}}
&\ldots &\frac{\partial V_{i-1}(x)}{\partial x_{j+1}} & \hat{d}_{j+1}\\
\vdots &  &  \vdots
& \vdots & \vdots
& &\vdots\\
\frac{\partial V_1(x)}{\partial x_{i}} & \ldots &  \frac{\partial V_{j-1}(x)}{\partial x_{i}} & \frac{\partial V_{j+1}(x)}{\partial x_{i}}
&\ldots &\frac{\partial V_{i-1}(x)}{\partial x_{i}} & \hat{d}_{i}\\
\end{vmatrix} > 0
\end{equation}
and that
\begin{equation}\label{eq:2}
    \begin{vmatrix}
\frac{\partial V_1(x)}{\partial x_1} & \ldots &  \frac{\partial V_{j-1}(x)}{\partial x_{1}} &
\frac{\partial V_j(x)}{\partial x_{1}}
&\frac{\partial V_{j+1}(x)}{\partial x_{1}}
&\ldots &\frac{\partial V_{i-1}(x)}{\partial x_{1}} & d_1\\
\vdots &  &  \vdots
& \vdots & \vdots &
&\vdots & \vdots\\
\frac{\partial V_{1}(x)}{\partial x_{j-1}} & \ldots &  \frac{\partial V_{j-1}(x)}{\partial x_{j-1}} &
\frac{\partial V_j(x)}{\partial x_{j-1}}
&\frac{\partial V_{j+1}(x)}{\partial x_{j-1}}
&\ldots &\frac{\partial V_{i-1}(x)}{\partial x_{j-1}} & d_{j-1}\\
\frac{\partial V_{1}(x)}{\partial x_{j}} & \ldots &  \frac{\partial V_{j-1}(x)}{\partial x_{j}} &
\frac{\partial V_j(x)}{\partial x_{j}}
&\frac{\partial V_{j+1}(x)}{\partial x_{j}}
&\ldots &\frac{\partial V_{i-1}(x)}{\partial x_{j}} & d_{j}\\
\frac{\partial V_{1}(x)}{\partial x_{j+1}} & \ldots &  \frac{\partial V_{j-1}(x)}{\partial x_{j+1}} &
\frac{\partial V_j(x)}{\partial x_{j+1}}
&\frac{\partial V_{j+1}(x)}{\partial x_{j+1}}
&\ldots &\frac{\partial V_{i-1}(x)}{\partial x_{j+1}} & d_{j+1}\\
\vdots &  &  \vdots
& \vdots & \vdots &
&\vdots & \vdots\\
\frac{\partial V_{1}(x)}{\partial x_{i}} & \ldots &  \frac{\partial V_{j-1}(x)}{\partial x_{i}} &
\frac{\partial V_j(x)}{\partial x_{i}}
&\frac{\partial V_{j+1}(x)}{\partial x_{i}}
&\ldots &\frac{\partial V_{i-1}(x)}{\partial x_{i}} & d_{i}\\
\end{vmatrix} < 0
\end{equation}
Combining the fact that $\left(\frac{\partial V_\ell(x)}{\partial x_1},\ldots, \frac{\partial V_\ell(x)}{\partial x_i}\right)^\top\cdot (d_1,\ldots,d_i) = 0$ (see Definition~\ref{d:directionality1}) with 
$d_j < 0$ (we have assumed that $x_j = 0$) we get by Equation~\ref{eq:2}, 
\begin{equation}
    \begin{vmatrix}
\frac{\partial V_1(x)}{\partial x_1} & \ldots &  \frac{\partial V_{j-1}(x)}{\partial x_{1}} &
\frac{\partial V_j(x)}{\partial x_{1}}
&\frac{\partial V_{j+1}(x)}{\partial x_{1}}
&\ldots &\frac{\partial V_{i-1}(x)}{\partial x_{1}} & d_1\\
\vdots &  &  \vdots
& \vdots & \vdots &
&\vdots & \vdots\\
\frac{\partial V_{1}(x)}{\partial x_{j-1}} & \ldots &  \frac{\partial V_{j-1}(x)}{\partial x_{j-1}} &
\frac{\partial V_j(x)}{\partial x_{j-1}}
&\frac{\partial V_{j+1}(x)}{\partial x_{j-1}}
&\ldots &\frac{\partial V_{i-1}(x)}{\partial x_{j-1}} & d_{j-1}\\
0 & \ldots &  0 &0
&0
&\ldots &0 & d_1^2 + \ldots + d_{i}^2\\
\frac{\partial V_{1}(x)}{\partial x_{j+1}} & \ldots &  \frac{\partial V_{j-1}(x)}{\partial x_{j+1}} &
\frac{\partial V_j(x)}{\partial x_{j+1}}
&\frac{\partial V_{j+1}(x)}{\partial x_{j+1}}
&\ldots &\frac{\partial V_{i-1}(x)}{\partial x_{j+1}} & d_{j+1}\\
\vdots &  &  \vdots
& \vdots & \vdots &
&\vdots & \vdots\\
\frac{\partial V_{1}(x)}{\partial x_{i}} & \ldots &  \frac{\partial V_{j-1}(x)}{\partial x_{i}} &
\frac{\partial V_j(x)}{\partial x_{i}}
&\frac{\partial V_{j+1}(x)}{\partial x_{i}}
&\ldots &\frac{\partial V_{i-1}(x)}{\partial x_{i}} & d_{i}\\
\end{vmatrix} > 0
\end{equation}
which implies that 
\begin{equation}
    \begin{vmatrix}
\frac{\partial V_1(x)}{\partial x_1} & \ldots &  \frac{\partial V_{j-1}(x)}{\partial x_{1}} 
&\frac{\partial V_{j+1}(x)}{\partial x_{1}}
&\ldots &\frac{\partial V_{i-1}(x)}{\partial x_{1}} & \frac{\partial V_j(x)}{\partial x_{1}}\\
\vdots &  &  \vdots
& \vdots & \vdots &
&\vdots \\
\frac{\partial V_{1}(x)}{\partial x_{j-1}} & \ldots &  \frac{\partial V_{j-1}(x)}{\partial x_{j-1}} 
&\frac{\partial V_{j+1}(x)}{\partial x_{j-1}}
&\ldots &\frac{\partial V_{i-1}(x)}{\partial x_{j-1}} & \frac{\partial V_j(x)}{\partial x_{j-1}}\\
\frac{\partial V_{1}(x)}{\partial x_{j+1}} & \ldots &  \frac{\partial V_{j-1}(x)}{\partial x_{j+1}}
&\frac{\partial V_{j+1}(x)}{\partial x_{j+1}}
&\ldots &\frac{\partial V_{i-1}(x)}{\partial x_{j+1}} & \frac{\partial V_j(x)}{\partial x_{j+1}}\\
\vdots &  &  \vdots
& \vdots & \vdots &
&\vdots \\
\frac{\partial V_{1}(x)}{\partial x_{i}} & \ldots &  \frac{\partial V_{j-1}(x)}{\partial x_{i}} 
&\frac{\partial V_{j+1}(x)}{\partial x_{i}}
&\ldots &\frac{\partial V_{i-1}(x)}{\partial x_{i}} & \frac{\partial V_j(x)}{\partial x_{i}}\\
\end{vmatrix} < 0
\end{equation}
Multiplying with Equation~\ref{eq:posa} we get,
\begin{equation*}
\begin{vmatrix}
\frac{\partial V_1(x)}{\partial x_1} & \ldots &  \frac{\partial V_1(x)}{\partial x_i}\\
\vdots & \vdots &  \vdots\\
\frac{\partial V_{j-1}(x)}{\partial x_{1}} & \ldots &  \frac{\partial V_{j-1}(x)}{\partial x_i}\\
\frac{\partial V_{j+1}(x)}{\partial x_{1}} & \ldots &  \frac{\partial V_{j+1}(x)}{\partial x_i}\\
\vdots & \vdots &  \vdots\\
\frac{\partial V_{i-1}(x)}{ \partial x_{1}} & \ldots &  \frac{\partial V_{i-1}(x)}{\partial x_{i}}\\
\hat{d}_1 & \ldots &  \hat{d}_i\\
\end{vmatrix} \cdot
    \begin{vmatrix}
\frac{\partial V_1(x)}{\partial x_1} & \ldots &  \frac{\partial V_{j-1}(x)}{\partial x_{1}} 
&\frac{\partial V_{j+1}(x)}{\partial x_{1}}
&\ldots &\frac{\partial V_{i-1}(x)}{\partial x_{1}} & \frac{\partial V_j(x)}{\partial x_{1}}\\
\vdots &  &  \vdots
& \vdots & \vdots &
&\vdots \\
\frac{\partial V_{1}(x)}{\partial x_{j-1}} & \ldots &  \frac{\partial V_{j-1}(x)}{\partial x_{j-1}} 
&\frac{\partial V_{j+1}(x)}{\partial x_{j-1}}
&\ldots &\frac{\partial V_{i-1}(x)}{\partial x_{j-1}} & \frac{\partial V_j(x)}{\partial x_{j-1}}\\
\frac{\partial V_{1}(x)}{\partial x_{j+1}} & \ldots &  \frac{\partial V_{j-1}(x)}{\partial x_{j+1}}
&\frac{\partial V_{j+1}(x)}{\partial x_{j+1}}
&\ldots &\frac{\partial V_{i-1}(x)}{\partial x_{j+1}} & \frac{\partial V_j(x)}{\partial x_{j+1}}\\
\vdots &  &  \vdots
& \vdots & \vdots &
&\vdots \\
\frac{\partial V_{1}(x)}{\partial x_{i}} & \ldots &  \frac{\partial V_{j-1}(x)}{\partial x_{i}} 
&\frac{\partial V_{j+1}(x)}{\partial x_{i}}
&\ldots &\frac{\partial V_{i-1}(x)}{\partial x_{i}} & \frac{\partial V_j(x)}{\partial x_{i}}\\
\end{vmatrix} < 0
\end{equation*}
Now using the fact that $\left(\frac{\partial V_\ell(x)}{\partial x_1},\ldots,\frac{\partial V_\ell(x)}{\partial x_{j-1}},\frac{\partial V_\ell(x)}{\partial x_{j+1}}, \ldots, \frac{\partial V_\ell(x)}{\partial x_i}\right)^\top\cdot (\hat{d}_1,\ldots,\hat{d}_{j-1},\hat{d}_{j+1},\ldots,\hat{d}_{i}) = 0$ (see Definition~\ref{d:directionality1}) implies that 
\begin{equation*}\label{eq:7}
    \begin{vmatrix}
\Phi_1^\top(x)\cdot \Phi_1(x) & \Phi_1^\top(x)\cdot \Phi_2(x) &\ldots &  \Phi_1^\top(x)\cdot \Phi_{i-1}(x) & A_1\\
\Phi_2^\top(x)\cdot \Phi_1(x) & \Phi_2^T(x)\cdot \Phi_2(x) &\ldots &  \Phi_2^T(x)\cdot \Phi_{i-1}(x) & A_2\\
\vdots & \vdots  &\ldots &  \vdots & \vdots\\
\Phi_{i-1}^\top(x)\cdot \Phi_1(x) & \Phi_{i-1}^\top(x)\cdot \Phi_2(x) &\ldots &  \Phi_{i-1}^\top(x)\cdot \Phi_{i-1}(x) & A_{i-1}\\
0 & 0 &\ldots &  0 & (\hat{d}_1,\ldots,\hat{d}_i)^\top \cdot \left(\frac{\partial V_j(x)}{\partial x_1},\ldots,\frac{\partial V_j(x)}{\partial x_i}\right)
\end{vmatrix} < 0
\end{equation*}
where $\Phi_\ell = \left(\frac{\partial V_\ell(x)}{\partial x_1},\ldots,\frac{\partial V_\ell(x)}{\partial x_{j-1}},\frac{\partial V_\ell(x)}{\partial x_{j+1}},\ldots,\frac{\partial V_\ell(x)}{\partial x_{i}} \right)$. The latter implies Claim~\ref{c:main}.

\subsection{Proof of Lemma~\ref{l:Lip2}}\label{l:9}
To simplify notation let $S:=\{1,\ldots,i-1\}$ and  for $x\in [0,1]^n$ consider the matrix $A(x)$ and the vector $b(x)$
\[
A(x):=\begin{pmatrix}
\frac{\partial V_{1}(x)}{\partial x_{1}} & \frac{\partial V_{2}(x)}{\partial x_{1}}&\ldots &\frac{\partial V_{i-1}(x)}{\partial x_{1}}\\ 
\frac{\partial V_{1}(x)}{\partial x_{2}} & \frac{\partial V_{2}(x)}{\partial x_{2}}&\ldots &\frac{\partial V_{i-1}(x)}{\partial x_{2}} \\ 
\vdots& \vdots & \vdots & \vdots & \vdots\\
\frac{\partial V_{1}(x)}{\partial x_{i-1}} & \frac{\partial V_{2}(x)}{\partial x_{i-1}}&\ldots &\frac{\partial V_{i-1}(x)}{\partial x_{i-1}}\\ 
\end{pmatrix}
~~\text{and}~~
b(x):=\begin{pmatrix}
\frac{\partial V_{1}(x)}{\partial x_{i}}\\ 
\frac{\partial V_{2}(x)}{\partial x_{i}}\\ 
\vdots\\
\frac{\partial V_{i-1}(x)}{\partial x_{i}}\\ 
\end{pmatrix}
\]
Notice that since $V_j(x) = 0$ for all $j = 1,\ldots,i-1$, Assumption~\ref{a:1} ensures that the matrix $A(x)$ admits singular value greater than $\sigma_{\min}$ and thus $A(x)$ is invertible. Moreover due to the fact that for all $x,y \in [0,1]^n$ \[\norm{\nabla V_j(x) - \nabla V_j(y)}_2 \leq L \cdot \norm{x-y}_2\]
we get that
\[\norm{A(x) - A(y)}_2 \leq \sqrt{n}L \cdot \norm{x - y}_2 \text{ and }\norm{b(x) - b(y)}_2 \leq L \cdot \norm{x - y}_2.\]
To simplify notation $C_A :=\sqrt{n}L$ and $C_b :=L$. Since $\norm{x - y}_2 \leq \frac{\sigma_{\text{min}}}{2\sqrt{n}L}$ we get that $A(y)$ admits singular value greater than $\sigma_{\text{min}}/2$ and thus $A(y)$ is invertible.

Notice that the direction $D^i_S(x)$ of Definition~\ref{d:directionality1} is either
\[
\left( \frac{A^{-1}(x) \cdot b(x)}{\sqrt{1 + \norm{A^{-1}(x) \cdot b(x)}_2^2}} , \frac{1}{{\sqrt{1 + \norm{A^{-1}(x) \cdot b(x)}_2^2}}}\right)
\text{  or  }
\left(- \frac{A^{-1}(x) \cdot b(x)}{\sqrt{1 + \norm{A^{-1}(x) \cdot b(x)}_2^2}} , -\frac{1}{{\sqrt{1 + \norm{A^{-1}(x) \cdot b(x)}_2^2}}}\right)
\]
depending on the sign of the determinant. We show that for an appropriately selected $L$,
\begin{eqnarray*}
&\norm{\left( \frac{A^{-1}(x) \cdot b(x)}{\sqrt{1 + \norm{A^{-1}(x) \cdot b(x)}_2^2}} , \frac{1}{{\sqrt{1 + \norm{A^{-1}(x) \cdot b(x)}_2^2}}}\right)-
\left( \frac{A^{-1}(y) \cdot b(y)}{\sqrt{1 + \norm{A^{-1}(y) \cdot b(y)}_2^2}} , \frac{1}{{\sqrt{1 + \norm{A^{-1}(y) \cdot b(y)}_2^2}}}\right)}_2\\& \leq M \cdot \norm{x-y}_2
\end{eqnarray*}

In order to prove the above, we use a standard lemma in sensitivity analysis of linear systems.
\begin{lemma}\footnote{\url{https://www.colorado.edu/amath/sites/default/files/attached-files/linearsystems_0.pdf}
}\label{l:sens}
Let the invertible square matrices $A,B$ such that $F:=\norm{(A-B)\cdot A^{-1}}_2 < 1$. Then,
\[\frac{ \norm{A^{-1}b -B^{-1}b}_2}{\norm{A^{-1}b}_2} \leq \frac{\sigma_{\text{max}}(A)}{\sigma_{\text{min}}(A)}\cdot \frac{\norm{F}_2}{1-\norm{F}_2}
\]
\end{lemma}

We prove the following $4$ inequalities,
\begin{itemize}
    \item $\norm{\left( \frac{A^{-1}(x) \cdot b(x)}{\sqrt{1 + \norm{A^{-1}(x) \cdot b(x)}_2^2}} , \frac{1}{{\sqrt{1 + \norm{A^{-1}(x) \cdot b(x)}_2^2}}}\right)-
\left( \frac{A^{-1}(y) \cdot b(x)}{\sqrt{1 + \norm{A^{-1}(x) \cdot b(x)}_2^2}} , \frac{1}{{\sqrt{1 + \norm{A^{-1}(x) \cdot b(x)}_2^2}}}\right)}_2 \leq M_1 \cdot \norm{x-y}_2$
\item $\norm{
\left( \frac{A^{-1}(y) \cdot b(x)}{\sqrt{1 + \norm{A^{-1}(x) \cdot b(x)}_2^2}} , \frac{1}{{\sqrt{1 + \norm{A^{-1}(x) \cdot b(x)}_2^2}}}\right)
-
\left( \frac{A^{-1}(y) \cdot b(x)}{\sqrt{1 + \norm{A^{-1}(y) \cdot b(x)}_2^2}} , \frac{1}{{\sqrt{1 + \norm{A^{-1}(y) \cdot b(x)}_2^2}}}\right)
}_2 \leq M_2 \cdot \norm{x-y}_2$

\item $\norm{
\left( \frac{A^{-1}(y) \cdot b(x)}{\sqrt{1 + \norm{A^{-1}(y) \cdot b(x)}_2^2}} , \frac{1}{{\sqrt{1 + \norm{A^{-1}(y) \cdot b(x)}_2^2}}}\right)
-
\left( \frac{A^{-1}(y) \cdot b(y)}{\sqrt{1 + \norm{A^{-1}(y) \cdot b(x)}_2^2}} , \frac{1}{{\sqrt{1 + \norm{A^{-1}(y) \cdot b(x)}_2^2}}}\right)
}_2 \leq M_3 \cdot \norm{x-y}_2$

\item $\norm{
\left( \frac{A^{-1}(y) \cdot b(y)}{\sqrt{1 + \norm{A^{-1}(y) \cdot b(x)}_2^2}} , \frac{1}{{\sqrt{1 + \norm{A^{-1}(y) \cdot b(x)}_2^2}}}\right)
-
\left( \frac{A^{-1}(y) \cdot b(y)}{\sqrt{1 + \norm{A^{-1}(y) \cdot b(y)}_2^2}} , \frac{1}{{\sqrt{1 + \norm{A^{-1}(y) \cdot b(y)}_2^2}}}\right)
}_2^2 \leq M_4 \cdot \norm{x-y}_2$
\end{itemize}
and then Lemma~\ref{l:Lip} follows for $M:= M_1 + M_2 + M_3 + M_4$.\\
\\
Let the matrix $F:= \left(A(x)-A(y)\right)\cdot A^{-1}(x)$ then 
the fact that $\norm{x-y}_2 \leq \frac{\sigma_{\text{min}}}{2C_A}$ implies,
\begin{equation}\label{eq:1/2}
\norm{F}_2 = \norm{(A(x)-A(y))\cdot A^{-1}(x)}_2 \leq \frac{C_A}{\sigma_{\text{min}}}\cdot \norm{x-y}_2 \leq \frac{1}{2}    
\end{equation}

For the first case we get,
\begin{eqnarray*}
&\norm{\left( \frac{A^{-1}(x) \cdot b(x)}{\sqrt{1 + \norm{A^{-1}(x) \cdot b(x)}_2^2}} , \frac{1}{{\sqrt{1 + \norm{A^{-1}(x) \cdot b(x)}_2^2}}}\right)-
\left( \frac{A^{-1}(y) \cdot b(x)}{\sqrt{1 + \norm{A^{-1}(x) \cdot b(x)}_2^2}} , \frac{1}{{\sqrt{1 + \norm{A^{-1}(x) \cdot b(x)}_2^2}}}\right)}_2^2\\
&= \frac{ \norm{A^{-1}(x) \cdot b(x) - A^{-1}(y) \cdot b(x)}_2^2}{1 + \norm{A^{-1}(x) \cdot b(x)}_2^2}\\
&\leq \frac{ \norm{A^{-1}(x) \cdot b(x) - A^{-1}(y) \cdot b(x)}_2^2}{\norm{A^{-1}(x) \cdot b(x)}_2^2} \\
&\leq \left(\frac{\sigma_{\text{max}}}{\sigma_{\text{min}}} \cdot \frac{\norm{F}_2}{1 - \norm{F}_2}\right)^2 ~~~~~~~~\text{by Lemma}~\ref{l:sens}\\
&\leq \left(\frac{\sigma_{\text{max}}}{\sigma_{\text{min}}} \cdot 2 \cdot \norm{F}_2\right)^2 ~~~~~~~~\text{by Equation}~\ref{eq:1/2}\\
&\leq \frac{\sigma^2_{\text{max}}}{\sigma^2_{\text{min}}} \cdot 4 \cdot \norm{(A(x)-A(y))\cdot A^{-1}(x)}^2_2 \leq 4\frac{\sigma^2_{\text{max}}}{\sigma^4_{\text{min}}} \cdot C_A^2 \cdot \norm{x-y}_2^2\\
\end{eqnarray*}
Thus $M_1:= 2C_A\frac{\sigma_{\text{max}}}{\sigma^2_{\text{min}}}$\\
\\
For the second case
\begin{eqnarray*}
&\norm{
\left( \frac{A^{-1}(y) \cdot b(x)}{\sqrt{1 + \norm{A^{-1}(x) \cdot b(x)}_2^2}} , \frac{1}{{\sqrt{1 + \norm{A^{-1}(x) \cdot b(x)}_2^2}}}\right)
-
\left( \frac{A^{-1}(y) \cdot b(x)}{\sqrt{1 + \norm{A^{-1}(y) \cdot b(x)}_2^2}} , \frac{1}{{\sqrt{1 + \norm{A^{-1}(y) \cdot b(x)}_2^2}}}\right)
}_2^2 \\
&= 
\left(\norm{A^{-1}(y) \cdot b(x)}^2 + 1\right)\frac{\left(\sqrt{1 + \norm{A^{-1}(x) \cdot b(x)}_2^2} - \sqrt{1 + \norm{A^{-1}(y) \cdot b(x)}_2^2} \right)^2}{(1 + \norm{A^{-1}(y) \cdot b(x)}_2^2)\cdot (1 + \norm{A^{-1}(x) \cdot b(x)}_2^2)}\\
&\leq
\left(\norm{A^{-1}(y) \cdot b(x)}^2 + 1\right)\frac{\left(\norm{A^{-1}(x) \cdot b(x)}_2 - \norm{A^{-1}(y) \cdot b(x)} \right)^2}{(1 + \norm{A^{-1}(y) \cdot b(x)}_2^2)\cdot (1 + \norm{A^{-1}(x) \cdot b(x)}_2^2)}~~~~~~~~\text{since~~} \sqrt{1+ b} - \sqrt{1+ a} \leq \sqrt{b} - \sqrt{a}\\
&\leq
\frac{\left(\norm{A^{-1}(x) \cdot b(x)}_2 - \norm{A^{-1}(y) \cdot b(x)} \right)^2}{ \norm{A^{-1}(x) \cdot b(x)}_2^2}\\
&\leq
\frac{\norm{A^{-1}(x) \cdot b(x) - A^{-1}(y) \cdot b(x)}_2^2}{ \norm{A^{-1}(x) \cdot b(x)}_2^2}\\
\end{eqnarray*}
Applying the exact same arguments as before, we get $M_2:= 2C_A\frac{\sigma_{\text{max}}}{\sigma^2_{\text{min}}}$.\\
\\
For the third case,
\begin{eqnarray*}
&\norm{
\left( \frac{A^{-1}(y) \cdot b(x)}{\sqrt{1 + \norm{A^{-1}(y) \cdot b(x)}_2^2}} , \frac{1}{{\sqrt{1 + \norm{A^{-1}(y) \cdot b(x)}_2^2}}}\right)
-
\left( \frac{A^{-1}(y) \cdot b(y)}{\sqrt{1 + \norm{A^{-1}(y) \cdot b(x)}_2^2}} , \frac{1}{{\sqrt{1 + \norm{A^{-1}(y) \cdot b(x)}_2^2}}}\right)
}_2^2\\
&= 
\frac{\norm{A^{-1}(y)\cdot b(y) - A^{-1}(y)\cdot b(x)}_2^2}{ 1 + \norm{A^{-1}(y) \cdot b(x)}_2^2}
\leq \frac{C_b^2}{\sigma^2_{\text{min}}} \cdot 
\norm{x - y}_2^2
\\
\end{eqnarray*}
For the forth case,
\begin{eqnarray*}
&\norm{
\left( \frac{A^{-1}(y) \cdot b(y)}{\sqrt{1 + \norm{A^{-1}(y) \cdot b(x)}_2^2}} , \frac{1}{{\sqrt{1 + \norm{A^{-1}(y) \cdot b(x)}_2^2}}}\right)
-
\left( \frac{A^{-1}(y) \cdot b(y)}{\sqrt{1 + \norm{A^{-1}(y) \cdot b(y)}_2^2}} , \frac{1}{{\sqrt{1 + \norm{A^{-1}(y) \cdot b(y)}_2^2}}}\right)
}_2^2\\
&\leq 
\left(\norm{A^{-1}(y)\cdot b(y)}_2^2 + 1\right)
\cdot
\frac{
\norm{A^{-1}(y)\cdot b(y) - A^{-1}(y)\cdot b(x)}^2
}{\left(1 + \norm{A^{-1}(y) \cdot b(y)}^2 \right) \cdot \left(1 + \norm{A^{-1}(y) \cdot b(x)}^2 \right)}\\
&\leq \norm{A^{-1}(y)\cdot b(y) -  A^{-1}(y)\cdot b(x)} \leq \frac{C_b^2}{\sigma^2_{\text{min}}} \cdot \norm{x-y}_2^2
\end{eqnarray*}
As a result, we overall get that $M:= \Theta\left( \frac{\sigma_{\text{max}}}{\sigma^2_\text{min}} \cdot L \cdot  \sqrt{n}  \right)$.

\subsection{Proof of Lemma~\ref{l:hit}}\label{appendix:l:hit}
To simplify notation let $S:= (1,\ldots,i-1)$ and let $D_S^i(x)$ be denoted as $D(x)$. The existence and uniqueness of trajectory $\gamma(t)$ follows by the Picard–Lindelöf theorem and the fact that $D(x)$ is $M$-Lipschitz continuous (see the proof Lemma~\ref{l:main_new} and Lemma~\ref{l:Lip2}).

We also denote as $\Phi_\ell(x)$ the gradient of $V_\ell(x)$ with respect to the coordinates $\{1,\ldots,i\}$, i.e. $\Phi(x) := \left(\frac{\partial V_\ell(x)}{\partial x_1},\ldots,\frac{\partial V_\ell(x)}{\partial x_i} \right)$.
\noindent To simplify things we repeat the definition of $D(x)$ of Definition~\ref{d:directionality1} with respect to the 
notation of this section.\\

\begin{definition}
Given $x\in [0,1]^i$ the direction $D(x)$ is defined as follows,

\begin{itemize}
    \item $\nabla V_j(x)^\top \cdot
    \left(d_{1},\ldots,d_{i-1},d_i\right) = 0
    $ for all $j = 1,\ldots,i$.
    
    \item the sign of $\begin{vmatrix}
\frac{\partial V_{1}(x)}{\partial x_{1}} & \frac{\partial V_{2}(x)}{\partial x_{1}}&\ldots &\frac{\partial V_{i-1}(x)}{\partial x_{1}} & d_{1}\\ 
\frac{\partial V_{1}(x)}{\partial x_{2}} & \frac{\partial V_{2}(x)}{\partial x_{2}}&\ldots &\frac{\partial V_{i-1}(x)}{\partial x_{2}} & d_{2}\\ 
\vdots& \vdots & \vdots & \vdots & \vdots\\
\frac{\partial V_{1}(x)}{\partial x_{i}} & \frac{\partial V_{2}(x)}{\partial x_{i}}&\ldots &\frac{\partial V_{i-1}(x)}{\partial x_{i}} & d_{i}\\ 
\end{vmatrix}$ equals $\text{sign}\left((-1)^{i-1} \right)$.
\item $d_{1}^2 + \cdots + d_{i-1}^2 + d_{i}^2 = 1$.
\end{itemize}
\end{definition}
Assumption~\ref{a:1} ensures that at any point $x \in [0,1]^i$ the matrix
\begin{equation}\label{eq:sing}
\Phi(x):=\begin{pmatrix}
\Phi_{1}(x)\\
\Phi_{2}(x)\\
\vdots\\
\Phi_{i-1}(x)\\
\end{pmatrix}
:=
\begin{pmatrix}
\nabla V_{1}(x)\\
\nabla V_{2}(x)\\
\vdots\\
\nabla V_{i-1}(x)\\
\end{pmatrix}
\end{equation}
admits singular values greater than $\sigma_{\text{min}}$ and smaller than $\sigma_{\text{max}}$.
\begin{corollary}\label{l:Lip}
For all $x,y \in [0,1]^i$ with $V_\ell(x) = V_\ell(y) = 0$ for $\ell \in \{1,\ldots,i-1\}$,
\[\norm{D(x) - D(y)}_2 \leq M\cdot \norm{x - y}_2\]
for $M:= \Theta(\frac{\sigma_{\text{max}}}{\sigma^2_{\text{min}}}\cdot \sqrt{n} \cdot L)$.
\end{corollary}
Corollary~\ref{l:Lip} follows directly by Lemma~\ref{l:Lip2}. Up next we show that there exist a finite time $t^\ast > 0$ at which $\gamma(t)$ hits the boundary $[0,1]^i$.\\
\begin{claim}~\label{c:4}
For each $t_0$, there exists $t \leq 1/M$ such that $\norm{\gamma(t + t_0) - \gamma(t_0)}_2 \geq \frac{1}{4M}$.
\end{claim}
\begin{proof}
To simplify notation let $t_0:=0$. and let us assume that $\norm{\gamma(t) - \gamma(0)}_2 \leq \frac{1}{4M}$ for all $t \in [0,1/M]$. The latter implies that for all $t_1,t_2 \in [0,1/M]$,
\[\norm{\gamma(t_1) - \gamma(t_2)}_2 \leq \frac{1}{2M}\]
which implies that for all $t_1,t_2 \in [0,1/M]$ 
\[\norm{D(\gamma(t_1)) - D(\gamma(t_2))}_2 \leq \frac{1}{2}.\] 
Using the fact that $\norm{D(\gamma(t_1))}_2 = \norm{D(\gamma(t_2))}_2 = 1$ we get that, 
\[ D^\top(\gamma(t_1))\cdot D(\gamma(t_2)) \geq 1/2\]
As a result, 
\begin{eqnarray*}
\norm{\gamma(1/M) - \gamma(0)}_2^2 &=& \norm{\int_{0}^{1/M}D(\gamma(s))~\partial s}^2\\
&=& \int_{0}^{1/M}\int_{0}^{1/M} D^\top(\gamma(s))\cdot D(\gamma(s')) ~\partial s~\partial s'\geq \frac{1}{2M^2}  
\end{eqnarray*}
\noindent and thus $\norm{\gamma(1/M) - \gamma(0)}_2 \geq \frac{1}{\sqrt{2}M}$ which is a contradiction. 
\end{proof}

\begin{claim}\label{c:2}
For any $t_0$, there exist $0 \leq t_1,t_2 \leq \frac{1}{M}$ such that
\begin{enumerate}
    \item $\norm{\gamma(t_0 + t_1) - \gamma(t_0)}_2 \geq \frac{1}{4M}$.
    \item $\norm{\gamma(t_0 - t_2) - \gamma(t_0)}_2 \geq \frac{1}{4M}$
\end{enumerate}
\end{claim}
\begin{proof}
Symmetrically as Claim~\ref{c:4}.
\end{proof}

\begin{lemma}\label{l:zero}
Let $\delta \leq 1/4$ and $\gamma \in [0,1]^n$ such that $\norm{\gamma(t_0) - \gamma} \leq \frac{\delta}{2M}$. Then there exists 
$t^\ast \in [-1/M,1/M]$ such that
\begin{itemize}
    \item $\norm{\gamma(t^\ast + t_0) - \gamma(t_0)}_2 \leq \frac{\delta}{M}$.
    \item $D^\top( \gamma(t^\ast + t_0)) \cdot (\gamma(t^\ast +t_0) - \gamma) = 0$.
\end{itemize}
\end{lemma}
\begin{proof}
By Claim~\ref{c:2} there exists $0 \leq t_1 \leq 1/M $ such that $\norm{\gamma(t_1+t_0) - \gamma(t_0)} \geq \frac{1}{4M}$. Let $t' = \text{inf}_{0 \leq t \leq 1/M}\{\norm{\gamma(t +  t_0) - \gamma(t_0)}_2 \geq \frac{\delta}{M}\}$.
By the triangle inequality, we get that $\norm{\gamma(t' + t_0) - \gamma}_2 \geq \frac{\delta}{2M}$ and thus there exists  $\hat{t}_1 \in [0,t']$ such that
\begin{itemize}
    \item $\norm{\gamma(\hat{t}_1 + t_0) - \gamma}_2 = \frac{\delta}{2M}$
    \item $\norm{\gamma(t + t_0) - \gamma}_2 < \frac{\delta}{2M}$ for $t \leq \hat{t}_1$.
\end{itemize}
The latter implies that 
\begin{itemize}
    \item $\norm{\gamma(t + t_0) - \gamma(t_0)}_2 \leq \frac{\delta}{M}$ for all $0 \leq t \leq \hat{t}_1$.
    \item $D^\top\left(\gamma(t_0 + \hat{t}_1)\right) \cdot \left(\gamma(t_0 + \hat{t}_1) - \gamma\right) \geq 0$.
\end{itemize}
Symmetrically we can prove that there exists $\hat{t_2}$ such that 
\begin{itemize}
    \item $\norm{\gamma(t_0 - t) - \gamma(t_0)}_2 \leq \frac{\delta}{M}$ for all $0 \leq t \leq \hat{t}_2$.
    \item $D^\top\left(\gamma(t_0 - \hat{t}_2)\right)\cdot(\gamma(t_0 - \hat{t}_2) - \gamma) \leq 0$.
\end{itemize}
The proof follows by continuity of $g(t):= D^\top\left(\gamma(t_0 + t)\right)\cdot(\gamma(t_0 + t) - \gamma)$ for $t \in [-\hat{t}_2,\hat{t}_1]$.
\end{proof}
Up next we present the main lemma of the section.

\begin{lemma}\label{l:8}
Let $\rho = \Theta\left(\frac{\sigma^3_{\text{min}}}{\sqrt{n}\cdot \sigma^2_{\text{max}}\cdot L}\right)$ and a point $p \in \mathbb{B}(\gamma(t_0),\rho)$ with $p \notin \gamma[t- 1/M , t+ 1/M]$ then $V_\ell(p) \neq 0$ for some $j \leq i-1$.
\end{lemma}
\begin{proof}
Let $\delta = M \cdot \rho$ and assume the that $V_\ell(p)= 0$ for all $j \leq i-1$. By Lemma~\ref{l:zero} we get that there exists $t^\ast \in [t-1/M,t+1/M]$ such that
\begin{enumerate}
    \item $\norm{\gamma(t_0 + t^\ast) - \gamma(t_0)}_2 \leq \frac{\delta}{M}$.
    
    \item  $D^\top( \gamma(t^\ast + t_0)) \cdot (\gamma(t^\ast +t_0) - p) = 0$.
\end{enumerate}
Using the fact that the matrix $\Phi\left( \gamma(t+t_0)\right)$ admits singular value greater than $\sigma_{\text{min}}$ we get that,
\begin{enumerate}
    \item $\norm{\gamma(t_0 + t^\ast) - \gamma(t_0)}_2 \leq \frac{\delta}{M}$.
    
    \item $p = \gamma(t_0 + t^\ast) + \sum_{j=1}^{i-1}\mu_j\cdot \Phi_j\left(\gamma(t_0 + t^\ast)\right)$.
\end{enumerate}
By the fact that $\norm{\gamma(t_0) - p}_2 \leq \frac{\delta}{M}$ (recall that $\delta = M \cdot \rho$) we get that,
\[
\norm{\sum_{j=1}^{i-1}\mu_j \cdot \Phi_j\left(\gamma(t_0 + t^\ast)\right)}_2 = \norm{\gamma(t_0 + t^\ast) - p}_2
 \leq \norm{\gamma(t_0) - \gamma(t_0 + t^\ast)}_2 + \norm{\gamma(t_0) - p}_2 \leq 2 \frac{\delta}{M}\]
and thus
\begin{equation}\label{eq:sing2}
\norm{\mu}_2 \leq  \frac{2\delta}{\sigma_{\text{min}} \cdot M}    
\end{equation}
Recall that 
$\norm{\Phi_j(x) - \Phi_j(y)}_2 \leq L \cdot \norm{x - y}_2$ and thus by applying the Taylor expansion on $V_j(\cdot)$ we get that
\[\left|V_j(p) - V_j\left(\gamma(t_0 + t')\right) -\Phi_\ell^\top\left(\gamma(t + t^\ast) \right) \cdot\sum_{j=1}^{i-1} \mu_j \Phi_j\left(\gamma(t + t^\ast)\right) \right| \leq \Theta\left(L \cdot \norm{\gamma(t+t_0) - p}_2^2\right)\]
Since $V_\ell(p) =V_\ell\left(\gamma(t+t^\ast) \right) = 0$
\[\left|\Phi_\ell^\top\left(\gamma(t + t^\ast) \right) \cdot\sum_{j=1}^{i-1} \mu_j \Phi_j\left(\gamma(t + t^\ast)\right) \right| \leq \Theta\left( L \cdot \norm{\sum_{j=1}^{i-1}\mu_j \cdot \Phi_j\left(\gamma(t+t^\ast) \right)}^2 \right)\leq \Theta\left(L \cdot \sigma_{\text{max}}^2 \cdot \norm{\mu}_2^2\right)\]
meaning that
$|\left[\Phi^T \cdot \Phi \cdot \mu\right]_\ell| \leq \Theta\left(L \cdot \sigma_{\text{max}}^2 \cdot \norm{\mu}_2^2\right)$
and thus \[\sigma_{\text{min}}^2 \norm{\mu}_2 \leq \norm{V^T \cdot V \mu}_2 \leq \Theta\left(\sqrt{n} \cdot L \cdot \sigma_{\text{max}}^2 \cdot \norm{\mu}_2^2\right) \rightarrow \norm{\mu}_2 \geq \Theta\left(\frac{\sigma_{\text{min}}^2}{\sqrt{n} \cdot L \cdot \sigma_{\text{max}}^2}\right)\]
selecting $\delta \geq \Theta\left(\frac{\sigma_{\text{min}}^3\cdot M}{\sqrt{n} \cdot L \cdot \sigma_{\text{max}}^2}\right)$ leads to contradiction.
\end{proof}
We conclude the section with the proof of Lemma~\ref{l:hit}. Let $\text{Vol}^n(\rho)$ denote the volume of $n$-dimensional ball with radius $\rho$ and let us assume that $\gamma(t) \in [0,1]^n$ for all $t \in (0, \frac{2\cdot 2^n}{M\cdot \text{Vol}^n(\rho/2)}]$ where $\rho = \Theta\left(\frac{\sigma^3_{\text{min}}}{\sqrt{n}\cdot \sigma^2_{\text{max}}\cdot L}\right)$.\\

Let $t_i := t_1 + \frac{2i}{M}$ and let the ball $\mathbb{B}_i:= \mathbb{B}(\gamma(t_i),\rho/2)$ where $\rho = \Theta\left(\frac{\sigma^3_{\text{min}}}{\sqrt{n}\cdot \sigma^2_{\text{max}}\cdot L}\right)$. Thus there are $\frac{2^n}{\text{Vol}^n(\rho/2)}$ such balls. Notice that by Lemma~\ref{l:8},
$\mathbb{B}_i \cap \mathbb{B}_j = \varnothing$ for $i \neq j$ and thus
the latter is a contradiction due to the fact that $\mathbb{B}_i \cap [0,1]^i$ are disjoint sets with volume greater than $\frac{\text{Vol}^n(\rho/2)}{2^n}$.

\subsection{Proof of Lemma~\ref{l:new}}
Let $Z^\ell(x)=\{\text{coordinates }j\leq \ell-1 \text{ such that }V_j(x) = 0\}$ and $F^\ell(x)=\{\text{coordinates }j\leq \ell-1 \text{ that are satisfied at } x\}$\\
By the definition of pivot we known that $V_\ell(x) > 0$. In case $x_\ell \in (0,1)$ then the coordinate is not frozen and the first item of Lemma~\ref{l:new} follows. In case $x_\ell =0 $ and $[D_{Z^\ell(x)}(x)]_\ell \geq 0$ then again the first item follows. As a result, without loss of generality we assume that $[D_{Z^\ell(x)}(x)]_\ell <0 $ and $x_\ell =0$. 

At first notice that in case $Z^\ell(x) = \varnothing$ then by Definition~\ref{d:directionality1} we get that $[D^\ell(x)]_\ell= 1$ which contradicts with the fact that coordinate $\ell$ is frozen. Also notice that since $x_\ell=0$, Assumption~\ref{a:2} implies that $x_j\in (0,1)$ for all coordinates $j \in Z^\ell(x)$.\\
\\
Let assume that $x_{\ell-1} =0$. As discussed above, Assumption~\ref{a:2} implies that $\ell-1 \notin Z^\ell(x)$ and thus $Z^{\ell-1}(x) = Z^{\ell}(x)$ which implies that $\mathrm{sign}\left([D^{\ell-1}(x)]_{\ell-1}\right) = \mathrm{sign}\left([D^{\ell}(x)]_{\ell}\right)$ and thus coordinate $\ell-1$ is also frozen. As a result, the only candidate is the coordinate
\[i:= \text{ the maximum } k \leq \ell \text{ with }x_k >0\]
Note the existence of such a coordinate is guaranteed by the fact that $Z^\ell(x) \neq \varnothing$ and by the fact that for all $j \in Z^\ell(x)$, $x_j \in (0,1)$ (Assumption~\ref{a:2}).

Let us consider the case where $x_i = 1$. Notice again that by Assumption~\ref{a:2}, $i \notin Z^{i + 1}(x) = Z^{\ell}(x)$ and thus $Z^i(x) = Z^{i + 1}(x) = Z^{\ell}(x)$ which implies that $[D^i(x)]_i < 0$. Thus coordinate $i$ is not frozen
and at the same time $V_i(x) > 0$ since coordinate $i$ is satisfied at $x$ and $V_i(x) \neq 0$ (Assumption~\ref{a:2}).\\
\\
Now let us consider the case where $x_i \in (0,1)$ and coordinate $i$ is not frozen. Due to the fact that $x$ is a pivot and thus coordinate $i$ is satisfied, we get that $V_i(x) =0$ and thus $i \in Z^\ell(x)$. Let $D^\ell(x):= (d_1,\ldots,d_i,d_\ell)$ and $D^i(x):= (\hat{d}_1,\ldots,\hat{d}_i)$. Let us assume that $|Z^\ell(x)|$ is even. Then by Definition~\ref{d:directionality1} we get that,
\begin{equation}
    \begin{vmatrix}
\frac{\partial V_1(x)}{\partial x_1}
&\ldots &\frac{\partial V_{i}(x)}{\partial x_{1}} & d_1\\
\vdots &  &  \vdots & \vdots\\
\frac{\partial V_{1}(x)}{\partial x_{i}} & \ldots &\frac{\partial V_{i}(x)}{\partial x_{i}} & d_{i}\\
\frac{\partial V_{1}(x)}{\partial x_{\ell}} & \ldots &\frac{\partial V_{i}(x)}{\partial x_{\ell}} & d_{\ell}\\
\end{vmatrix} > 0
\end{equation}
Since $d_{\ell} < 0$ then we get that
\begin{equation}\label{eq:narek}
    \begin{vmatrix}
\frac{\partial V_1(x)}{\partial x_1}
&\ldots &\frac{\partial V_{i}(x)}{\partial x_{1}} & d_1\\
\vdots &  &  \vdots & \vdots\\
\frac{\partial V_{1}(x)}{\partial x_{i}} & \ldots &\frac{\partial V_{i}(x)}{\partial x_{i}} & d_{i}\\
0 & \ldots &0 & d_1^2 + \ldots + d_i^2 + d_{\ell}^2\\
\end{vmatrix} < 0
\end{equation}
implying that
\begin{equation}
    \begin{vmatrix}
\frac{\partial V_1(x)}{\partial x_1}
&\ldots &\frac{\partial V_{i}(x)}{\partial x_{1}}\\
\vdots &  &  \vdots \\
\frac{\partial V_{1}(x)}{\partial x_{i}} & \ldots &\frac{\partial V_{i}(x)}{\partial x_{i}} \\
\end{vmatrix} < 0
\end{equation}
Since $|Z^\ell(x)|$ is even then 
$|Z^i(x)|$ is odd ($Z^\ell(x) =Z^i(x) \cup \{i\} $) and thus
by Definition~\ref{d:directionality1}
\begin{equation}\label{eq:yerek}
    \begin{vmatrix}
\frac{\partial V_1(x)}{\partial x_1}
&\ldots &\frac{\partial V_{i-1}(x)}{\partial x_{1}} & \hat{d}_1\\
\vdots &  &  \vdots & \vdots\\
\frac{\partial V_{1}(x)}{\partial x_{i}} & \ldots &\frac{\partial V_{i-1}(x)}{\partial x_{i}} & \hat{d}_{i}
\end{vmatrix} < 0
\end{equation}
Multiplying Equation~\ref{eq:yerek} with Equation~\ref{eq:narek} we get that,
\[(\hat{d}_1,\ldots,\hat{d}_i)^\top \cdot \left(\frac{\partial V_i(x)}{\partial x_1},\ldots,\frac{\partial V_i(x)}{\partial x_i}\right) > 0\]

\section{Proof of Lemma~\ref{l:indegree}}
Let the pivots $x_1$ and $x_2$ with admissible pairs $(i,S)$ and $(i',S')$ respectively. Consider the trajectory $\dot{z}(t) =D_S^i(z(t))$ with $z(0) = x_1$ and the trajectory
$\dot{y}(t) = D_{S'}^{i'}(y(t))$ for $y(0) = x_2$ where $x_1 \neq x_2$.

We first assume that
$x^\ast = z(t_1)$ for some $t_1$ and $x^\ast = y(t_2)$ for some $t_2$ where $\dot{y}(t) = D_{S'}^{i'}(y(t))$ for $y(0) = x_2$ and we will reach a contradiction. Let $M:= (S' / S) \cup (S / S')$.\\

\begin{itemize}
    \item \underline{$|M| \geq 2$}: 
    
    \begin{itemize}
        \item $\underline{ i = i'}$: In this case $i \notin S'$ and $i' \notin S$. Let $\ell_1,\ell_2 \in M$. We will show that $\ell_1$ (resp. $\ell_2$) lies on the boundary ($x_{\ell_1} =0 $ or $x_{\ell_1} =0 $) and $V_{\ell_1}(x^\ast) =0$. Once the latter is established, consider the set of coordinates $A:= S\cup S' \cup\{i\} \cup\{i'\}$ . Notice that $x_j^\ast=0$ or $x_j^\ast=1$ for any coordinates $j \notin A$. At the same time there exist two coordinates $\ell_1,\ell_2 \in A$ that both lie on the boundary and admit $V_{\ell_1}(x^\ast) = V_{\ell_2}(x^\ast) = 0$. The latter violates Assumption~\ref{a:2}.\\
        
        Up next we establish that $x^\ast_{\ell_1} =0$ and $V_{\ell_1}(x^\ast) =0$. Without loss of generality let $\ell_1 \in S' / S$. Since $x' = \mathrm{Next}(x_2)$ and $\ell_1 \in S'$ then Lemma~\ref{l:outdegree} implies that $V_{\ell_1}(x^\ast) =0$. At the same time since $\ell_1 \notin S$ and $\ell_1 \neq i$, we get that either $x_{\ell_1}^1 =0$ or $x_{\ell_1}^1 =1$. Since $\ell_1 \notin S$ the coordinate $\ell_1$
stands still in the trajectory $\dot{z}(t) := D_S^i(z(t))$ with $z(0) = x_1$ and thus $x_{\ell_1}^\ast = 1$ or $x_{\ell_1}^\ast = 0$.   \smallskip
 \smallskip
 
    \item \underline{$i' > i$}: Since $x_1$ is a pivot at which $i$ is the under examination coordinate. By Definition~\ref{d:pivot_new} we get that $x_{i'}^1 = 0$. The latter implies that $x_{i'}^\ast = 0$ since coordinate $i'$ stands still in the trajectory $\dot{z}(t) = D_S^i(z(t))$ with $z(0) = x_1$. Consider the set $A:= \{ j \leq i' -1:~V_j(x^\ast) =0\}$. Since $x_j^\ast =0$ or $x_j^\ast =1$ for all $j \notin A\cup\{i'\}$ and $x_{i+1}^\ast =0$, Assumption~\ref{a:2} implies that $x_j^\ast \in (0,1)$ for all coordinates $j \in A$. Then Definition~\ref{d:directionality2} and Definition~\ref{d:admissible_pairs} imply that $S = \{j \leq i-1:~V_j(x^\ast) =0\}$.\\
    
    Since $(i,S)$ is the admissible pair of pivot $x_1$, $x_j^1 =0$ for $j \geq i+1$ which implies that $x_j^\ast = 0$ for all $j \geq i+1$. As a result, $V_j(x^\ast)\neq 0$ for all $j \geq i+1$. Then Definition~\ref{d:directionality2} and Definition~\ref{d:admissible_pairs} imply that $S' \subseteq \{j \leq i-1:~V_j(x^\ast)\} \cup \{i\} = S \cup \{i\}$. However the latter contradicts with the fact that $|M|=2$. 
    \end{itemize}

\smallskip
\item \underline{$|M| =1$ and $i' > i$}: Without loss of generality we consider $\ell \in S' / S$. Since $\ell \in S'$, by Lemma~\ref{l:outdegree} we get that $V_\ell(x^\ast) =0$. At the same time since $i$ is the under examination coordinate at $x_1$ and $i' > i$ by Lemma~\ref{l:outdegree} we get that $x_{i'}^\ast = 0$. 
\begin{itemize}
    \item \underline{$\ell \neq i$}: Since $\ell \neq i$ and $\ell \notin S$, we get that either $x_\ell^1 = 0$ or $x_\ell^1 = 1$. Since coordinate $\ell$ stands still in the trajectory $\dot{z}(t):= D^i_S(z(t))$ with $z(0) = x_1$ we get that $x^\ast_\ell =0$ or $x^\ast_\ell =1$.\\
    
    Since $x^\ast = \mathrm{Next}(x_1)$ and $i$ is the under examination variable at $x_1$, by Lemma~\ref{l:outdegree} we know that $x_j^\ast =0 $ for all $j \geq i+1$. As a result, $x_{i'}^\ast =0$. Now consider the set of coordinates $A = \{j \leq i'-1:~ V_j(x^\ast) =0\} \cup \{i'\}$. Notice that $x_j^\ast=0$ or $x_j^\ast = 1$ for all coordinates $j \notin A$. At the same time both coordinates $i',\ell \in A$ lie on the boundary at $x^\ast$. The latter contradicts with Assumption~\ref{a:2}.
    \smallskip
    \smallskip
    \item \underline{$\ell = i$}: In this case $S' = S\cup \{i\}$. Since $i' > i$ and $i$ is the under examination coordinate at point $x_1$ we get that $x_{i'}^\ast =0$. Due to the fact that $i'$ is the under examination coordinate at $x_2$ we also get that $[D^{i'}_{S'}(x^\ast)]_{i'} = [D^{i'}_{S\cup \{i\}}(x^\ast)]_{i'} \leq 0$ while Assumption~\ref{a:3} implies $[D^{i'}_{S'}(x^\ast)]_{i'} = [D^{i'}_{S\cup \{i\}}(x^\ast)]_{i'} < 0$. The latter implies that $D^{i}_{S}(x^\ast)^\top \cdot \nabla V_i(x^\ast) > 0$\footnote{See Equations~(14)-(17) in the proof of Lemma~\ref{l:new}.}.\\
    
    Since $i \in S'$, Lemma~\ref{l:outdegree} implies that $V_i(x^\ast) =0 $. Since $i$ is the under examination coordinate at $x_1$, by Lemma~\ref{l:new} we get that $V_i(z(t))>0$ for all $t\in (0,\delta)$ where $\delta$ is sufficiently small. Since $V_i(x^\ast) =0 $ the latter implies that $D^{i}_{S}(x^\ast)^\top \nabla V_i(x^\ast) \leq 0$ which is a contradiction.
    
\end{itemize}

\smallskip
\item \underline{$|M| =1$ and $i' = i$}: 
Consider $\ell \in S' / S$. Since $x^\ast = \mathrm{Next}(x_2)$ by Lemma~\ref{l:outdegree} we get that $V_\ell(x^\ast) =0$ since $\ell \in S'$. By the fact that $\ell \notin S$, $i \neq \ell$ and $x_1$ is a pivot, we know by Definition~\ref{d:pivot_new} that either $x_\ell^1 =0$ or $x_\ell^1 =1$. Since $[D^i_S(\cdot)]_\ell=0$, we get that $x^\ast_\ell =0$ or  $x^\ast_\ell =1$.\\

Let us consider the following mutually exclusive cases,
\begin{itemize}
    \item \underline{$x_i =0$ or $x_i = 1$}:
    Consider the set of coordinates $A = \{j \leq i-1:~ V_j(x^\ast) =0\} \cup \{i\}$. For all coordinates $j \in S$ it holds $V_j(x^\ast) =0 $ while for $j \notin S$, $x_j^\ast =0$ or $x_j^\ast =1$. Since $S \subseteq A$, all coordinates $j \notin A$ admit $x_j^\ast =0$ or $x_j^\ast =1$. Notice that both coordinates $i,\ell \in A$ lie on the boundary at $x^\ast$ which contradicts Assumption~\ref{a:2}.
    \smallskip
        \item \underline{$x_i \in (0,1)$ and $V_i(x^\ast) = 0$}: Consider the set $A:=  \{j \leq i-1:~ V_j(x^\ast) =0\} \cup \{i\}\cup \{i+1\}$.  For all coordinates $j \in S$ it holds $V_j(x^\ast) =0 $ while for $j \notin S$, $x_j^\ast =0$ or $x_j^\ast =1$. Since $S \subseteq A$, all coordinates $j \notin A$ admit $x_j^\ast =0$ or $x_j^\ast =1$. At the same time coordinates $\ell , i+1 \in A'$ lie on the boundary at $x^\ast$. The latter violates Assumption~\ref{a:2}.
    \smallskip    
        \item \underline{$x_i \in (0,1)$ and $V_i(x^\ast) > 0$}: Without loss of generality we assume that $x^\ast_\ell =0$. By Lemma~\ref{l:outdegree} we know that $\ell$ remains satisfied during the trajectory $\dot{z}(t) = D_S^i(z(t))$ with $z(0) = x_1$. Moreover $z_\ell(t) =0$ since $[D^\ell_S(z(t))]_\ell =0$ and $x^1_\ell =0$. The latter implies that $D^i_S(x^\ast)^\top \cdot \nabla V_\ell(x^\ast) \geq 0$.\\
        
        Since $\ell \in S'$, we get that $y_\ell(t) \in (0,1)$ for $t \in (0,t_2)$. The latter implies that $[D^i_{S'}(x^\ast)]_{\ell} =[D^i_{S \cup \{\ell\}}(x^\ast)]_{\ell} \leq 0$. By Assumption~\ref{a:3} we additionally get that $[D^i_{S'}(x^\ast)]_{\ell} =[D^i_{S \cup \{\ell\}}(x^\ast)]_{\ell} < 0$. Then Lemma~\ref{c:main} implies that $D^i_S(x^\ast)^\top \cdot \nabla V_\ell(x^\ast) < 0$ which is a contradiction. 
\end{itemize}
\smallskip

\item \underline{$|M| =0$ and $i' > i$}: Without loss of generality let us assume that $i' > i$. Since $i$ is the under examination variable at $x_1$, Lemma~\ref{l:outdegree} implies that $x_{i'}^\ast =0$. Since
$y_{i'}(t) \in [0,1]$ for all $t \in [0,t_2]$ we additionally get that $[D^{i'}_S(x^\ast)]_{i'} < 0$. Since $ i \notin S$, $\mathrm{sign}([D_S^{i}(x^\ast)]_{i}) = \mathrm{sign}([D_S^{i'}(x^\ast)]_{i'}) < 0$ (see the proof of Lemma~\ref{l:new}).\\

Since $i < i'$ we know that coordinate $i$ is satisfied at $x^\ast$ and thus one of the following holds,
\begin{itemize}
    \item $x_i^\ast \in (0,1)$ and $V_i(x^\ast) =0$.
    \item $x_i^\ast = 1$ and $V_i(x^\ast) \geq 0$.
    \item $x_i^\ast = 0$ and $V_i(x^\ast) \leq 0$.
\end{itemize}
Since $i \notin S$ coordinate $i$ lies on the boundary at point $x_2$ while it stands still in the trajectory $\dot{y}(t) = D_{S}^{i'}(y(t))$ with $y(0) = x_2$. Thus $x_i^\ast =0$ or $x_i^\ast =1$. The latter excludes the first case. Since coordinate $i$ is under examination at $x_1$, in case $x_i^\ast = 1$ then $[D_S^i(x^\ast)]_i \geq 0$ which contradicts with the fact that $[D_S^i(x^\ast)]_i < 0$. Up next we exclude the third case where $x_i^\ast =0$ and $V_i(x^\ast) \leq 0$. By Lemma~\ref{l:new} we know that $V_i(z(t)) >0$ for 
all $t \in (0,\delta)$ once $\delta$ is selected sufficiently small. The latter together with the fact that $V_i(x^\ast)\leq 0$ implies that $V_i(x^\ast) = 0$. Now consider the set $A:= S \cup\{i\}\cup\{i'\}$ and notice that $x_j^\ast =0$ for all $j \notin A$. The fact that $x^\ast_i = x^\ast_{i'} = 0$ and $V_j(x^\ast) =0$ for all $j \in S \cup \{i\}$ contradicts with Assumption~\ref{a:2}.

\smallskip
\item \underline{$|M| =0$ and $i' = i$}: In case $i = i'$ then $\dot{z}(t) = D_S^i(z(t))$ and 
$\dot{y}(t) = D_S^i(y(t))$. The Lipschitz-continuity of $D_S^i(\cdot)$ implies that $z(t_1) = y(t_2)$ can only occur in case $x_1 = x_2$. 

\end{itemize}

\section{Proof of Lemma~\ref{l:0}}
Let the trajectory $\dot{z}(t) = D_S^i(z(t))$ with $z(0) = x$ where $x$ is a pivot and $(i,S)$ is an admissible pair for pivot $x$. Let us assume that there exists $t^\ast$ such that $z(t^\ast)=(0,\ldots,0)$ and $z(t)$ is not a pivot for all $t \in (0,t^\ast)$. Notice that in case $|S| =0$ then $[D_S^i(z(t))]_i = 1$ which leads to a contradiction. As a result, $|S| \geq 1$. Notice that for all $ j \in S$, $V_j(z(t)) =0$ for all $t \in [0,t^\ast]$ and thus $V_j(0,\ldots,0) =0$ for all coordinates $j \in S$. Now consider the set $A =\{j \leq i-1:~ V_j(0,\ldots,0) =0\}\cup\{i\}$. Notice that all coordinates $j \in A$ admit $V_j(0,\ldots,0) =0$ and Assumption~\ref{a:3} is violated.

\section{Discrete-Time Dynamics} \label{sec:discrete}

We begin with the adaptation of the Dynamics \ref{dyn:Sperner} to discrete-time algorithms. The main change we need to make is to change the step 5 of Dynamics \ref{dyn:Sperner} to the following $z^{(k + 1)} \leftarrow z^{(k)} + D^i_S(z^{(k)})$. But then we need also to adapt the notion of exit points as follows.

\begin{definition}[$(\eps, \gamma)$-Exit Points] \label{def:exitConditionsDiscrete}
    Suppose $i\in [n]$, $S \subseteq [i-1]$ and $x'$ is a point where coordinates in $S$ are zero-satisfied and coordinates in $[i-1]\setminus S$ are boundary-satisfied. Then $x'$ is an {\em exit point} for epoch $(i, S)$ iff it satisfies one of the following:
    \begin{itemize}
        \item \textbf{(Good Exit Point)}: 
        Coordinate $i$ is almost satisfied at $x'$, i.e., $\norm{V_i(x')} \le \eps$, or $x'_i = 0$ and $V_i(x') < \eps$, or $x'_i = 1$ and $V_i(x') > -\eps$.
        \item \textbf{(Bad Exit Point)}:  For some $j \in S \cup \{i\}$, it holds that $(D^i_S(x'))_j > 0$ and $x'_j = 1$, or $(D^i_S(x'))_j < 0$ and $x'_j = 0$; in other words, if the dynamics for epoch $(i,S)$ were to continue from $x'$ onward, they would violate the constraints.
        \item \textbf{(Middling Exit Point)}: Let $x'' = x' + \gamma D^i_S(x')$ and for some $j \in [i - 1] \setminus S$, one of the following holds: $V(x''_j) > 0$ and $x'_j = 0$, or $V(x''_j) < 0$ and $x'_j = 1$; in other words, if the dynamics for epoch $(i,S)$ were to continue from $x'$ onward, some boundary-satisfied coordinate would become unsatisfied.
    \end{itemize}
\end{definition}

\noindent We next present our solution concept for the discrete-time dynamics.

\begin{definition} \label{def:approximateVI}
  We say that a point $x$ is an $\alpha$-approximate solution of $\mathrm{VI}(V, [0, 1]^n)$ if and only if $V(x)^{\top} (x - y) \le \alpha$.
\end{definition}

\noindent We also define $\Pi : \R^n \to [0, 1]^n$ to be the Euclidean projection of a vector in $\R^n$ to the hypercube $[0, 1]^n$. In Dynamics \ref{dyn:SpernerDiscrete} we define our discrete-time dynamics for which we show Theorem \ref{thm:mainDiscrete}.

\begin{algorithm}[t]
  \caption{Discrete STay-ON-the-Ridge with step size $\gamma$ and errors $\alpha$, $\eps$}\label{dyn:SpernerDiscrete}
 \begin{algorithmic}[1]
 \STATE Initially $x^{(0)} \leftarrow (0,\ldots,0)$, $i \leftarrow 1$, $S \leftarrow \emptyset$, $m \leftarrow 0$.

 \WHILE {$x^{(m)}$ is not an $\alpha$-approximate VI solution}
        \STATE $z^{(0)} \leftarrow x^{(m)}$
        \WHILE{$\Pi(z^{(k)})$ is not an $(\eps, \gamma)$-exit point as per Definition \ref{def:exitConditionsDiscrete}}
        \smallskip
          \STATE $z^{(k + 1)} \leftarrow z^{(k)} + \gamma \cdot D^i_S(z^{(k)})$
          \STATE $k \leftarrow k + 1$
        \smallskip
        \ENDWHILE
        \STATE $x^{(m + 1)} \leftarrow \Pi(z^{(k)})$ 
        \smallskip
        \IF{$x^{(m + 1)}$ is a (Good Exit Point) as in Definition \ref{def:exitConditionsDiscrete}}
        \IF{$i$ is zero-satisfied at $x{(t+1}$}
        \STATE Update $S \leftarrow S \cup \{i\}$.
        \ENDIF
        \STATE Update $i \leftarrow i + 1$.
        \ELSIF{$x^{(m + 1)}$ is a (Bad Exit Point) as in Definition \ref{def:exitConditionsDiscrete} for $j = i$}
        \STATE Update $i \leftarrow i - 1$ and $S \leftarrow S \setminus \{i - 1\}$.
        \ELSIF{$x^{(m + 1)}$ is a (Bad Exit Point) as in Definition \ref{def:exitConditionsDiscrete} for $j \neq i$}
        \STATE Update $S \leftarrow S \setminus \{j\}$.
        \ELSIF{$x^{(m + 1)}$ is a (Middling Exit Point) as in Definition \ref{def:exitConditionsDiscrete} for $j < i$}
        \STATE Update $S \leftarrow S \cup \{j\}$.
        \ENDIF
        \STATE Set $m \leftarrow m + 1$.
\ENDWHILE
  \STATE \textbf{return}$x^{(m)}$
  \end{algorithmic}
\end{algorithm}

\begin{theorem} \label{thm:mainDiscrete}
  We assume Assumptions \ref{a:1}, \ref{a:2}, and \ref{a:3}. For every $\alpha > 0$, there exist constants $\eps$, $\gamma$, $\bar{M}$, $K$ such that Dynamics \ref{dyn:SpernerDiscrete} with step size $\gamma$ and error $\eps$ finish after $M \le \bar{M}$ iterations of the while-loop at line 2 and it holds that $x^{(M)}$ is an $\alpha$-approximate solution of $\mathrm{VI}(V, [0, 1]^n)$. Additionally, for every iteration $m \le M$ of the while-loop in line 2, the while-loop in line 4 does at most $K$ iterations.
\end{theorem}

\begin{proof}
  The main idea of the proof is to show that, for sufficiently small step size $\gamma$, the Dynamics \ref{dyn:SpernerDiscrete} will always stay in Euclidean distance at most $\delta := \alpha/\Lambda$ from the continuous-time Dynamics \ref{dyn:Sperner}. Then, since Dynamics \ref{dyn:Sperner} converge to a solution of $\mathrm{VI}(V, [0, 1]^n)$ (see Theorem \ref{t:main}) and since $V$ is $\Lambda$-Lipschitz we conclude that the discrete Dynamics \ref{dyn:SpernerDiscrete} will converge to a point that is an $\alpha$-approximate solution of $\mathrm{VI}(V, [0, 1]^n)$.
  
  The proof of Theorem \ref{thm:mainDiscrete} boils down to showing that there exists a step size $\gamma$ and an error $\eps$ such that Dynamics \ref{dyn:SpernerDiscrete} are always $\alpha/\Lambda$ close to Dynamics \ref{dyn:Sperner}. To show this we use standard tools for the error of Euler discretized differential equations. In particular we use the following theorem.
  
  \begin{theorem}[Section 1.2 of \cite{iserles2009first}] \label{thm:Iserles}
    Let $y(t) \in \R^n$ be the solution to the differential equation $\dot{y} = G(y)$ with initial condition $y(0) = w$, where $G$ is a Lipschitz map $\R^n \to \R^n$. Let also $y^{(k + 1)} = y^{(k)} + \gamma \cdot G(y^{(k)})$, with initial condition $y^{(0)} = w'$, with $\norm{w - w'}_2 \le \zeta$. Then, for every $\eta > \zeta$ and every $T > 0$, there exists a step size $\gamma > 0$ such that
    \[ \norm{y(k \cdot \gamma) - y^{(k)}}_2 \le \eta ~~~~ \text{  for all } ~~ 0 \le k \le T/\gamma. \]
    Additionally, if the above holds for some $\gamma = \bar{\gamma}$ then it also holds for all $\gamma \le \bar{\gamma}$.
  \end{theorem}
  
  Given that $D^i_S(x)$ is Lipschitz (see Lemma \ref{l:Lip2}) we can apply Theorem \ref{thm:Iserles} to the while-loop of line 4 in Dynamics \ref{dyn:SpernerDiscrete} and inductively show that $x^{(m)}$ of Dynamics \ref{dyn:SpernerDiscrete} is close to the corresponding point of Dynamics \ref{dyn:Sperner}.
  
  Let $\tau_j$ be the value of the $\tau_{\mathrm{exit}}$ variable after the $j$-th time that the while-loop of line 4 in Dynamics \ref{dyn:Sperner} has ended. For every $i \in \mathbb{N}$ we define $t_i = \sum_{j = 1}^i \tau_i$. Our goal is to show that the $\norm{x^{(m)} - x(t_m)}_2$ is small. We do this inductively. For the base of our induction observe that $x^{(0)} = x(0)$. Now assume that we have chosen a step size $\gamma_m$ and that we have achieved $\norm{x^{(m)} - x(t_m)}_2 \le \zeta_m$ 
  Also we assume as an inductive hypothesis that before the beginning of $m$th while-loop of line 4 we have same epoch $(i, S)$ in both the execution of Dynamics \ref{dyn:Sperner} and the execution of Dynamics \ref{dyn:SpernerDiscrete}. Then, in the next execution of the while-loop of line 4 we have that $\norm{z^{(0)} - z(0)}_2 \le \zeta_m$. Also, from the proof of Theorem \ref{t:main} we know that there exists a finite $\tau_{m + 1}$ such that $z(\tau_{m + 1})$ is an exit point. Hence, we can apply Theorem \ref{thm:Iserles} and we get that for every $\eta > \zeta_m$, there exists a step size $\Gamma_{m + 1}$ such that
  \[ \norm{z^{(k)} - z(k \cdot \Gamma_{m + 1})}_2 \le \zeta_m + \frac{\delta}{2^{m + 1}} := \zeta_{m + 1} ~~~~ \text{for all } ~~ 0 \le k \le \tau_{m + 1}/\Gamma_{m + 1}. \]
  Since $x(t_m + \tau_{m + 1}) = z(\tau_{m + 1})$ we get that $\norm{x^{(m + 1)} - x(t_{m + 1})}_2 \le \zeta_{m + 1}$. The only thing that is missing is to show that the update on $(i, S)$ will be the same in the continuous and the discrete dynamics. Observe that if an exit point happens in the continuous dynamics then due to the Lipschitzness of $V$ the same exit point has to occur in as an $(\zeta_{m + 1}, \Gamma_{m + 1})$-exit point in the discrete dynamics. Now repeating the argument from the proof of Theorem \ref{t:main} we can easily show that it is impossible for more than one exit events to happen even in the discrete case. In particular, this follows easily from Assumption \ref{a:2} and Assumption \ref{a:1}. Hence, the update on $(i, S)$ will be the same. Then, we set $\gamma_{m + 1} = \min\{\gamma_m, \Gamma_{m + 1}\}$ and due to the last sentence of Theorem \ref{thm:Iserles} we know that using the step size $\gamma_{m + 1}$ in all the steps before $m + 1$ it will result only to better guarantees for the distance between $x^{(\ell)}$ and $x(t_{\ell})$ and therefore our induction follows. At the last iteration $M$ we will have
  \[ \norm{x^{(M)} - x(t_{M})}_2 \le \zeta_{M} \le \delta \left( \sum_{j = 1}^m \frac{1}{2^j} \right) \le \delta. \]
  Since $x(t_M)$ is a solution to $\mathrm{VI}(V, [0, 1]^n)$ we have that $x^{(M)}$ is an $\alpha$-approximate solution to $\mathrm{VI}(V, [0, 1]^n)$ and the step size that we used is $\gamma = \gamma_M$.
  
  Finally, the quantities $\bar{M}$ and $K$ are bounded by the constant $\bar{T}$ of Theorem \ref{t:main} divided by $\gamma = \gamma_m$.
\end{proof}
\end{document}